\documentclass[journal]{IEEEtran}

\usepackage{amsmath, amsfonts}
\usepackage{algorithmic}
\usepackage{algorithm}
\usepackage{array}
\usepackage{xr-hyper}
\makeatletter
\newcommand*{\addFileDependency}[1]{
  \typeout{(#1)}
  \@addtofilelist{#1}
  \IfFileExists{#1}{}{\typeout{No file #1.}}
}
\makeatother

\usepackage[colorlinks, linkcolor=red, citecolor=ForestGreen]{hyperref}
\usepackage{textcomp}
\usepackage{stfloats}
\usepackage{url}
\usepackage{verbatim}
\usepackage{graphicx}
\usepackage{subfigure}
\usepackage{mathrsfs}
\usepackage{array}
\usepackage{multirow}
\usepackage{multicol}
\usepackage{cite}
\usepackage{amsthm}
\usepackage{colortbl}
\usepackage[dvipsnames]{xcolor}
\usepackage{bookmark}
\usepackage{booktabs}
\usepackage{wasysym}
\usepackage{threeparttable}
\usepackage{float} 
\usepackage{tabularx} 
\PassOptionsToPackage{prologue,dvipsnames}{xcolor}

\definecolor{myPurple}{HTML}{D3B6E6}

\usepackage{amssymb} 
\usepackage{pifont} 
\newcommand{\cmark}{\ding{51}}
\newcommand{\xmark}{\ding{55}}

\usepackage{xspace}
\makeatletter
\DeclareRobustCommand\onedot{\futurelet\@let@token\@onedot}
\def\@onedot{\ifx\@let@token.\else.\null\fi\xspace}

\def\eg{\emph{e.g}\onedot} 
\def\ie{\emph{i.e}\onedot} 
 
\def\etc{\emph{etc}\onedot}

\makeatother

\DeclareMathOperator*{\argmax}{argmax~}

\renewcommand{\eqref}[1]{Eq.~\ref{#1}}

\newif\ifcomment
\commenttrue
\ifcomment

	\newcommand{\yl}[1]{ \noindent {\color{red} {\bf yl:} {#1}} }
 	
\else
	\newcommand{\yl}[1]{}
\fi

\newcommand{\red}[1]{\noindent{\color{red}{#1}}}
\newcommand{\blue}[1]{\noindent{\color{black}{#1}}}
\newcommand{\bluei}[1]{{\color{black}{#1}}}
\newcommand{\green}[1]{\noindent{\color{green}{#1}}}
\newcommand{\gray}[1]{\noindent{\color{gray}{#1}}}

\usepackage[dvipsnames]{xcolor}

\newtheorem{lemma}{Lemma}
\newtheorem{theorem}{Theorem}
\newtheorem{definition}{Definition}
\begin{document}

\title{RING\#: \blue{PR-by-PE Global Localization with Roto-translation Equivariant Gram Learning}}

\author{Sha Lu$^{1}$, Xuecheng Xu$^{1}$, Yuxuan Wu$^{2}$, Haojian Lu$^{1}$, Xieyuanli Chen$^{3}$, Rong Xiong$^{1}$ and Yue Wang$^{1\dagger}$
\thanks{$^{1}$Institute of Cyber-Systems and Control, Zhejiang University, China.}
\thanks{$^{2}$School of Astronautics, Beihang University, China.}
\thanks{$^{3}$College of Intelligence Science and Technology, National University of Defense Technology, China.}
\thanks{$^{\dagger}$\textit{Corresponding author: Yue Wang.} (E-mail: wangyue@iipc.zju.edu.cn)}

}



\maketitle

\begin{abstract}
Global localization using onboard perception sensors, such as cameras and LiDARs, is crucial in autonomous driving and robotics applications when GPS signals are unreliable. \blue{Most approaches achieve global localization by sequential place recognition (PR) and pose estimation (PE). Some methods train separate models for each task, while others employ a single model with dual heads, trained jointly with separate task-specific losses.} However, the accuracy of localization heavily depends on the success of place recognition, which often fails in scenarios with significant changes in viewpoint or environmental appearance. Consequently, this renders the final pose estimation of localization ineffective. \blue{To address this, we introduce a new paradigm, \textit{PR-by-PE localization}, which bypasses the need for separate place recognition by directly deriving it from pose estimation. We propose RING\#, an end-to-end \textit{PR-by-PE localization} network that operates in the bird's-eye-view (BEV) space, compatible with both vision and LiDAR sensors. RING\# incorporates a novel design that learns two equivariant representations from BEV features, enabling globally convergent and computationally efficient pose estimation.} Comprehensive experiments on the NCLT and Oxford datasets show that RING\# outperforms state-of-the-art \blue{methods in both vision and LiDAR modalities, validating the effectiveness of the proposed approach.} The code will be publicly released.

\end{abstract}


\section{Introduction}
\IEEEPARstart{G}{lobal} localization is a fundamental task in autonomous driving and mobile robot navigation systems. The need for global localization arises when GPS signals are unavailable or inaccurate, such as indoor and urban canyons. It also enables loop closures in SLAM~\cite{cummins2008fab, cummins2011appearance, milford2012seqslam, mur2015orb, lowry2015visual}, recovers the current pose of kidnapped robots, and merges multi-robot/multi-session maps. To achieve global localization, we must register the current sensor observation against the entire map without prior knowledge. 


\begin{figure}[t]
  \centering
  \includegraphics[width=0.9\linewidth]{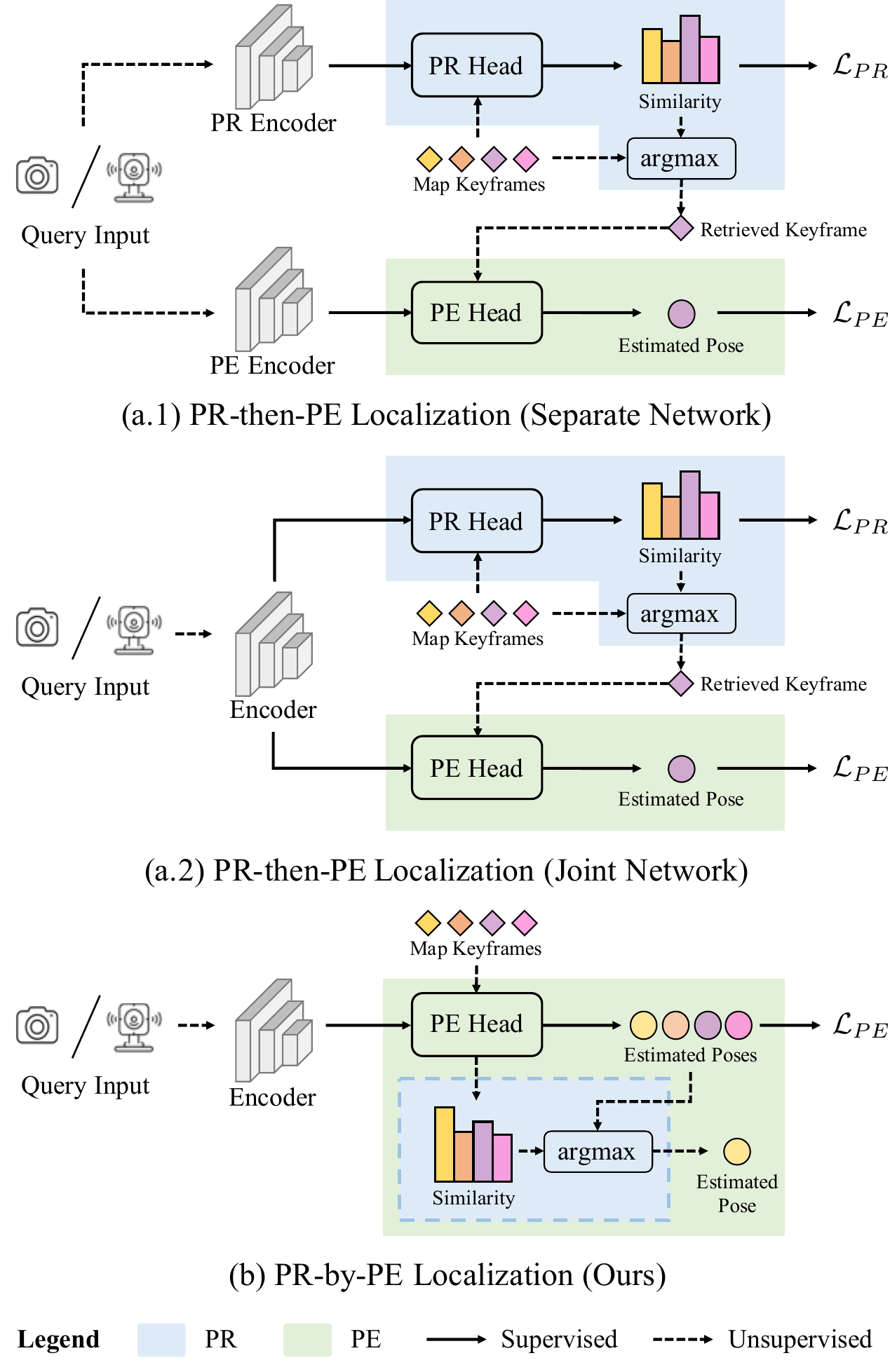}
  \vspace{-0.3cm}
  \caption{\blue{\textbf{Comparison on the various paradigms of global localization.} (a) The PR-then-PE localization paradigm treats place recognition and pose estimation as upstream and downstream tasks, either handled by two independent models (a.1) or jointly learned within a single model (a.2). (b) We introduce a novel paradigm: PR-by-PE localization, which leverages pose estimation to derive place recognition in a single model.}}
  \label{fig:teaser}
  \vspace{-0.5cm}
\end{figure}

\bluei{Traditional global localization approaches, originating from early visual methods~\cite{angeli2008fast, milford2012seqslam, arandjelovic2013all}, typically follow a paradigm: first place recognition (PR) then pose estimation (PE). This paradigm, referred to as \textit{PR-then-PE localization}, treats place recognition and pose estimation as upstream and downstream tasks, as illustrated in Fig.~\ref{fig:teaser}(a). Place recognition~\cite{galvez2012bags, arandjelovic2016netvlad, ge2020self, hausler2021patch, keetha2023anyloc, ali2023mixvpr, he2016m2dp, uy2018pointnetvlad, ma2022overlaptransformer, lu2023deepring} retrieves the map keyframe most similar to the query observation, narrowing down the search space for subsequent pose estimation. Pose estimation~\cite{fischler1981random, chum2003locally, chum2005matching, cavalli2020adalam, besl1992method, shi2023rdmnet, shi2024lcrnet} then aligns the query observation with the retrieved keyframe to estimate the robot's metric pose. As depicted in Fig.~\ref{fig:teaser}(a.1), early methods handle place recognition and pose estimation using separate models trained independently. Recent works~\cite{sarlin2018leveraging, sarlin2019coarse, du2020dh3d, chen2021overlapnet, xu2021disco, komorowski2021egonn, cattaneo2022lcdnet} have advanced this paradigm by integrating these two tasks into a single network, leveraging a shared encoder and task-specific heads to facilitate joint learning, as shown in Fig.~\ref{fig:teaser}(a.2). However, a common limitation of these \textit{PR-then-PE localization} methods is their heavy reliance on the success of place recognition. If place recognition fails to identify the correct keyframe, the subsequent pose estimation will also fail. Moreover, the objectives of these two tasks are inherently misaligned: place recognition aims for viewpoint invariance, while pose estimation requires viewpoint awareness. This misalignment can lead to cascading errors that degrade overall localization performance. This leads to a crucial question: \textit{Is it possible to bypass the need for a separate place recognition module and better align the objectives of these two tasks?}}

\bluei{To answer this question, we propose a novel paradigm: \textit{PR-by-PE localization}, as shown in Fig.~\ref{fig:teaser}(b). This paradigm infers place recognition by pose estimation, eliminating the separate place recognition module and avoiding error accumulation. A straightforward implementation would involve exhaustively applying an off-the-shelf pose solver to compute the relative pose and similarity score (\eg the number of inliers) of the query observation against all map keyframes. The keyframe with the highest similarity score is selected as the place recognition result. However, existing solvers either require robust optimization~\cite{sattler2012improving, sarlin2020superglue, jiang2021cotr, sun2021loftr, wang2022matchformer} or iterative search~\cite{besl1992method}, which are computationally expensive and prone to local minima. While directly regressing the relative pose between the query observation and all map keyframes offers a faster alternative, it typically suffers from limited accuracy and lacks a similarity measure for place recognition. Absolute pose regression~\cite{kendall2015posenet} is another option, but it struggles to generalize to unseen environments. Therefore, current methods are not well-suited for \textit{PR-by-PE localization}. We argue that a successful \textit{PR-by-PE localization} approach requires a relative pose estimation network that is both globally convergent and computationally efficient, with a built-in similarity assessment mechanism.}

\bluei{To this end, we propose RING\#, an end-to-end \textit{PR-by-PE localization} framework that explicitly predicts a 3-DoF pose (x, y, yaw) with a built-in similarity score, supervised solely by the pose estimation loss. The key design is the two equivariant representations to decouple pose estimation into sequential rotation and translation estimation, which effectively reduces the search space dimensionality.} \blue{By employing a correlation-based exhaustive search on equivariant features in two subspaces}, \blue{RING\# achieves globally} convergent pose estimation. \blue{The resulting correlation values serve as similarity scores for place recognition, thus effectively embodying the \textit{PR-by-PE localization} paradigm.} Thanks to fast correlation computation using the Fast Fourier Transform (FFT), batch processing on GPU, as well as the sparse map keyframe aided by the large convergence basin, RING\# is computationally feasible. \blue{Furthermore, RING\# employs a bird's-eye-view (BEV) architecture that supports both vision and LiDAR modalities, making it a versatile solution for various sensor inputs.} Overall, our contributions are summarized as follows:
\begin{itemize}
\item \blue{We introduce a novel paradigm for global localization: \textit{PR-by-PE localization}, which derives place recognition by pose estimation.}
\item \blue{We present RING\#, an end-to-end framework that learns equivariant representations to enable globally convergent localization and efficient evaluation.}
\item \blue{We design a BEV-based feature learning architecture compatible with both vision and LiDAR modalities.}
\item \blue{We validate the effectiveness of RING\# through extensive experiments on the NCLT and Oxford datasets, demonstrating superior performance across both vision and LiDAR modalities, with RING\#-V even outperforming most LiDAR-based methods.}
\end{itemize}


\bluei{In our previous work RING++~\cite{xu2023ring++}, we propose a learning-free framework that aims to construct a roto-translation invariant representation for global localization on a sparse scan map. However, RING++ is limited by its inability to learn from data, leading to suboptimal performance in challenging scenarios and restricting its application to LiDAR inputs only. In contrast, RING\# introduces an end-to-end network that learns BEV features while maintaining equivariance to both rotation and translation. This design not only enhances LiDAR-based localization performance, but also extends its applicability to vision-based scenarios. In addition, RING++ operates under the \textit{PR-then-PE localization} paradigm, where errors in place recognition propagate and affect the final pose estimation. RING\#, however, adopts the \textit{PR-by-PE localization} paradigm, avoiding the cascading error accumulation. This shift results in a substantial improvement in localization performance, increasing the global localization success rate by around 20\%.}

\section{Related Work}
In this section, we provide an overview of methods for visual- and LiDAR-based global localization, followed by a review of methods for BEV representation learning.

\textbf{Vision-based Global Localization.}
Vision-based localization methods typically \blue{adhere to the \textit{PR-then-PE localization}} paradigm: place recognition then pose estimation. In the place recognition stage, the most similar map image to the query image is retrieved for coarse localization. Traditional methods rely on handcrafted local features like SIFT~\cite{lowe2004distinctive}, SURF~\cite{bay2006surf} and ORB~\cite{rublee2011orb} to extract local descriptors, which are then aggregated into global descriptors using aggregation algorithms such as Bag of Words (BoW)~\cite{galvez2012bags}, Fisher Kernel~\cite{perronnin2007fisher, perronnin2010large}, and Vector of Locally Aggregated Descriptors (VLAD) \cite{arandjelovic2013all}. With the advent of deep learning, there is a shift towards learned local features~\cite{detone2018superpoint, revaud2019r2d2, jau2020deep, sun2021loftr}, along with learnable aggregation algorithms~\cite{arandjelovic2016netvlad, radenovic2018fine, peng2021attentional}. In addition to CNN-based methods~\cite{uy2018pointnetvlad, hausler2021patch, ge2020self, ali2023mixvpr}, recent approaches~\cite{wang2022transvpr, keetha2023anyloc, kannan2024placeformer} introduce vision Transformers for end-to-end learning of global descriptors. In the pose estimation stage, feature matching methods~\cite{sarlin2020superglue, jiang2021cotr, wang2022matchformer} establish correspondences between query and retrieved images using the above local descriptors. Based on the correspondences, a robust pose solver~\cite{fischler1981random, chum2003locally, chum2005matching, cavalli2020adalam} is applied to obtain the query pose in the map coordinate. To boost the place recognition performance, matching local features in the above process is employed for the top N candidates, serving as a post-processing step to refine the place recognition result. Recently, there's a trend towards integrating place recognition and pose estimation in an end-to-end manner~\cite{sarlin2018leveraging, sarlin2019coarse} or utilizing local matching to re-rank place recognition results~\cite{hausler2021patch, wang2022transvpr, barbarani2023local, zhu2023r2former, kannan2024placeformer}.

\textbf{LiDAR-based Global Localization.}
\blue{LiDAR-based localization methods have gained increasing attention due to their robustness to appearance changes. Following the \textit{PR-then-PE localization} paradigm, these methods} address global localization through place recognition algorithms, followed by point cloud registration techniques like Iterative Closest Point (ICP)~\cite{besl1992method}. M2DP~\cite{he2016m2dp} represents 3D point clouds as fingerprints by projecting onto multiple 2D planes. Scan Context and its variants~\cite{kim2018scan, wang2020intensity, wang2020lidar, kim2021scan, li2021ssc} transform the raw 3D point cloud to a BEV representation in the polar coordinate with the height, intensity, or semantic information. RING~\cite{lu2022one} and its extension RING++~\cite{xu2023ring++} employ the Radon transform to represent a point cloud as a sinogram for both place recognition and pose estimation. BoW3D~\cite{cui2022bow3d} exploits a novel linear keypoint descriptor~\cite{cui2024link3d} to build BoW for point clouds and achieve LiDAR-based place recognition and pose estimation. Equipped with the power of deep learning, PointNetVLAD~\cite{uy2018pointnetvlad} extracts local features of a LiDAR scan with PointNet~\cite{qi2017pointnet} and utilizes NetVLAD~\cite{arandjelovic2016netvlad} to aggregate them into a global descriptor. DiSCO~\cite{xu2021disco} leverages the property of the Fourier transform to output a rotation-invariant global descriptor. OverlapNet~\cite{chen2021overlapnet} and its derivative OverlapTransformer~\cite{ma2022overlaptransformer} estimate the overlap between two scans for place recognition. DeepRING~\cite{lu2023deepring} extracts features from the sinogram to obtain a rotation-invariant global representation. Moreover, SpectralGV~\cite{vidanapathirana2023spectral} presents a spectral method for geometric verification to enhance place recognition performance. Except for learning merely place recognition, some works~\cite{du2020dh3d, komorowski2021egonn, cattaneo2022lcdnet} learn both global and local descriptors for place recognition and pose estimation in a single forward pass.

\textbf{BEV Representation Learning.}
There are two mainstream classes of view transformation methods to generate BEV features: depth-based methods and transformer-based methods. After LSS~\cite{philion2020lift} method arises, BEV representation is widely used in 3D perception tasks, \eg 3D objection detection, segmentation, and scene completion. These two approaches represent two lines of BEV perception studies. One line of work incorporates depth-based methods like LSS~\cite{philion2020lift}, BEVDet~\cite{huang2021bevdet}, BEVDepth~\cite{li2023bevdepth}, \etc Depth-based methods lift 2D perspective view (PV) features to 3D features by estimating the depth distribution of each pixel and then generating BEV features by reducing the vertical dimension. The other line is composed of transformer-based methods like DETR3D~\cite{wang2022detr3d}, PETR~\cite{liu2022petr}, BEVFormer~\cite{li2022bevformer,yang2023bevformer}, \etc Unlike depth-based methods in a bottom-up manner, transformer-based methods utilize BEV queries to retrieve PV features through the cross-attention mechanism in a top-down manner. 

However, the performance of these localization methods is inherently tied to the success of place recognition. When place recognition fails, the overall localization performance suffers significantly. To overcome this limitation, we present a \blue{BEV-based \textit{PR-by-PE localization}} framework by efficient exhaustive pose estimation, showcasing superior performance.


\section{Overview}
\label{sec:overview}
We first introduce the problem statement of global localization in Sec.~\ref{sec:problem}. Then, we provide an overview of the proposed framework RING\# in Sec.~\ref{sec:framework}.

\subsection{Problem Statement}
\label{sec:problem}
Given a query observation $Q$ and a database of map keyframe observations $\mathfrak{M} \triangleq \{M_1, M_2, \cdots, M_n\}$ with known poses, the goal of global localization is to estimate the pose of $Q$ in the map coordinate. For efficient computation and low memory usage, \blue{the map density,} denoted as $|\mathfrak{M}|$, is expected to be as small as possible. In autonomous driving and robotics applications, the gravity direction is easily known and changes in pitch, roll, and height within a local area are generally negligible~\cite{Lu_2019_CVPR, sarlin2024snap}, which reduces the global localization problem from estimating a 6-DoF pose to a 3-DoF pose $T \in \text{SE(2)}$. Therefore, we define the global localization problem as a global 3-DoF pose estimation task. $T$ can be decomposed into a \blue{1-DoF} rotation $\theta \in [0, 2\pi)$ and a \blue{2-DoF} translation $(x, y) \in \mathbb{R}^2$ as follows:
\begin{equation}
    T = \begin{bmatrix}
        \cos\theta & -\sin\theta & x \\
        \sin\theta & \cos\theta & y \\
        0 & 0 & 1
    \end{bmatrix}.
\end{equation}

\textbf{\blue{PR-then-PE Localization.}} Existing global localization methods generally adhere to \blue{the \textit{PR-then-PE localization} paradigm}, formulating the global localization problem as two sequential sub-problems: place recognition and pose estimation. \blue{These methods tackle place recognition and pose estimation either using two standalone models or a single model with a shared encoder and task-specific heads.} The place recognition algorithm retrieves the most similar map keyframe to $Q$ from $\mathfrak{M}$, as defined in Eq.~(\ref{eq:pr}). Subsequently, the pose estimation algorithm estimates the relative pose between $Q$ and the \blue{retrieved} map keyframe, as formulated in Eq.~(\ref{eq:pe}).
\begin{align}
    \hat{M}_{i} & = \argmax_{M_i \in \mathfrak{M}} \mathcal{S}_{pr}(Q, M_i) \label{eq:pr}, \\
    \hat{T}^{\hat{M}_{i}}_{Q} & = \argmax_{T^{\hat{M}_{i}}_{Q}} \mathcal{S}_{pe}(Q, \hat{M}_{i}, T^{\hat{M}_{i}}_{Q}), \label{eq:pe}
\end{align}
where $\hat{M}_{i}$ is the retrieved map keyframe, $\mathcal{S}_{pr}(\cdot)$ is the place recognition similarity function, $\hat{T}^{\hat{M}_{i}}_{Q}$ is the estimated relative pose between $Q$ and $\hat{M}_{i}$, and $\mathcal{S}_{pe}(\cdot)$ is the pose estimation similarity function. The global pose of $Q$ is then derived by
\begin{equation}
    \hat{T}^{\mathfrak{M}}_{Q} = T^{\mathfrak{M}}_{\hat{M}_{i}} \hat{T}^{\hat{M}_{i}}_{Q},
    \label{eq:gl_pose}
\end{equation}
where \blue{$\hat{T}^{\mathfrak{M}}_{Q}$ is the estimated pose of $Q$ in the map coordinate, and} $T^{\mathfrak{M}}_{\hat{M}_{i}}$ is the pose of $\hat{M}_{i}$ in the map coordinate. However, the localization performance of such methods heavily depends on the success of place recognition in identifying the most similar map keyframe in the database. If place recognition using $\mathcal{S}_{pr}(\cdot)$ fails, the output of $\mathcal{S}_{pe}(\cdot)$ becomes meaningless, inevitably resulting in incorrect localization.


\textbf{\blue{PR-by-PE Localization.}} Rethinking the similarity function in Eq.~(\ref{eq:pe}), we note that the map keyframe with the highest similarity after pose alignment should be the correct retrieval for place recognition when performing Eq.~(\ref{eq:pe}) for all map keyframes $M_i$ in $\mathfrak{M}$. \blue{This insight forms the basis of the \textit{PR-by-PE localization} paradigm}. \blue{In this paradigm, the 3-DoF localization problem is formulated as}
\begin{equation}
    \hat{M}_{i}, \hat{\theta}, \hat{x}, \hat{y} = \argmax_{M_i \in \mathfrak{M}, \theta, x, y} \mathcal{S}_{pe}(Q, M_i, \theta, x, y),
    \label{eq:gl}
\end{equation}
where $\hat{M}_{i}$ is the retrieved map keyframe, $\hat{\theta}$ is the estimated relative rotation angle, $(\hat{x}$, $\hat{y})^T$ is the estimated relative translation vector, and $\mathcal{S}_{pe}(\cdot)$ is the \blue{\textit{PR-by-PE localization}} similarity function. \blue{This paradigm derives place recognition as a by-product of pose estimation, effectively bypassing the inherent limitations of \textit{PR-then-PE localization}}. However, there are two main challenges in designing an effective similarity function $\mathcal{S}_{pe}(\cdot)$: the \textit{efficient evaluation} of similarity throughout the entire database for each frame, and the \textit{global estimation} of 3-DoF pose to eliminate the dependency on the initial value.

\begin{figure}[t]
    \centering
    \includegraphics[width=1.0\linewidth]{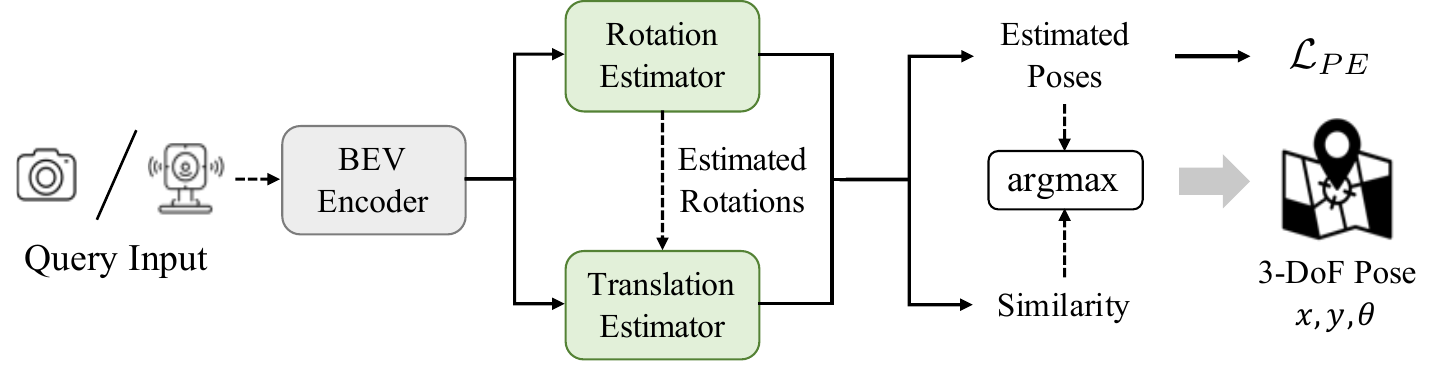}
    \vspace{-0.5cm}
    \caption{\textbf{\blue{PR-by-PE localization framework.}} Given a \blue{raw sensor} observation, we encode it into BEV features first. Based on the BEV features, we construct two equivariant representations that enable the decoupling of pose estimation into rotation estimation and translation estimation.}
    \label{fig:general_framework}
    \vspace{-0.5cm}
\end{figure}

\subsection{Framework}
\label{sec:framework}

To address the above challenges, we propose a framework that \blue{adopts the \textit{PR-by-PE localization} paradigm}, as illustrated in Fig.~\ref{fig:general_framework}. We first encode the query observation into BEV features. Based on the BEV features, we construct a rotation-equivariant and translation-invariant representation, and a translation-equivariant and rotation-invariant representation (proved in Sec.~\ref{sec:equivariance}). With these two equivariant representations, we decouple the 3-DoF pose estimation into sequential 1-DoF rotation estimation and 2-DoF translation estimation, \blue{thereby reducing the search space dimensionality}. For both rotation and translation estimation, we design a correlation-based similarity function to exhaustively search for all possible pose configurations in the 3-DoF pose space. As a result, our localization framework offers the following benefits:


\textbf{\blue{Built-in Similarity.}} \blue{The similarity score for place recognition is derived directly from the maximum correlation value obtained through correlation-based pose estimation. This design naturally embodies the \textit{PR-by-PE localization} paradigm, facilitating place recognition by pose estimation.}

\textbf{Global Convergence.} The construction of equivariant representations ensures that the learned features remain equivariant to pose transformations of the input data. This property enables exhaustive matching between the query and all map keyframes in the database across all possible pose configurations, $(x, y, \theta)$. Benefiting from exhaustive matching, our approach converges to the optimal solution without dependency on the initial value.

\textbf{Evaluation Efficiency.} \blue{The decoupling of 3-DoF pose estimation into rotation and translation estimation, combined with the use of correlation-based similarity function, greatly enhances computational efficiency. This similarity function is accelerated using the Fast Fourier Transform (FFT) and batch processing on GPU, drastically reducing the computational cost. This design enables efficient similarity evaluation across the entire database for each frame, making the exhaustive search feasible in practice.}



\bluei{\textbf{Vision and LiDAR Compatibility.} Our framework incorporates a BEV-based feature learning architecture capable of encoding sensor observations into BEV features for both vision and LiDAR modalities. These BEV features are then used to construct equivariant representations that are effective across different sensor types,} allowing the framework applicable to both vision- and LiDAR-based localization.

\section{Equivariant Representations}
\label{sec:equivariance}
In this section, we analyze the rotation and translation equivariance and invariance properties of the convolution (CNN), the Radon transform (RT), and the Fourier transform (FT). Upon these properties, we construct two equivariant representations for RING\# and provide theoretical proofs.

\subsection{Definitions}
We first present the definition of equivariance and invariance as follows:

\begin{definition}[Equivariance]
\label{def:equivariance}
For a group of transformations $G$, a function $f$ is equivariant if:
\begin{equation}
    \label{eq:equivariance}
    f(T_g[x]) = S_g[f(x)], \quad \forall x \in X, \, g \in G,
\end{equation}
where $x$ is the input, $g$ is an element of the group $G$, and $T_g$ and $S_g$ are transformations parameterized by $g$.
\end{definition}

\begin{definition}[Invariance]
\label{def:invariance}
For a group of transformations $G$, a function $f$ is invariant if:
\begin{equation}
    \label{eq:invariance}
    f(T_g[x]) = f(x), \quad \forall x \in X, \, g \in G,
\end{equation}
where $x$ is the input, $g$ is an element of the group $G$, and $T_g$ is a transformation parameterized by $g$.
\end{definition}

\subsection{Convolution}
The convolution on $f(x)$ with the kernel $h(x)$ is formulated as follows:
\begin{equation}
    \label{eq:cnn}
    f_h(x) \triangleq \phi_h(f(x)) = \int_{-\infty}^{\infty} f(u) h(x - u) \, \mathrm{d}u,
\end{equation}
where $f_h(x)$ is the output of $f(x)$ after convolution, and $\phi_h(\cdot)$ is the convolution operator with kernel $h(x)$. Then we show the equivariance of the function after applying convolution.

\begin{lemma}
\label{lemma 1}
$\phi_h(f(x))$ is translation equivariant.
\end{lemma}

\begin{proof}
Translate the input function $f(x)$ by $\Delta{x}$, which generates $f'(x) \triangleq f(x - \Delta{x})$. The convolution on $f'(x)$ with the kernel $h(x)$, denoted as $f'_h(x)$ is calculated by
\begin{equation}
\begin{aligned}
    f'_h(x) &= \int_{-\infty}^{\infty} f'(u) h(x - u) \, \mathrm{d}u \\
    &= \int_{-\infty}^{\infty} f(u - \Delta{x}) h(x - u) \, \mathrm{d}u.
\end{aligned}
\end{equation}

Let $u' = u - \Delta{x}$, which implies $\mathrm{d}u' = \mathrm{d}u$. Then we have
\begin{equation}
\begin{aligned}
    f'_h(x) &= \phi_h(f(x - \Delta{x})) \\
    &= \int_{-\infty}^{\infty} f(u') h((x - \Delta{x}) - u') \, \mathrm{d}u' \\
    &= f_h(x - \Delta{x}).
\end{aligned}
\end{equation}
Therefore, $\phi_h(f(x))$ is translation equivariant according to Definition~\ref{def:equivariance}. In addition, the equivariance can be extended to higher dimensions. We omit the proof here.
\end{proof}

\subsection{Radon Transform}
\label{sec:radon}
The Radon transform \cite{deans2007radon} is a linear integral transform that computes the integral along a set of straight lines. The mathematical equation of the Radon transform is as follows:
\begin{equation}
    \label{eq:radon}
    \begin{split}
        S(\theta, \tau) &\triangleq \mathcal{R}(f(x,y)) \\
        &=\int_{L(\theta, \tau): \; x\cos\theta + y\sin\theta = \tau} f(x,y) \mathrm{d}x \mathrm{d}y \\ 
        &=\int_{-\infty}^{\infty}\int_{-\infty}^{\infty} f(x,y) \delta(\tau-x\cos\theta-y\sin\theta) \mathrm{d}x \mathrm{d}y,
    \end{split}
\end{equation}
where $S(\theta, \tau)$ represents the resultant sinogram, $\mathcal{R}(\cdot)$ denotes the Radon transform operation, $f(x, y)$ is the input 2D function, $L(\theta, \tau): x\cos\theta + y\sin\theta = \tau$ is the line for integral, $\theta \in [0, 2\pi)$ is the tangent angle of the line $L(\theta, \tau)$, $\tau \in (-\infty, \infty)$ is the distance from the origin to $L(\theta, \tau)$, and $\delta(\cdot)$ is the Dirac delta function.




\textbf{Sinogram after Rotation and Translation.} Let $f'(x,y)$ be transformed by $f(x,y)$ with a 3-DoF pose transformation $T$ parameterized by a rotation angle $\alpha$ and a translation vector $t \triangleq (\Delta x, \Delta y)^{T}$. $f'(x,y)$ is formulated as
\begin{equation}
\begin{aligned}
    &f'(x,y) \triangleq f(R_{\alpha} X - t), \\
    R_{\alpha} &\triangleq \begin{bmatrix}
        \cos\alpha & -\sin\alpha \\
        \sin\alpha & \cos\alpha
    \end{bmatrix}, X \triangleq \begin{bmatrix} x \\ y \end{bmatrix},
\end{aligned}
\end{equation}
where $R_{\alpha}$ represents the rotation matrix parameterized by $\alpha$.

Applying the Radon transform to $f'(x,y)$, the resultant sinogram $S'(\theta, \tau)$ can be expressed as
\begin{equation}
\begin{aligned}
    \label{eq:radon_pose}
    S'(\theta, \tau) &= \mathcal{R}(f'(x,y)) \\
    &= \int_{-\infty}^{\infty}\int_{-\infty}^{\infty} f(R_{\alpha} X - t) \delta(\tau - k_{\theta} \cdot X) \mathrm{d}x \mathrm{d}y \\
    &= S(\theta + \alpha, \tau - \Delta{\tau}),
\end{aligned}
\end{equation}
where $k_{\theta} \triangleq (\cos{\theta},\sin{\theta})^{T}$ is a unit vector of the line $L(\theta, \tau)$, and $\Delta{\tau}$ is equivalent to the projection of the translation vector $t$ on the line $L(\theta + \alpha, \tau)$, which is calculated by
\begin{equation}
\begin{aligned}
\Delta{\tau} &= {k}_{\theta+\alpha} \cdot t \\
&= (\cos(\theta+\alpha),\sin(\theta+\alpha)) \cdot (\Delta{x}, \Delta{y}) \\
&= \Delta{x}\cos(\theta+\alpha) + \Delta{y}\sin(\theta+\alpha).
\end{aligned}
\end{equation}
Therefore, a rotation angle $\alpha$ on $f(x, y)$ causes a circular shift along the $\theta$ axis of $S(\theta, \tau)$, and a translation vector $t$ on $f(x, y)$ results in a shift in the variable $\tau$ equal to the projection of $t$ onto the line $L(\theta+\alpha, \tau)$.

\textbf{Comparison with Polar Transform.}
The polar transform (PT) \cite{matungka2009image} is widely used to construct an observation representation, which is formulated as
\begin{equation}
\begin{aligned}
    \label{eq:polar}
    &p(r,\theta) \triangleq \mathcal{P}(f(x,y)) = f(r\cos{\theta}, r\sin{\theta}), \\
    &r = \sqrt{x^2+y^2}, \\
    &\theta = \arctan\frac{y}{x},
\end{aligned}
\end{equation}
where $p(r,\theta)$ is the result of the polar transform, $\mathcal{P}(\cdot)$ is the polar transform operator, and $f(x,y)$ is the input 2D image. The polar transform of $f'(x,y)$ is formulated as
\begin{equation}
\begin{aligned}
    \label{eq:polar_pose}
    &p'(r, \theta) = p(r', \theta'), \\
    &r' = \sqrt{(r\cos(\theta+\alpha)-\Delta{x})^2+(r\sin(\theta+\alpha)-\Delta{y})^2}, \\
    &\theta' = \arctan\frac{r\sin(\theta+\alpha)-\Delta{y}}{r\cos(\theta+\alpha)-\Delta{x}},
\end{aligned}
\end{equation}
where $p'(r, \theta)$ is the polar representation of $f'(x,y)$. In contrast to the Radon transform in Eq.~(\ref{eq:radon_pose}), the $\theta$ axis of $p(r,\theta)$ is a nonlinear combination of the rotation angle $\alpha$ and the translation $t$, so is the $r$ axis. Such representation obviously loses equivariance after neural network processing.

\subsection{Fourier Transform}
The Fourier transform \cite{brigham1988fast} is an integral transform that represents a function in the frequency domain, whose formula is
\begin{equation}
    \widehat{f}(\omega) \triangleq \mathcal{F}(f(x)) = \int_{-\infty}^{\infty} f(x) e^{-i \omega x} \mathrm{d}x,
\end{equation}
where $f(x)$ is a function in the time domain, $\mathcal{F}(\cdot)$ is the Fourier transform operator, $\omega$ is the angular frequency, and $\widehat{f}(\omega)$ is the representation in the frequency domain. Let $\mathcal{A}(\cdot) \triangleq |\mathcal{F}(\cdot)|$ be the composed operator of $\mathcal{F}(\cdot)$ and $|\cdot|$, where $|\cdot|$ is the operation of taking magnitude.

\begin{lemma}
\label{lemma 3}
$\mathcal{A}(f(x))$ is translation invariant.
\end{lemma}

\begin{proof}
Suppose $f'(x) \triangleq f(x - \Delta{x})$ is the function after translating $f(x)$ by $\Delta{x}$. Referring to the time shifting property, the Fourier transform converts a shift $\Delta{x}$ in the time domain to a phase shift $-\omega\Delta{x}$ in the frequency domain:
\begin{equation}
\begin{aligned}
    \label{eq:fourier_shift}
    \widehat{f'}(\omega) &= \mathcal{F}(f(x - \Delta{x})) \\
    &= \int_{-\infty}^{\infty} f(x - \Delta{x}) e^{-i \omega x} \mathrm{d}x \\
    &= \int_{-\infty}^{\infty} f(x - \Delta{x}) e^{-i \omega (x - \Delta{x})} e^{-i\omega\Delta{x}} \mathrm{d}(x - \Delta{x}) \\
    &= e^{-i\omega\Delta{x}} \int_{-\infty}^{\infty} f(x - \Delta{x}) e^{-i \omega (x - \Delta{x})} \mathrm{d}(x - \Delta{x}) \\    
    &= e^{-i\omega\Delta{x}}\widehat{f}(\omega).
\end{aligned}
\end{equation}
$\widehat{f'}(\omega)$ is the representation of $f'(x)$ in the frequency domain.

The amplitude of $\widehat{f'}(\omega)$ remains the same regardless of the value of $\Delta{x}$:
\begin{equation}
\begin{aligned}
A'(\omega) &= \mathcal{A}(f(x - \Delta{x}))\\
&= |\widehat{f'}(\omega)| \\
&= |e^{-i \omega \Delta{x}} \widehat{f}(\omega)| \\
&= |\widehat{f}(\omega)| \\
&= A(\omega),
\end{aligned}
\end{equation}
where $A(\omega)$ and $A'(\omega)$ are the amplitudes of $f(x)$ and $f'(x)$ in the frequency domain, satisfying Definition~\ref{def:invariance}. Therefore, $\mathcal{A}(f(x))$ is translation invariant.
\end{proof}

\begin{figure*}[t]
    \centering
    \includegraphics[width=18cm]{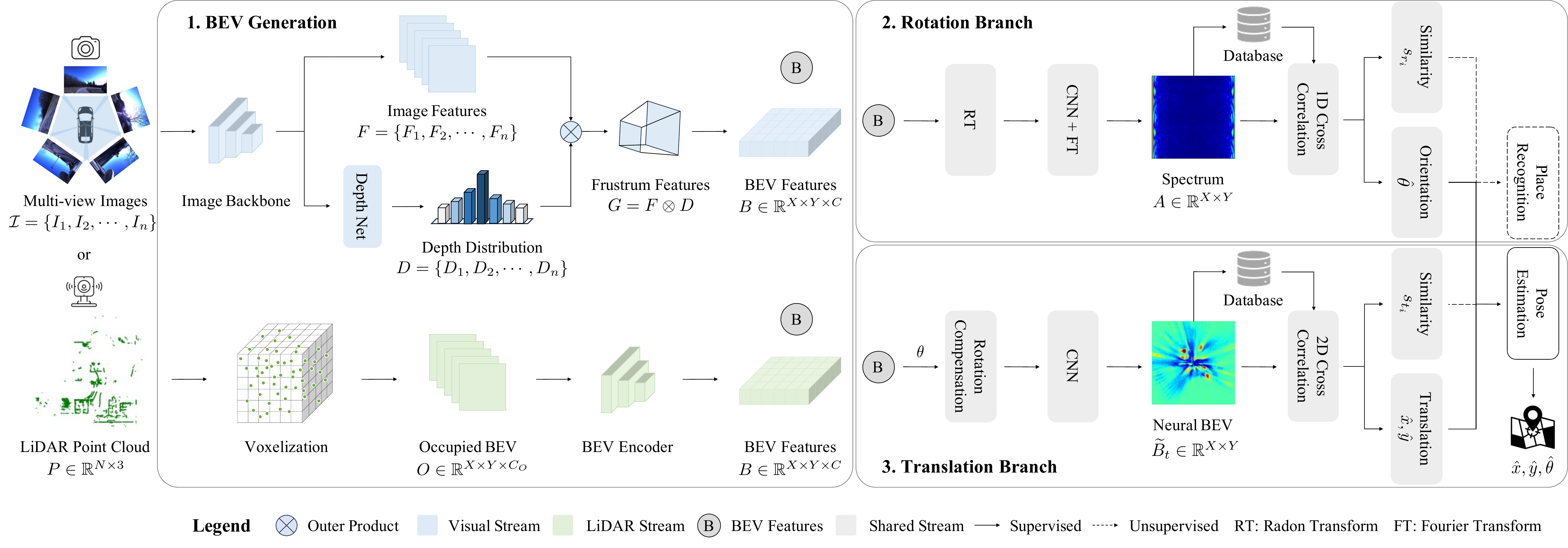}
    \vspace{-0.3cm}
    \caption{\textbf{Overview of the \blue{PR-by-PE localization} framework RING\#.} 1. Our BEV generation module converts inputs from multi-view images $\mathcal{I}$ or a LiDAR point cloud $P$ into BEV features $B$. 2. Using the Radon Transform (RT), a Convolutional Neural Network (CNN), and the Fourier Transform (FT), the rotation branch transforms $B$ into rotation-equivariant and translation-invariant representations $A$ and then uses 1D cross-correlation to estimate the relative rotation $\hat{\theta}$. 3. The translation branch compensates for the relative rotation $\theta$ \blue{which equals to the ground truth rotation $\theta^{*}$ during training and equals to the estimated rotation by the rotation $\hat{\theta}$ branch during inference} and uses a CNN to yield rotation-invariant and translation-equivariant representations $\widetilde{B}_t$. Subsequent 2D cross-correlation is employed to determine the relative translation $\hat{x}, \hat{y}$. RING\# is supervised by poses only in an end-to-end manner.}
    \label{fig:overview}
    \vspace{-0.5cm}
\end{figure*}

\subsection{Rotation Equivariant Representation}
\label{sec:rotation}
Given a bounded 2D function $f(x,y)$, we first apply the Radon transform $\mathcal{R}(\cdot)$ to it to generate a sinogram $S(\theta, \tau)$. Then, we perform a 1D CNN $\phi_h(\cdot)$ on the \blue{$\theta$} coordinate of $S(\theta, \tau)$ to extract features $S_h(\theta, \tau)$. Ultimately, we employ the Fourier transform with the magnitude operation $\mathcal{A}(\cdot)$ to the variable $\tau$, yielding the magnitude spectrum $A(\theta, \omega)$ (\ie the spectrum in Fig.~\ref{fig:overview} and Sec.~\ref{sec:rotation}).

\begin{theorem}
\label{theorem 1}
$A(\theta, \omega)$ is rotation equivariant and translation invariant.
\end{theorem}

\begin{proof} By the properties of the RT in Eq.~(\ref{eq:radon_pose}), CNN (Lemma~\ref{lemma 1}), and FT (Lemma~\ref{lemma 3}), the magnitude spectrum of $f'(x,y)$ can be expressed as
\begin{equation}
\begin{aligned}
    A'(\theta, \omega) &= \mathcal{A}(\phi_h(\mathcal{R}(f'(x,y)))) \\
    &= \mathcal{A}(\phi_h(S(\theta + \alpha, \tau - \Delta{\tau}))) \\
    &= \mathcal{A}(S_h(\theta + \alpha, \tau - \Delta{\tau})) \\
    &= A(\theta + \alpha, \omega),
\end{aligned}
\end{equation}
where $A'(\theta, \omega)$ is the resultant spectrum generated by $f'(x,y)$. The rotation equivariance satisfies Definition~\ref{def:equivariance} and the translation invariance satisfies Definition~\ref{def:invariance}. Therefore, $A(\theta, \omega)$ is rotation equivariant and translation invariant.
\end{proof}

\subsection{Translation Equivariant Representation}
\label{sec:translation}
Define $r(f(x,y), f'(x,y))$ as a function that compensates for the relative rotation between $f(x,y)$ and $f'(x,y)$. The rotation-compensated function, $\widetilde{f}(x,y)$, is expressed as
\begin{equation}
    \label{eq:rot_com_func}
    \widetilde{f}(x,y) \triangleq r(f(x,y), f'(x,y)) = f(R_{\alpha} X),
\end{equation}
where $\alpha$ is the rotation angle required to align $f(x,y)$ with $f'(x,y)$, and $R_{\alpha}$ is the corresponding rotation matrix.

Given a bounded 2D function $f(x,y)$, we use the rotation compensation function $r(\cdot, \cdot)$ to get a rotation-compensated function $\widetilde{f}(x,y)$. Then we apply a 2D CNN $\phi_h(\cdot)$ to $\widetilde{f}(x,y)$, generating features $\widetilde{f}_h(x, y)$ (\ie the neural BEV in Fig.~\ref{fig:overview} and Sec.~\ref{sec:translation}).

\begin{theorem}
\label{theorem 2}
$\widetilde{f}_h(x, y)$ is rotation invariant and translation equivariant.
\end{theorem}

\begin{proof} After rotation compensation on $f(x, y)$ and $f'(x, y)$, we have
\begin{align}
\begin{aligned}
    \widetilde{f}(x,y) &= r(f(x,y), f'(x,y)) \\
    &= f(R_{\alpha} X),
\end{aligned} \\
\begin{aligned}
    \widetilde{f}'(x,y) &= r(f'(x,y), f'(x,y)) \\
    &= f'(x,y),
\end{aligned}
\end{align}
where $\widetilde{f}'(x,y)$ denotes the rotation compensation of $f'(x,y)$. Applying the 2D CNN $\phi_h(\cdot)$ to $\widetilde{f}(x,y)$, we can get
\begin{equation}
\begin{aligned}
    \label{eq:rot_com}
    \widetilde{f}_h(x, y) = \phi_h(\widetilde{f}(x,y)) &= \phi_h(f(R_{\alpha} X)) \\
    &= \phi_h(f(\widetilde{x}, \widetilde{y})) \\
    &= f_h(\widetilde{x}, \widetilde{y}),
\end{aligned}
\end{equation}
where $(\widetilde{x}, \widetilde{y})^T \triangleq R_{\alpha} X$. Employing the 2D CNN $\phi_h(\cdot)$ to $f'(x,y) = f(R_{\alpha} X - t)$ and utilizing the translation equivariance of CNN (Lemma~\ref{lemma 1}), we arrive at
\begin{equation}
\begin{aligned}
    \widetilde{f}'_h(x,y) = \phi_h(\widetilde{f}'(x,y)) &= \phi_h(f'(x,y)) \\
    &= \phi_h(f(R_{\alpha} X - t)) \\
    &= \phi_h(f(\widetilde{x} - \Delta{x}, \widetilde{y} - \Delta{y})) \\
    &= f_h(\widetilde{x} - \Delta{x}, \widetilde{y} - \Delta{y}).
\end{aligned}
\end{equation}
Compared with Eq.~(\ref{eq:rot_com}), $\widetilde{f}_h(x,y)$ is rotation invariant and translation equivariant from Definition~\ref{def:equivariance} and Definition~\ref{def:invariance}.
\end{proof}

\section{\blue{PR-by-PE} Localization}
In this section, we detail each component of the RING\# architecture illustrated in Fig.~\ref{fig:overview} in the following subsections.

\subsection{BEV Generation}
\label{sec:bev_generation}
BEV generation involves two distinct pipelines: one for extracting vision BEV features and the other for deriving LiDAR BEV features.

\subsubsection{Vision Stream}
For vision inputs, we adopt the view transformation module in BEVDepth \cite{li2023bevdepth} to aggregate multi-view image features from the perspective view into BEV features. It has three sub-modules: a feature extraction module, a depth distribution prediction module, and a feature aggregation module. 

\textbf{Feature Extraction.} Given a set of multi-view images $\mathcal{I} = \{I_1, I_2, \dots, I_n\}$ where $I_i \in \mathbb{R}^{3 \times H \times W}$ is the image captured by the $ith$ camera, $n$ is the number of views, we leverage ResNet-50~\cite{he2016deep} as the feature extractor $f_{e}(\cdot)$ to encode image features $F = \{F_1, F_2, \dots, F_n\}$, where $F_i = f_{e}(I_i) \in \mathbb{R}^{C_F \times H_F \times W_F}$, $C_F$ is the number of channels, $H_F$ and $W_F$ are the height and width of the feature map.

\textbf{Depth Distribution Prediction.} To lift 2D features into 3D space, we need to predict the depth distribution of the scene. We input the camera intrinsics and extrinsics into the convolutional neural network DepthNet $f_{d}(\cdot)$ proposed in~\cite{li2023bevdepth} to predict the depth distribution $D = \{D_1, D_2, \dots, D_n\}$, where $D_i = f_{d}(F_{I_i}) \in \mathbb{R}^{C_D \times H_F \times W_F}$, $C_D$ is the number of depth bins. The depth distribution $D$ is then normalized to $[0, 1]$ by a sigmoid function. The depth distribution prediction module is trained by minimizing Binary Cross Entropy (BCE) loss.
\begin{equation}
    \mathcal{L}_{d_i} = -\frac{1}{N}
\sum_{j=1}^{H_F}\sum_{k=1}^{W_F}(D_{ijk}^*\log(D_{ijk}) + (1 - D_{ijk}^*)\log(1 - D_{ijk})),
\end{equation}
where $\mathcal{L}_{d_i}$ is the depth loss of the $i$th camera image, $N = H_F \times W_F$ is the total number of pixels, and $D_{i}$ and $D_{i}^*$ are the predicted and ground truth depth distributions of the $i$th camera image. $D_{i}^*$ is generated by projecting the 3D LiDAR points onto the $i$th camera image. Then the total depth loss is $\mathcal{L}_{d} = \sum_{i=1}^{n} \mathcal{L}_{d_i}$.

\textbf{Feature Aggregation.} \blue{Based on the predicted depth distribution $D$, we lift the 2D image features $F$ into 3D frustum features $G = \{G_1, G_2, \dots, G_n\}$, calculated by}
\begin{equation}
    \blue{G_i(u, v) = F_i(u, v) \otimes D_i(u, v),}
\end{equation}
\blue{where $G_i(u, v) \in \mathbb{R}^{C_F \times C_D}$ is the output matrix at the feature pixel $(u, v)$ of the $i$th camera image,} $\otimes$ is the outer product operation. \blue{We then apply several $3 \times 3$ convolution layers to aggregate and refine the frustum features $G$ along the depth axis. These refined features are projected into 3D voxel features $V \in \mathbb{R}^{X \times Y \times Z \times C}$ by efficient voxel pooling. Finally, we reduce the vertical dimension $Z$ to obtain BEV features $B \in \mathbb{R}^{X \times Y \times C}$.} Through depth supervision and view transformation, the scale of BEV features $B$ is almost consistent, \blue{enabling that $B$ maintains a high degree of equivariance}.


\subsubsection{LiDAR Stream}
For LiDAR inputs, we directly convert a 3D LiDAR point cloud into a multi-channel BEV representation with occupancy information. Based on the BEV representation, we extract equivariant BEV features using e2cnn~\cite{e2cnn} detailed below.

\textbf{Multi-channel Occupied BEV.}
Given a 3D LiDAR point cloud $P \in \mathbb{R}^{N \times 3}$, we first remove the ground plane by the z axis and voxelize it into 3D voxels. Then we assign either $0$ (\blue{free}) or $1$ (occupied) to each voxel according to its occupancy, generating a multi-channel occupied BEV $O \in \mathbb{R}^{X \times Y \times C_O}$, where $X$, $Y$, and $C_O$ are \blue{the number of voxels in the x, y, and z axis}, respectively.

\textbf{Equivariant Feature Extraction.}
We apply e2cnn~\cite{e2cnn} to generate equivariant BEV features $B \in \mathbb{R}^{X \times Y \times C}$ from the multi-channel occupied BEV $O$. The e2cnn is a group equivariant convolutional neural network, which is equivariant to the group of 2D Euclidean transformations, namely the E(2) group. The forward pass of e2cnn is formulated as follows:
\blue{\begin{equation}
    B = E(O),
\end{equation}}
\blue{where $E(\cdot)$ denotes the e2cnn operation which comprises group convolutions and pooling. The resultant BEV features $B$ are equivariant to discrete SE(2) transformations}.

\subsection{Rotation Branch}
\label{sec:rotation}
The rotation branch comprises a rotation-equivariant representation module and a rotation estimation module. The former transforms BEV features $B$ into a rotation-equivariant and translation-invariant representation $A \in \mathbb{R}^{X \times Y}$. The latter determines the relative rotation between the query and the map keyframe via 1D circular cross-correlation applied to $A$.

\subsubsection{Rotation Equivariant Representation}
Based on the equivariant BEV features $B$ developed by the BEV generation module, we employ the Radon transform \blue{to each channel of BEV features $B$ independently} to construct a rotation-equivariant sinogram $S \in \mathbb{R}^{X \times Y \times C}$. The Radon transform converts a rotation angle on $B$ into a circular shift on $S$ in the Radon space, as illustrated in Eq.~(\ref{eq:radon_pose}). Then we apply 1D convolutional layers $\phi_{r}(\cdot)$ to $S$ along the \blue{$\theta$} dimension, squeezing the feature channel from $C$ to 1 and generating features $S_r \in \mathbb{R}^{X \times Y}$. To eliminate the effect of large translations, we deploy the Fourier transform along the \blue{$\tau$} axis of $S_r$ and take the magnitude, yielding the rotation-equivariant and translation-invariant representation $A$ according to Theorem~\ref{theorem 1}:
\begin{equation}
    A(\theta, \omega) = \mathcal{A}(\phi_{r}(\mathcal{R}(B(x, y, c)))).
    \label{rot}
\end{equation}
For general convolutional neural networks, there is a nonlinear activation function like ReLU between convolutions. As such operation is pixel-wise in the feature, the equivariance is still reserved.

\subsubsection{Rotation Estimation}
Taking advantage of the rotation equivariance and translation invariance of $A$, we solve the relative rotation $\theta$ between the query $Q$ and the map keyframe $M_i$ by 1D circular cross-correlation. The 1D cross-correlation can be formulated as follows:
\begin{equation}
\begin{aligned}
    \label{eq:rot_corr}
   c_{r_i}(d_{\theta}) &= \mathcal{S}_r(A_{Q}, A_{M_i}, d_{\theta}) \\
   &= A_{Q}(\theta, \omega) \star A_{M_i}(\theta, \omega) \\
   &= \sum_{\theta} \sum_{\omega} A_{Q}(\theta, \omega) A_{M_i}(\theta - d_{\theta}, \omega),
\end{aligned}
\end{equation}
where $c_{r_i}(d_{\theta})$ is the resultant correlation vector parameterized by $d_{\theta}$, $\mathcal{S}_r(\cdot)$ is the similarity function in the rotation branch, $A_{Q}$ and $A_{M_i}$ are the rotation-equivariant and translation-invariant representations of $Q$ and $M_i$, respectively, and $\star$ is the cross-correlation operation. As a result, we can calculate the similarity and the relative rotation angle simultaneously by
\begin{equation}
    \label{eq:rot_sim}
    s_{r_i} = \max_{d_{\theta}} c_{r_i}(d_{\theta}), \quad \hat{\theta} = \argmax_{d_{\theta}} c_{r_i}(d_{\theta}),
\end{equation}
where $s_{r_i}$ and $\hat{\theta}$ are the similarity and estimated rotation angle, respectively. Due to the property of the Radon transform, $c_{r_i}(d_{\theta})$ is a binomial distribution peaking at $\hat{\theta}$ and $\hat{\theta} - \pi$. Therefore, we choose Kullback-Leibler (KL) divergence loss as the rotation estimation loss:
\begin{equation}
\begin{aligned}
&q(d_{\theta}) = \text{softmax}(c_{r_i}(d_{\theta})), \\
&\mathcal{L}_{r} = \sum_{d_{\theta}} p(d_{\theta}) \log \frac{p(d_{\theta})}{q(d_{\theta})}, 
\end{aligned}
\end{equation}
where $q(d_{\theta})$ and $p(d_{\theta})$ are the predicted and ground truth rotation probability distributions. $p(d_{\theta})$ is a binomial gaussian distribution peaking at $\theta^{*}$ (ground truth rotation) and $\theta^{*} - \pi$.

\subsection{Translation Branch}
\label{sec:translation}
To eliminate the rotation effect, we employ the function specified in Eq.~(\ref{eq:rot_com_func}) to the BEV features of $Q$ and $M_i$, which rotates the BEV features of $Q$ by an angle $\theta$, defined as:
\begin{equation}
    \label{eq:theta}
    \theta = \begin{cases}
    \theta^{*}, & \text{in the training phase} \\
    \hat{\theta}, & \text{in the inference phase}.
\end{cases}
\end{equation}

After rotation compensation, we apply 2D convolution layers $\phi_{t}(\cdot)$ to the rotation-compensated BEV features $\widetilde{B}$, resulting in the rotation-invariant and translation-equivariant neural BEV $\widetilde{B}_t \in \mathbb{R}^{X \times Y}$ as stated in Theorem~\ref{theorem 2}. $\widetilde{B}_{t_Q}$ and $\widetilde{B}_{t_{M_i}}$ of $Q$ and $M_i$ are generated by
\begin{equation}
    \widetilde{B}_{t_Q} = \phi_{t}(B_Q(R_{\theta} X)), \quad \widetilde{B}_{t_{M_i}} = \phi_{t}(B_{M_i}),
\end{equation}
where $\theta$ is defined as in Eq.~(\ref{eq:theta}), and $R_{\theta}$ is the associated rotation matrix. Subsequently, to determine the relative translation between $Q$ and $M_i$, we employ 2D cross-correlation:
\begin{equation}
\begin{aligned}
    \label{eq:trans_corr}
   c_{t_i}(d_x, d_y) &= \mathcal{S}_t(\widetilde{B}_{t_Q}, \widetilde{B}_{t_{M_i}}, d_x, d_y) \\
   &= \widetilde{B}_{t_Q}(x, y) \star \widetilde{B}_{t_{M_i}}(x, y) \\
   &= \sum_{x} \sum_{y} \widetilde{B}_{t_Q}(x, y) \widetilde{B}_{t_{M_i}}(x - d_x, y - d_y),
\end{aligned}
\end{equation}
where $c_{t_i}(d_x, d_y)$ is the 2D correlation map and $\mathcal{S}_t(\cdot)$ is the similarity function in the translation branch. Then the similarity $s_{t_i}$ and relative translation $\hat{x}, \hat{y}$ can be estimated simultaneously by
\begin{equation}
    \label{eq:trans_sim}
    s_{t_i} = \max_{d_x, d_y} c_{t_i}(d_x, d_y), \quad \hat{x}, \hat{y} = \argmax_{d_x, d_y} c_{t_i}(d_x, d_y).
\end{equation}

We choose negative log-likelihood (NLL) loss as the translation estimation loss:
\begin{equation}
\begin{aligned}
    & q(d_x, d_y) = \text{softmax}(c_{t_i}(d_x, d_y)), \\
    &\mathcal{L}_{t} = - \log(q(x^{*}, y^{*})),
\end{aligned}
\end{equation}
where $q(d_x, d_y)$ is the predicted translation probability and $(x^{*}, y^{*})^T$ is the ground truth translation vector. \blue{The total loss is $\mathcal{L} =  \lambda_{d} \mathcal{L}_{d} + \lambda_{r} \mathcal{L}_{r} + \lambda_{t} \mathcal{L}_{t}$, where $\lambda_{d}$, $\lambda_{r}$, and $\lambda_{t}$ are the weights of the depth, rotation, and translation losses.}

\subsection{Place Recognition Derived by Pose Estimation}
In a \blue{\textit{PR-by-PE localization}} manner, place recognition is a by-product of pose estimation in our method. By Eq.~(\ref{eq:rot_sim}) and Eq.~(\ref{eq:trans_sim}), we can estimate the similarity $s_{r_i}$ and $s_{t_i}$ between the query $Q$ and each map keyframe $M_i$ in the database $\mathfrak{M}$, which enables place recognition. We select $\mathcal{S}_t(\cdot)$ in Eq.~(\ref{eq:trans_corr}) as the similarity function to recognize places.
\begin{equation}
\begin{aligned}
    \hat{M}_{i}, \hat{x}, \hat{y} &= \argmax_{M_i \in \mathfrak{M}, d_x, d_y} \mathcal{S}_t(\widetilde{B}_{t_Q}, \widetilde{B}_{t_{M_i}}, d_x, d_y) \\
    &= \argmax_{M_i \in \mathfrak{M}, d_x, d_y} \mathcal{S}_t(\phi_t(B_Q(R_{\hat{\theta}} X)), \phi_t(B_{M_i}), d_x, d_y), \\
    \text{s.t. } \hat{\theta} &= \argmax_{d_{\theta}} c_{r_i}(d_{\theta}).
\end{aligned}
\end{equation}
Finally, we re-write this similarity function as
\begin{equation}
\label{eq:gl_sim}
\hat{M}_{i}, \hat{\theta}, \hat{x}, \hat{y} = \argmax_{M_i \in \mathfrak{M}, d_{\theta}, d_x, d_y} \mathcal{S}(B_{Q}, B_{M_i}, d_{\theta}, d_x, d_y),
\end{equation}
which can be regarded as a concrete form of Eq.~(\ref{eq:gl}). This form allows for rotation and translation estimation by correlation-based exhaustive search, making the solver to Eq.~(\ref{eq:gl_sim}) global and efficient, satisfying the desirable properties.


\subsection{Pose Refinement}
\label{sec:pose_refinement}
Upon the estimated pose $\hat{\theta}, \hat{x}, \hat{y}$, we perform additional pose refinement to yield a more accurate pose. In the vision stream, we employ 3-DoF exhaustive matching on the neural BEV in a local range to refine the pose. Specifically, we rotate the query BEV features $B_Q$ by a set of candidate angles $\Theta = \{\theta_1, \theta_2, ..., \theta_m\}$ as inputs of $\phi_{t}(\cdot)$, generating $\{\widetilde{B}_{t_{Q_1}}, \widetilde{B}_{t_{Q_2}}, ..., \widetilde{B}_{t_{Q_m}}\}$. Referring to Eq.~(\ref{eq:trans_corr}) and Eq.~(\ref{eq:trans_sim}), the refined pose is computed by
\begin{equation}
\begin{aligned}
    \label{eq:pose_refine}
   \hat{\theta}, \hat{x}, \hat{y} &= \argmax_{\theta_j \in \Theta, d_x, d_y} \mathcal{S}_t(\widetilde{B}_{t_{Q_j}}, \widetilde{B}_{t_{\hat{M}_{i}}}, d_x, d_y) \\
   &= \argmax_{\theta_j \in \Theta, d_x, d_y} \mathcal{S}_t(\phi_t(B_Q(R_{\theta_j} X)), \phi_t(B_{\hat{M}_{i}}), d_x, d_y),
\end{aligned}
\end{equation}
where $R_{\theta_j}$ is the rotation matrix of $\theta_j$ and $\hat{M}_{i}$ is the retrieved map keyframe by Eq.~(\ref{eq:gl_sim}). In the LiDAR stream, we refine the 3-DoF pose by ICP alignment with FastGICP \cite{koide2021voxelized}. Ultimately, we obtain the localization pose of query $Q$ against map coordinate by Eq.~(\ref{eq:gl_pose}).
\section{Experiments}
In this section, we evaluate our method on the NCLT and Oxford datasets (Sec.~\ref{sec:dataset}) in terms of place recognition (Sec.~\ref{sec:place_recognition}), pose estimation (Sec.~\ref{sec:pose_estimation}), \blue{two-stage global localization evaluation} (Sec.~\ref{sec:two_stage_global_localization}) and \blue{one-stage global localization evaluation} (Sec.~\ref{sec:one_stage_global_localization}) under three evaluation protocols (Sec.~\ref{sec:protocol}), respectively. Moreover, we carry out ablation studies (Sec.~\ref{sec:ablation_study}) to \blue{further investigate the effectiveness of the proposed method}. Finally, we compare the runtime of our approach with other approaches (Sec.~\ref{sec:runtime_model_size}).

\subsection{Datasets}
\label{sec:dataset}
\textbf{NCLT Dataset}~\cite{carlevaris2016university} is a long-term dataset collected by a mobile segway robot in an urban environment. It contains 27 sessions with environmental changes, including weather, illumination, and season changes. The ground truth 6DoF poses are provided by a high-precision RTK GPS system. It contains loops under various rotation changes, which is widely used in the field of global localization. It provides six-view camera images captured by Pointgrey Ladybug3 omnidirectional camera and 3D scans collected by Velodyne HDL-32E. In our experiments, we use five-view camera images as inputs to train RING\#-V since the camera 0 faces the sky, which is useless for localization. Besides, we also crop the images and resize them to 224 $\times$ 384 to save the training memory.

\textbf{Oxford Radar RobotCar Dataset}~\cite{barnes2020oxford} is a large-scale dataset collected by a mobile car mounted on multi-view cameras, LiDAR (Velodyne HDL-32E) and Radar (FMCW) sensors, which is a radar extension of the Oxford Robotcar dataset~\cite{maddern20171}. It covers a large area of Oxford city center and contains multiple sessions with environmental changes in January 2019. The car is equipped with one Point Grey Bumblebee XB3 trinocular camera and three Point Grey Grasshopper2 monocular cameras for 360$^\circ$ vision sensing. Additionally, it utilizes two Velodyne HDL-32E mounted on the left and right of the radar for 3D scene understanding. In our experiments, we leverage four-view camera images captured from the center stereo camera of Point Grey Bumblebee XB3 and three Point Grey Grasshopper2 monocular cameras to train the vision model. Likewise, we crop these images and resize them to 320 $\times$ 640 during image preprocessing. In the LiDAR stream, we concatenate the point clouds collected by the left and right 3D LiDAR sensors into a single point cloud for training and evaluation. Since the ground truth poses of the Oxford dataset are not enough precise, we apply FastGICP~\cite{koide2021voxelized} for ICP refinement to generate more accurate poses.

\subsection{Implementation Details}
\label{sec:experimental_settings}
We implement our method in PyTorch~\cite{paszke2019pytorch} and train it on two NVIDIA GeForce RTX 4090 GPUs. We use the Adam optimizer~\cite{kingma2014adam} with a learning rate of \blue{$1 \times 10^{-3}$} and a weight decay of \blue{$1 \times 10^{-4}$}. We follow the batch size strategy in~\cite{komorowski2021minkloc3d} for batch generation, setting our batch size to 16. \blue{The loss weights $\lambda_{d}$, $\lambda_{r}$, and $\lambda_{t}$ are set to 3.0, 1.0 and 1.0.} We exclusively rely on pose supervision, training our model with data collected within a 25m radius of the current pose for 30 epochs. In the vision stream, we adopt the method in BEVDepth~\cite{li2023bevdepth} to construct the BEV features $B \in \mathbb{R}^{128 \times 128 \times 80}$. The BEV features $B$ represent the spatial range of $[102.4m \times 102.4m]$ with a grid size of $0.8m$. In the LiDAR stream, the multi-channel occupied BEV $O \in \mathbb{R}^{160 \times 160 \times 20}$ and the extracted BEV features $B \in \mathbb{R}^{160 \times 160 \times 128}$ both represent a region of $[140m \times 140m]$ with a grid size of $0.875m$. \blue{We train models on the NCLT and Oxford datasets separately.}


\subsection{Evaluation Protocols}
\label{sec:protocol}
We propose three protocols to evaluate compared methods under different variations: place variation, appearance variation, and both place and appearance variation.
\begin{itemize}
  \item \textbf{Protocol 1: Place Variation.} We split the sessions into training and test sets. Then, we train the model on the split training set and evaluate it on the split test set. The test sessions are collected in the same season and weather conditions as the training sessions. For the NCLT dataset, we choose ``2012-02-04'' as the map session and ``2012-03-17'' as the query session, and then follow~\cite{lu2023deepring} to split the training and test sets. For the Oxford dataset, we split ``2019-01-11-13-24-51'' and ``2019-01-15-13-06-37'' sessions into training and test sets for training and evaluation respectively.
  \item \textbf{Protocol 2: Appearance Variation.} We train the model on several entire sessions and test it on other entire sessions that are not used for training. The test sessions are collected in different seasons and weather conditions. For the NCLT dataset, we select ``2012-02-04'', ``2012-03-17'', ``2012-05-26'', and ``2013-04-05'' sequences for training and ``2012-01-08'', ``2012-08-20'' and ``2012-11-16'' sequences for testing. For the Oxford dataset, we select the entire sequence of ``2019-01-11-13-24-51'' as the map session and the entire sequence of ``2019-01-15-13-06-37'' as the query session for training. In the test phase, ``2019-01-11-14-37-14'' is used as the map session, and ``2019-01-17-12-48-25'' is used as the query session.
  \item \textbf{Protocol 3: Place and Appearance Variation.} We utilize the trained model in Protocol 1 to evaluate the performance on the test sessions in Protocol 2. Specifically, we test all methods on ``2012-01-08'', ``2012-08-20'', and ``2012-11-16'' sequences of Protocol 2 using the model trained on the split ``2012-02-04'' and ``2012-03-17'' sequences of Protocol 1 for the NCLT dataset. Likewise, we test all methods on ``2019-01-11-14-37-14'' and ``2019-01-17-12-48-25'' sequences of Protocol 2 using the model trained on the split ``2019-01-11-13-24-51'' and ``2019-01-15-13-06-37'' sequences of Protocol 1 for the Oxford dataset.
\end{itemize}

In the following experiments, we perform multi-session localization evaluation, where the query and map trajectories are sampled at 5m and 20m intervals, as used in~\cite{lu2022one,lu2023deepring, xu2023ring++}. \blue{Table~\ref{tab:dataset} summarizes the number of training and test samples for each dataset under the three evaluation protocols.}

\begin{table}[htbp]
    \renewcommand\arraystretch{1.1}
    \centering
    \blue{\caption{Datasets for Global Localization Evaluation}
    \label{tab:dataset}
    \begin{tabular}{lccc}
    \toprule[1pt]
    \multirow{2}{*}{}{\begin{tabular}[c]{@{}c@{}}Dataset\end{tabular}} &
    \multirow{2}{*}{}{\begin{tabular}[c]{@{}c@{}}Protocol 1 \\ (\# Train / Test)\end{tabular}} &
    \multirow{2}{*}{}{\begin{tabular}[c]{@{}c@{}}Protocol 2 \\ (\# Train / Test)\end{tabular}} &
    \multirow{2}{*}{}{\begin{tabular}[c]{@{}c@{}}Protocol 3 \\ (\# Train / Test)\end{tabular}} \\ \midrule
    NCLT \cite{carlevaris2016university} & 28331 / 265 & 80099 / 2729 & 28331 / 2729 \\
    Oxford \cite{barnes2020oxford} & 40843 / 274  & 50866 / 1733 & 40843 / 1733 \\
    
    \bottomrule[1pt]
    \end{tabular}}
    \vspace{-0.5cm}
\end{table}
\begin{table*}[htbp]
    \renewcommand\arraystretch{1.1}
    \centering
    \caption{Quantitive Results of Place Recognition of Protocol 1}
    \label{tab:pr_protocol1}
    \begin{threeparttable}
    \begin{tabular}{clccccccc}
    \toprule[1pt]
    \multicolumn{2}{c}{\multirow{2}{*}{Approach}} & \multirow{2}{*}{\blue{Representation}} & \multicolumn{3}{c}{NCLT} & \multicolumn{3}{c}{Oxford} \\ \cline{4-9}
    \multicolumn{3}{c}{} & Recall@1 $\uparrow$ & F1 Score $\uparrow$ & AUC $\uparrow$ & Recall@1 $\uparrow$ & F1 Score $\uparrow$ & AUC $\uparrow$ \\ \hline
    \multirow{8}{*}{Vision} & NetVLAD \cite{arandjelovic2016netvlad} & \blue{PV} & 0.37 & 0.51 & 0.43 & 0.62 & 0.75 & 0.69 \\
    & Patch-NetVLAD \cite{hausler2021patch} & \blue{PV} & 0.41 & 0.54 & 0.43 & 0.67 & 0.78 & 0.73 \\
    & AnyLoc \cite{keetha2023anyloc} & \blue{PV} & 0.47 & 0.60 & 0.42 & 0.73 & 0.83 & 0.81 \\
    & SFRS \cite{ge2020self} & \blue{PV} & 0.50 & 0.62 & 0.54 & 0.74 & 0.83 & 0.86 \\
    & \cellcolor{gray!30}\blue{Exhaustive SS \cite{detone2018superpoint, sarlin2020superglue}}$^{\dagger}$ & \cellcolor{gray!30}\blue{PV} & \cellcolor{gray!30}\blue{0.66} & \cellcolor{gray!30}\blue{0.74} & \cellcolor{gray!30}\blue{\underline{0.82}} & \cellcolor{gray!30}\blue{\textbf{0.86}} & \cellcolor{gray!30}\blue{\textbf{0.91}} & \cellcolor{gray!30}\blue{\textbf{0.95}} \\
    & \blue{BEV-NetVLAD-MLP} & \blue{BEV} & \blue{0.60} & \blue{0.71} & \blue{0.64} & \blue{0.74} & \blue{0.83} & \blue{0.72} \\
    & \blue{vDiSCO \cite{xu2023leveraging}} & \blue{BEV} & \blue{\underline{0.76}} & \blue{\underline{0.82}} & \blue{0.73} & \blue{\underline{0.80}} & \blue{\underline{0.87}} & \blue{0.87} \\
    & \cellcolor{gray!30}\textbf{RING\#-V (Ours)} & \cellcolor{gray!30}\blue{BEV} & \cellcolor{gray!30}\textbf{0.82} & \cellcolor{gray!30}\textbf{0.86} & \cellcolor{gray!30}\textbf{0.93} & \cellcolor{gray!30}\textbf{0.86} & \cellcolor{gray!30}\textbf{0.91} & \cellcolor{gray!30}\underline{0.94} \\ \hline
    \multirow{7}{*}{LiDAR} & OverlapTransformer \cite{ma2022overlaptransformer} & \blue{RI} & 0.71 & 0.78 & 0.76 & 0.71 & 0.81 & 0.70 \\
    & LCDNet \cite{cattaneo2022lcdnet} & \blue{PC} & 0.70 & 0.78 & 0.75 & 0.62 & 0.75 & 0.69 \\
    & DiSCO \cite{xu2021disco} & \blue{Polar BEV} & 0.76 & 0.82 & 0.80 & 0.87 & 0.91 & 0.86 \\
    & RING \cite{lu2022one} & \blue{BEV} & 0.67 & 0.76 & 0.79 & 0.76 & 0.84 & 0.88 \\
    & \blue{RING++ \cite{xu2023ring++}} & \blue{BEV} & \blue{0.68} & \blue{0.77} & \blue{0.78} & \blue{0.83} & \blue{0.89} & \blue{\underline{0.93}} \\
    & EgoNN \cite{komorowski2021egonn} & \blue{PC} & \underline{0.80} & \underline{0.85} & \underline{0.85} & \underline{0.89} & \underline{0.92} & \textbf{0.95} \\
    & \cellcolor{gray!30}\textbf{RING\#-L (Ours)} & \cellcolor{gray!30}\blue{BEV} & \cellcolor{gray!30}\textbf{0.85} & \cellcolor{gray!30}\textbf{0.87} & \cellcolor{gray!30}\textbf{0.91} & \cellcolor{gray!30}\textbf{0.91} & \cellcolor{gray!30}\textbf{0.93} & \cellcolor{gray!30}\underline{0.93} \\
  
    \bottomrule[1pt]
    \end{tabular}
    \begin{tablenotes}
        \footnotesize
        \item[$\dagger$] \blue{SS: Superpoint + SuperGlue, PV: Perspective View, BEV: Bird's-Eye-View, RI: Range Image, PC: Point Cloud. \gray{Gray} rows represent the results of PR-by-PE localization methods.} The best result is highlighted in \textbf{bold} and the second best is \underline{underlined}.
    \end{tablenotes}
    \end{threeparttable}
    \vspace{-0.4cm}
\end{table*}

\begin{figure*}[htbp]
    \centering
    \subfigure[]{
		\includegraphics[trim=1.5cm 1.8cm 0.75cm 1.247cm, clip, width=2.405cm]{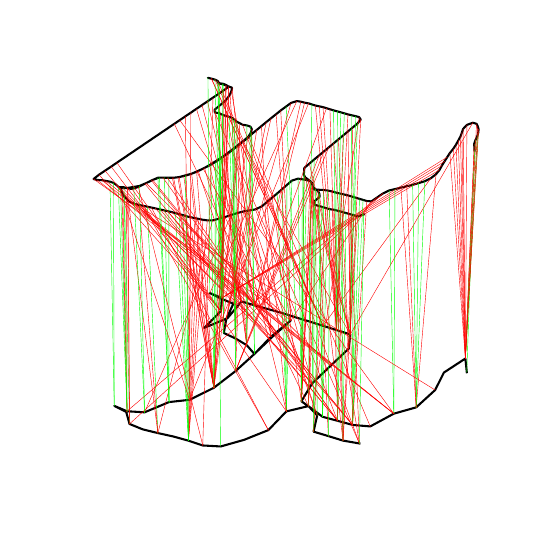}}
    \hspace{-0.515cm}
    \vspace{-0.04cm}
    \subfigure[]{
		\includegraphics[trim=1.5cm 1.8cm 0.75cm 1.247cm, clip, width=2.405cm]{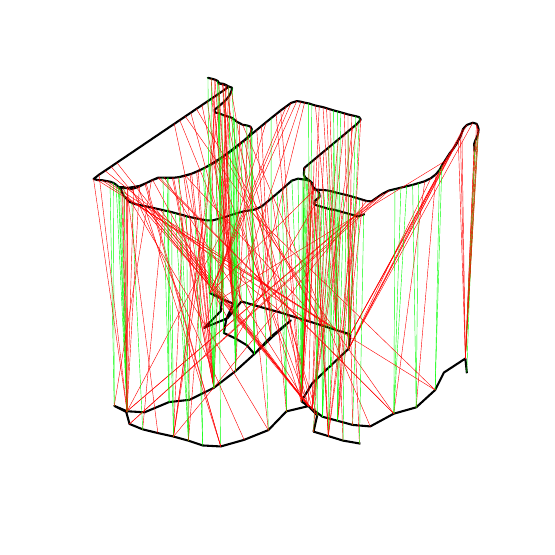}}
    \hspace{-0.515cm}
    \vspace{-0.04cm}
    \subfigure[]{
		\includegraphics[trim=1.5cm 1.8cm 0.75cm 1.247cm, clip, width=2.405cm]{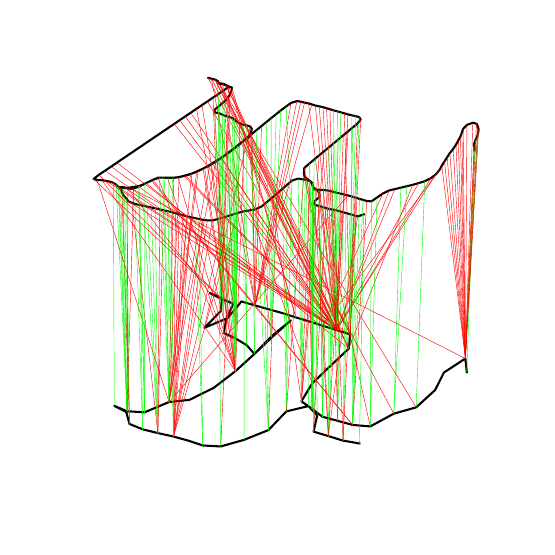}}	
    \hspace{-0.515cm}
    \vspace{-0.04cm}
    \subfigure[]{
		\includegraphics[trim=1.5cm 1.8cm 0.75cm 1.247cm, clip, width=2.405cm]{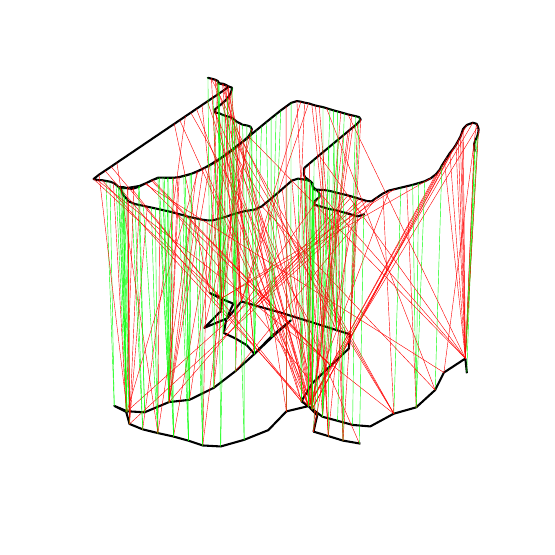}}
    \hspace{-0.515cm}
    \vspace{-0.04cm}
    \subfigure[]{
		\includegraphics[trim=1.5cm 1.8cm 0.75cm 1.247cm, clip, width=2.405cm]{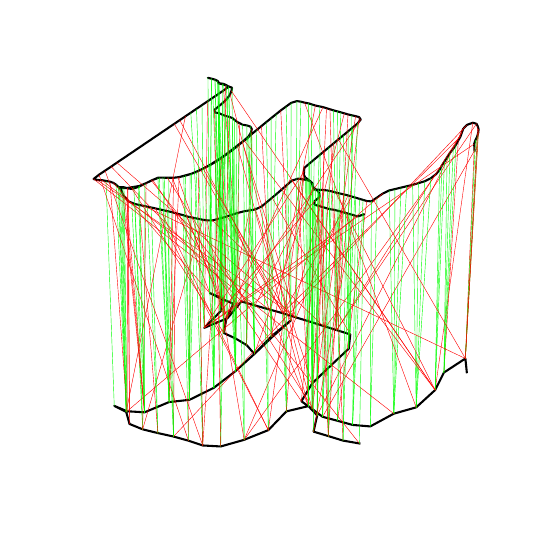}}
    \hspace{-0.515cm}
    \vspace{-0.04cm}    
    \subfigure[]{
		\includegraphics[trim=1.5cm 1.8cm 0.75cm 1.247cm, clip, width=2.405cm]{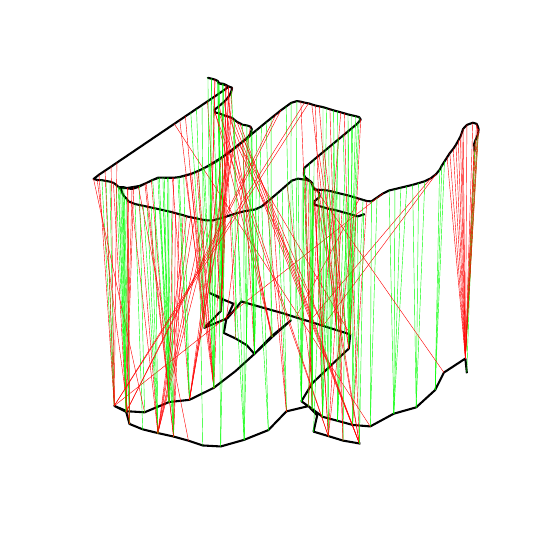}}
    \hspace{-0.515cm}
    \vspace{-0.04cm}
    \subfigure[]{
    \includegraphics[trim=1.5cm 1.8cm 0.75cm 1.247cm, clip, width=2.405cm]{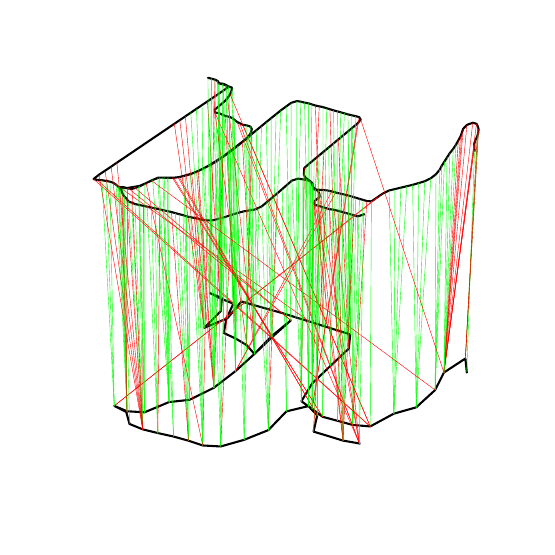}}
    \hspace{-0.515cm}
    \vspace{-0.04cm}
    \subfigure[]{
		\includegraphics[trim=1.5cm 1.8cm 0.75cm 1.247cm, clip, width=2.405cm]{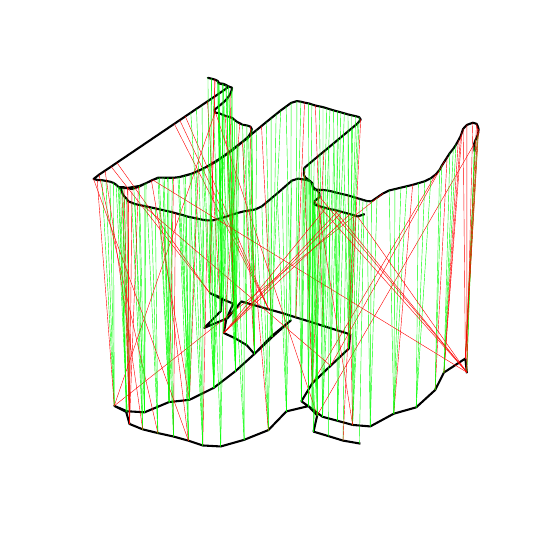}}
    \hspace{-0.515cm}
    \vspace{-0.04cm}
    \subfigure[]{
		\includegraphics[trim=1.5cm 1.8cm 0.75cm 1.247cm, clip, width=2.405cm]{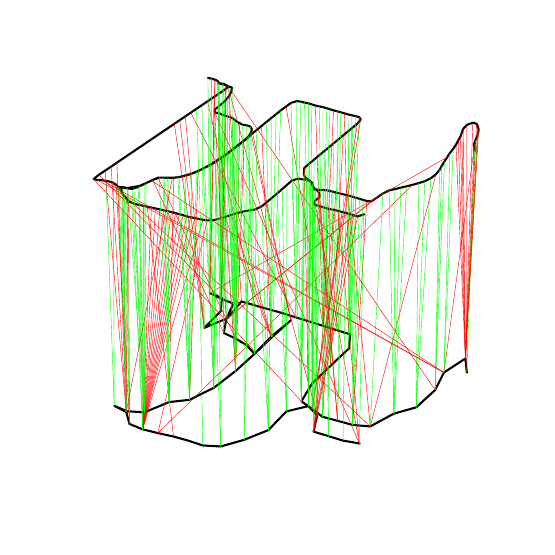}}
    \hspace{-0.515cm}
    \vspace{-0.01cm}
    \subfigure[]{
		\includegraphics[trim=1.5cm 1.8cm 0.75cm 1.247cm, clip, width=2.405cm]{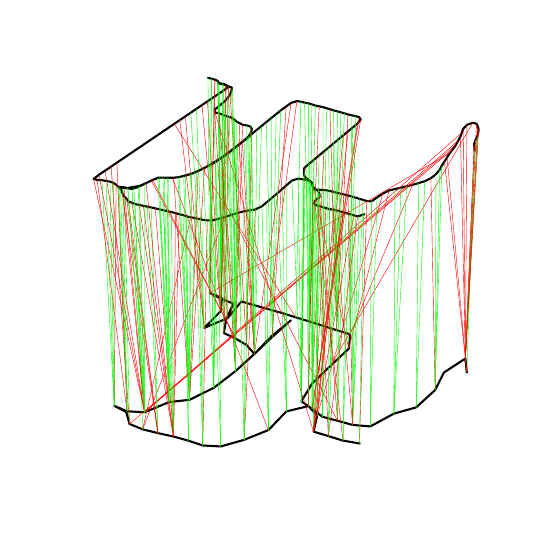}}
    \hspace{-0.515cm}
    \vspace{-0.01cm}
    \subfigure[]{
		\includegraphics[trim=1.5cm 1.8cm 0.75cm 1.247cm, clip, width=2.405cm]{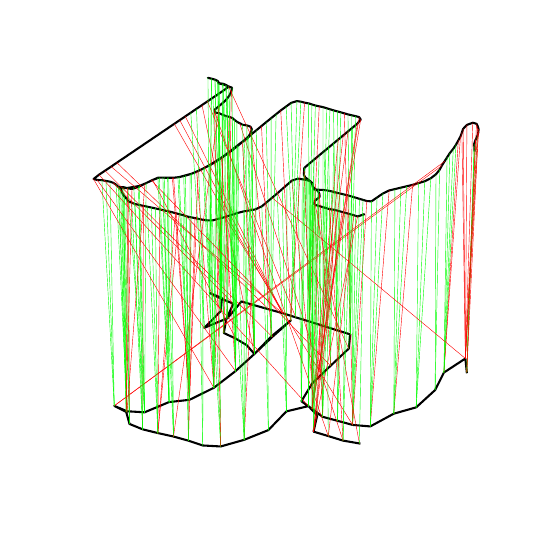}}
    \hspace{-0.515cm}
    \vspace{-0.01cm}
    \subfigure[]{
		\includegraphics[trim=1.5cm 1.8cm 0.75cm 1.247cm, clip, width=2.405cm]{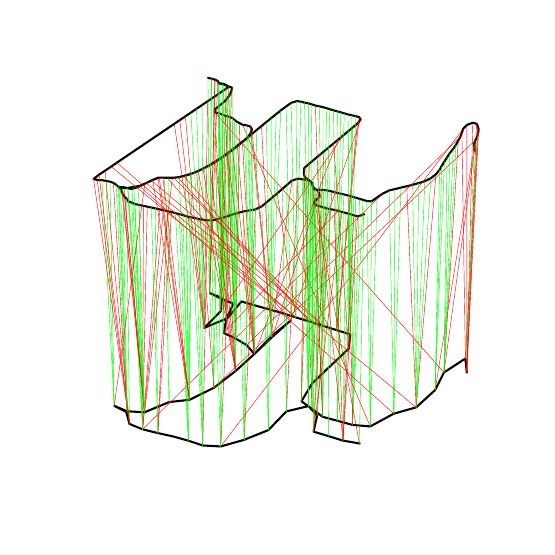}}    
    \hspace{-0.515cm}
    \vspace{-0.01cm}
    \subfigure[]{
		\includegraphics[trim=1.5cm 1.8cm 0.75cm 1.247cm, clip, width=2.405cm]{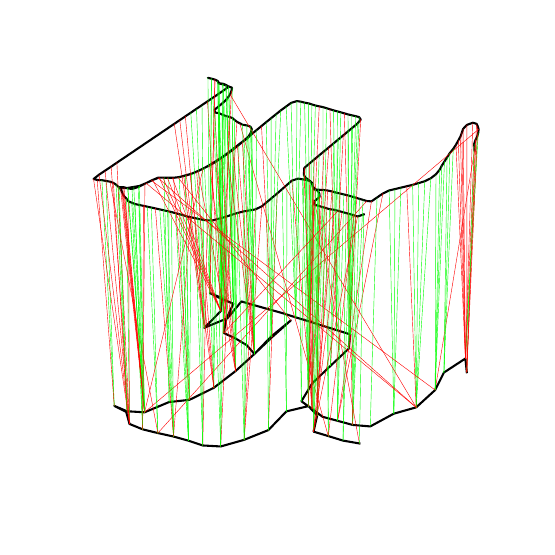}}    
    \hspace{-0.515cm}
    \vspace{-0.01cm}
    \subfigure[]{
		\includegraphics[trim=1.5cm 1.8cm 0.75cm 1.247cm, clip, width=2.405cm]{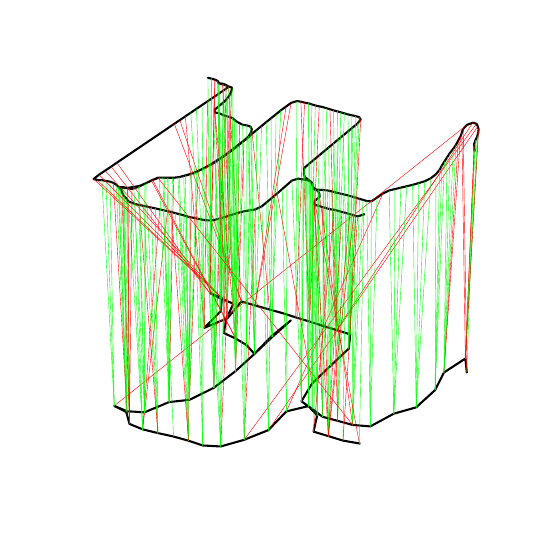}}
    \hspace{-0.515cm}
    \vspace{-0.01cm}
    \subfigure[]{
		\includegraphics[trim=1.5cm 1.8cm 0.75cm 1.247cm, clip, width=2.405cm]{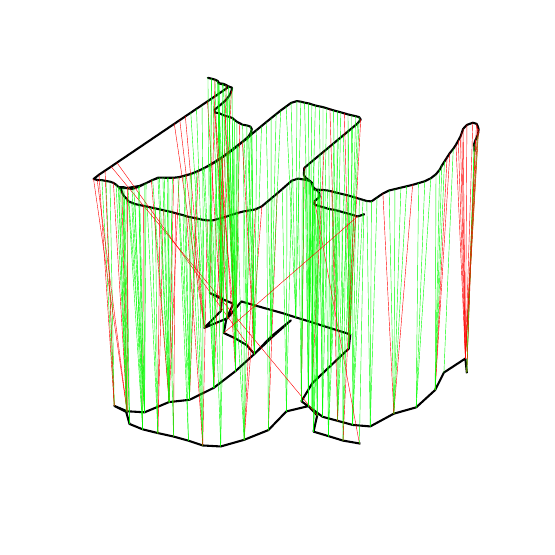}}
    \caption{\blue{\textbf{Top 1 retrieved matches for protocol 1 on the NCLT dataset.} (a) NetVLAD~\cite{arandjelovic2016netvlad}. (b) Patch-NetVLAD~\cite{hausler2021patch}. (c) AnyLoc~\cite{keetha2023anyloc}. (d) SFRS~\cite{ge2020self}. (e) Exhaustive SS~\cite{detone2018superpoint, sarlin2020superglue}. (f) BEV-NetVLAD-MLP. (g) vDiSCO~\cite{xu2023leveraging}. (h) RING\#-V (Ours). (i) OverlapTransformer~\cite{ma2022overlaptransformer}. (j) LCDNet~\cite{cattaneo2022lcdnet}. (k) DiSCO~\cite{xu2021disco}. (l) RING~\cite{lu2022one}. (m) RING++~\cite{xu2023ring++}. (n) EgoNN~\cite{komorowski2021egonn}. (o) RING\#-L (Ours).} The black line \raisebox{0.5ex}{\rule{0.3cm}{0.5pt}} represents the trajectory, the green line {\color{green}{\raisebox{0.5ex}{\rule{0.3cm}{0.5pt}}}} represents the correct retrieval match, and the red line {\color{red}{\raisebox{0.5ex}{\rule{0.3cm}{0.5pt}}}} represents the wrong retrieval match.}
    \label{fig:nclt_matches}
    \vspace{-0.4cm}
\end{figure*}

\subsection{Evaluation of Place Recognition}
\label{sec:place_recognition}
\textbf{Metrics.} \blue{For place recognition evaluation, a revisit threshold $r$ is used to determine whether a retrieved match is correct.} \blue{A retrieval is deemed successful if it lies within this threshold from the query.} \blue{In this experiment, we set the revisit threshold to $r = 10m$.} We leverage five metrics to assess the performance of all methods: 1) Recall@1: the percentage of queries whose top 1 retrieved match is correct; 2) F1 Score\blue{: the harmonic mean of precision (the ratio of true positives to all retrieved matches) and recall (the ratio of true positives to actual positives)} at \blue{various} thresholds\blue{\footnotemark[4]}, \blue{with the maximum F1 score reported}; 3) Precision-Recall Curve: a curve that plots the precision and recall of the retrieval results as the threshold\blue{\footnotemark[4]} changes; 4) AUC (Area Under Curve): the area under the precision-recall curve to quantitatively evaluate the performance of the precision-recall curve; 5) Recall@N: the ratio of queries where at least one of the top N retrieved matches is correct.

\footnotetext[4]{\blue{In a retrieval system, the threshold refers to the similarity score that determines if a retrieved item is considered a positive match.}}

\textbf{Baselines.} \blue{We evaluate our method against a range of state-of-the-art approaches across vision and LiDAR modalities. In the vision domain, we compare against several image matching and retrieval techniques in the PV space, including Exhaustive SS (SuperPoint~\cite{detone2018superpoint} + SuperGlue~\cite{sarlin2020superglue}),} NetVLAD~\cite{arandjelovic2016netvlad}, Patch-NetVLAD~\cite{hausler2021patch}, AnyLoc~\cite{keetha2023anyloc}, and SFRS~\cite{ge2020self}. \blue{Exhaustive SS combines SuperPoint and SuperGlue to exhaustively perform feature matching, selecting the match with the highest number of inliers as the top 1 retrieval. For} AnyLoc, we select the ViT-G AnyLoc-VLAD-DINOv2 model \blue{that uses the foundation model DINOv2~\cite{oquab2023dinov2} for feature extraction}. \blue{To ensure a fair comparison under a multi-camera setup, these image retrieval methods utilize panoramic images as inputs. Additionally, we assess BEV-based approaches BEV-NetVLAD-MLP and vDiSCO~\cite{xu2023leveraging}. BEV-NetVLAD-MLP consists of a shared BEV-based backbone with a NetVLAD head for place recognition and a multi-layer perceptron (MLP) head for pose estimation, where the BEV-based backbone is the same as RING\#, providing a direct comparison in the BEV space.} In the \blue{LiDAR domain}, we compare our method with \blue{six leading approaches:} OverlapTransformer~\cite{ma2022overlaptransformer}, LCDNet~\cite{cattaneo2022lcdnet}, DiSCO~\cite{xu2021disco}, RING~\cite{lu2022one}, \blue{RING++~\cite{xu2023ring++},} and EgoNN~\cite{komorowski2021egonn}. \blue{Except for Exhaustive SS, AnyLoc, and SFRS, for which we use the authors' pre-trained weights, we retrain all other methods using the official implementations on both datasets.}



\begin{figure*}[htbp]
    \centering
    \includegraphics[width=18cm]{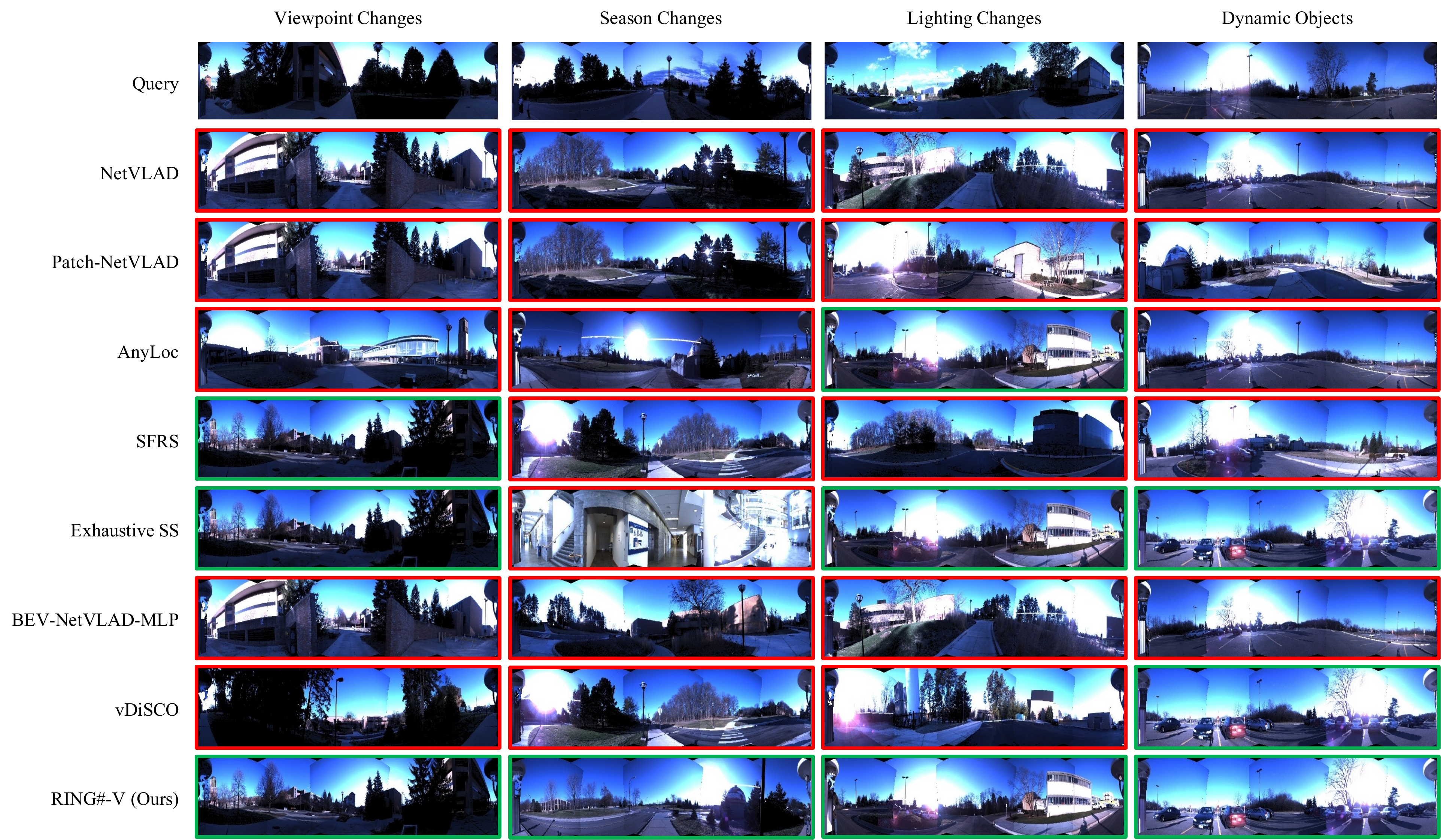}
    \vspace{-0.1cm}
    \caption{\textbf{Qualitative vision examples of some queries and their top 1 retrieved matches on the NCLT dataset.} The red rectangle {\color{red}{$\square$}} represents the wrong retrieval result and the green rectangle {\color{green}{$\square$}} represents the correct retrieval result.}
    \label{fig:case_pr_vision}
    \vspace{-0.5cm}
\end{figure*}

\textbf{Results.} \blue{Table~\ref{tab:pr_protocol1} compares the place recognition performance of all methods under Protocol 1. Apart from Recall@1 in Table~\ref{tab:pr_protocol1}, we report the recall of the top 10 retrieved matches and the precision-recall curve of all methods in Fig.~\ref{fig:nclt_oxford_Recall_N} and Fig.~\ref{fig:nclt_oxford_PR}, respectively. Overall, RING\#-V and RING\#-L demonstrate superior performance among vision- and lidar-based methods, verifying the effectiveness of the proposed \textit{PR-by-PE localization} paradigm. Specifically, we provide several key findings as follows:}
\blue{\begin{itemize}
  \item Among the PV-based methods, AnyLoc and SFRS leverage the foundation model and self-supervised learning, respectively, to extract powerful features for place recognition, outperforming NetVLAD and Patch-NetVLAD on both datasets. However, Exhaustive SS, the only method following the \textit{PR-by-PE localization} paradigm, achieves the best performance across all metrics on both NCLT and Oxford datasets, surpassing the second-best PV-based method, SFRS, by 16\% in Recall@1 on the NCLT dataset. This demonstrates that pose estimation sufficiently improves the performance of place recognition, further validating the effectiveness of the proposed \textit{PR-by-PE localization} paradigm.
  \item BEV-based methods generally perform better than PV-based methods due to the inherent structural awareness of BEV features. Notably, vDiSCO, which explicitly models rotation invariance, surpasses BEV-NetVLAD-MLP by a large margin, particularly on the NCLT dataset that contains more viewpoint changes. However, vDiSCO still lags behind RING\#-V on both datasets as it follows the \textit{PR-then-PE localization} paradigm.
  \item LiDAR-based methods, which benefit from the 3D geometric information invariant to appearance changes and design rotation-invariant global descriptors, have significant advantages over vision-based methods. Compared to these methods, RING\#-L achieves state-of-the-art performance on both datasets. This verifies the effectiveness of our proposed \textit{PR-by-PE localization} paradigm, aligning with the findings in vision-based methods.
  \item While RING\#-L exhibits higher Recall@1 than RING\#-V due to its reliance on explicit geometric point clouds, RING\#-V outperforms RING\#-L in precision-recall curve (Fig.~\ref{fig:nclt_oxford_PR}) and AUC (Table~\ref{tab:pr_protocol1}). The rich texture information captured by cameras enables the model to make more confident and accurate predictions, reducing false positives when classifying similar but distinct places.
\end{itemize}}

We further visualize the top 1 retrieved matches on two distinct trajectories of the NCLT dataset in Fig.~\ref{fig:nclt_matches} \blue{and the Oxford dataset in Appendix~\ref{sec:appendix_pr}}, consistent with the results of Recall@1. Fig.~\ref{fig:case_pr_vision} and Fig.~\ref{fig:case_pr_lidar} display some qualitative results of queries and their top 1 retrieved matches. \blue{Our approach is able to retrieve correct matches in various challenging scenarios, such as large viewpoint changes and seasonal changes, where other methods are prone to fail.} The underlying reason is that \blue{explicitly embedding equivariance into the network enables the framework} to learn patterns invariant to both environment and viewpoint changes. In contrast, compared methods tend to learn patterns that are coupled with changes in environment and viewpoint. 

\begin{figure*}[htbp]
    \centering
    \includegraphics[width=18cm]{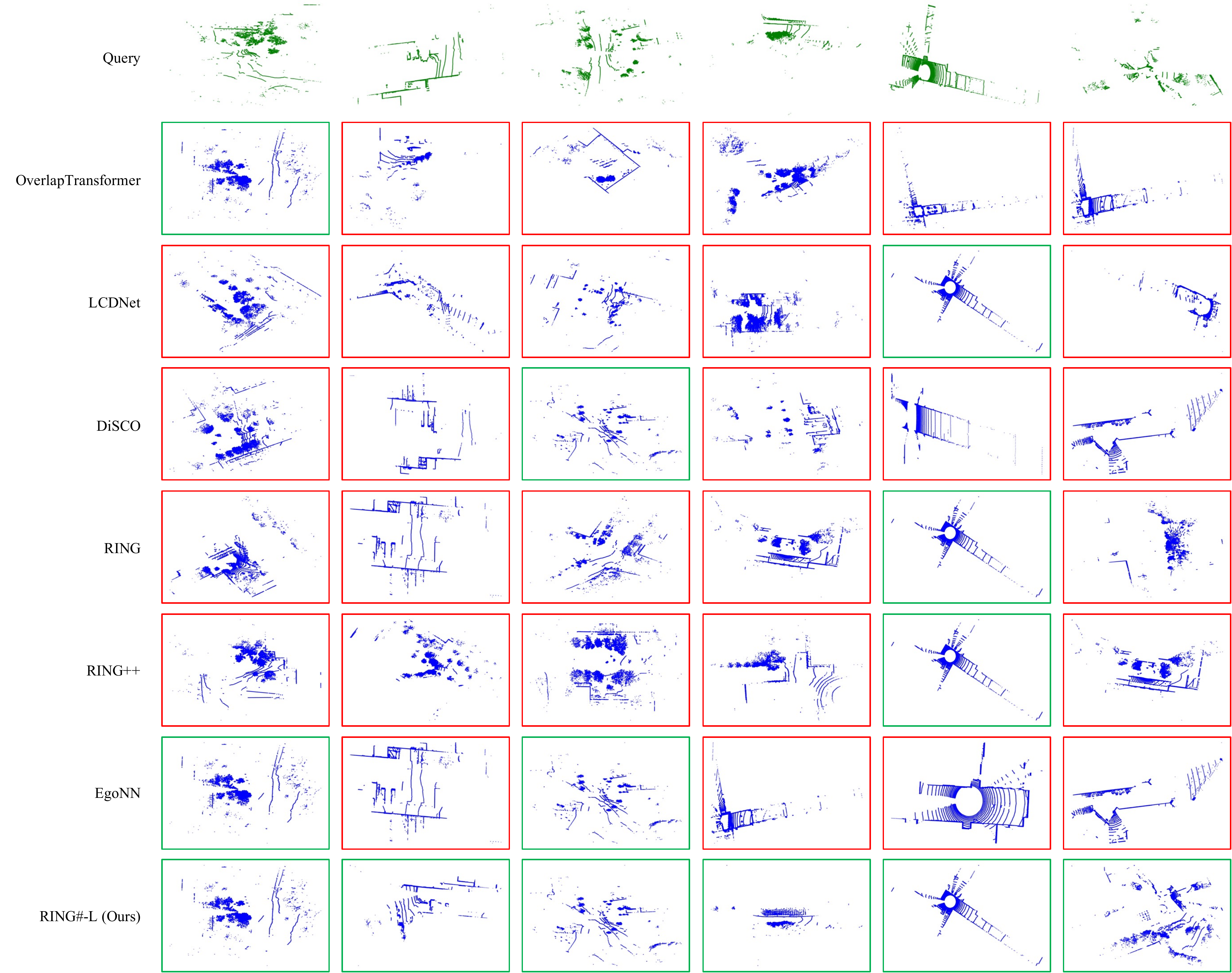}
    \caption{\textbf{Qualitative LiDAR examples of some queries and their top 1 retrieved matches on the NCLT dataset.} The red rectangle {\color{red}{$\square$}} represents the wrong retrieval result and the green rectangle {\color{green}{$\square$}} represents the correct retrieval result.}
    \label{fig:case_pr_lidar}
    \vspace{-0.5cm}
\end{figure*}
\begin{figure*}
    \centering
    \includegraphics[width=18cm]{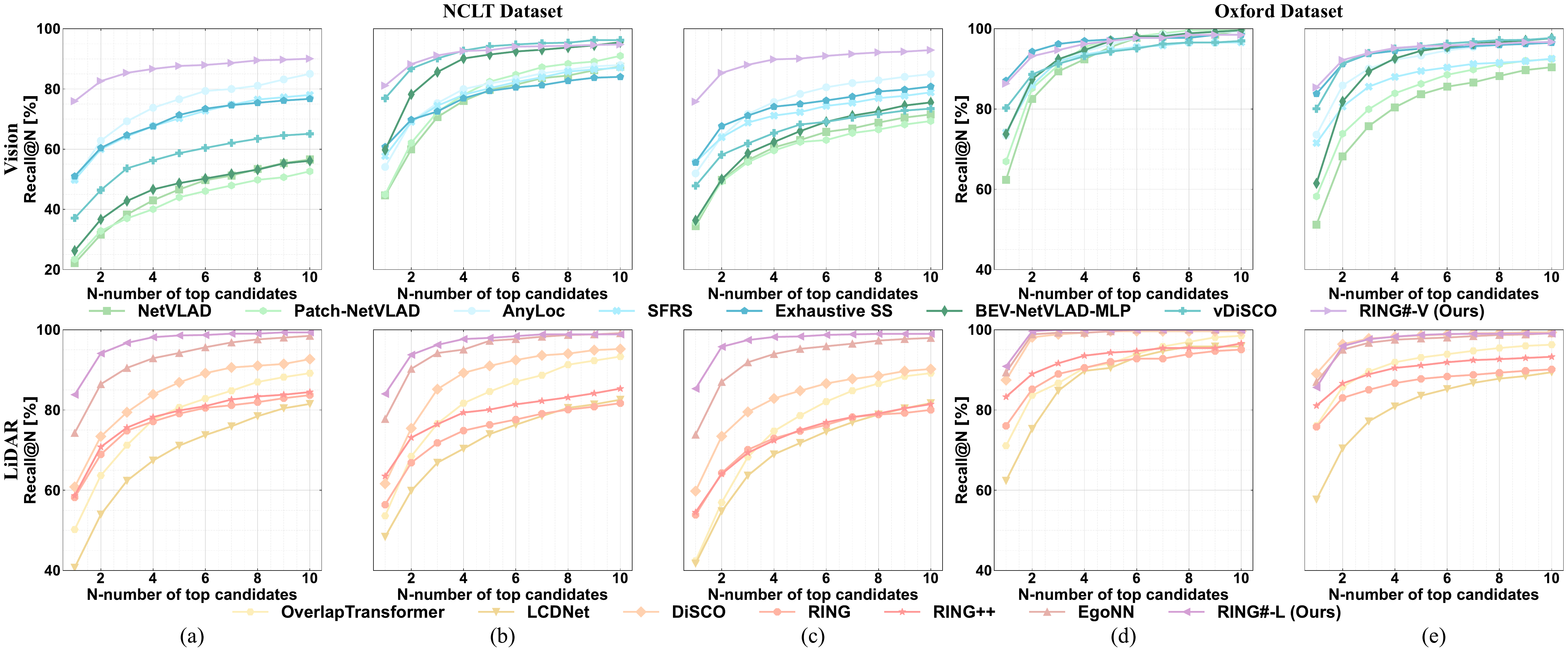}
    \vspace{-0.3cm}
    \caption{\blue{\textbf{Recall@N curves on the NCLT and Oxford datasets.} (a) 2012-01-08 to 2012-08-20. (b) 2012-01-08 to 2012-11-16. (c) 2012-08-20 to 2012-11-16. (d) 2019-01-11-13-24-51 to 2019-01-15-13-06-37. (e) 2019-01-11-13-24-51 to 2019-01-17-12-48-25.}}
    \label{fig:nclt_oxford_Recall_N}
    \vspace{-0.4cm}
\end{figure*}
\begin{figure*}
    \centering
    \includegraphics[width=18cm]{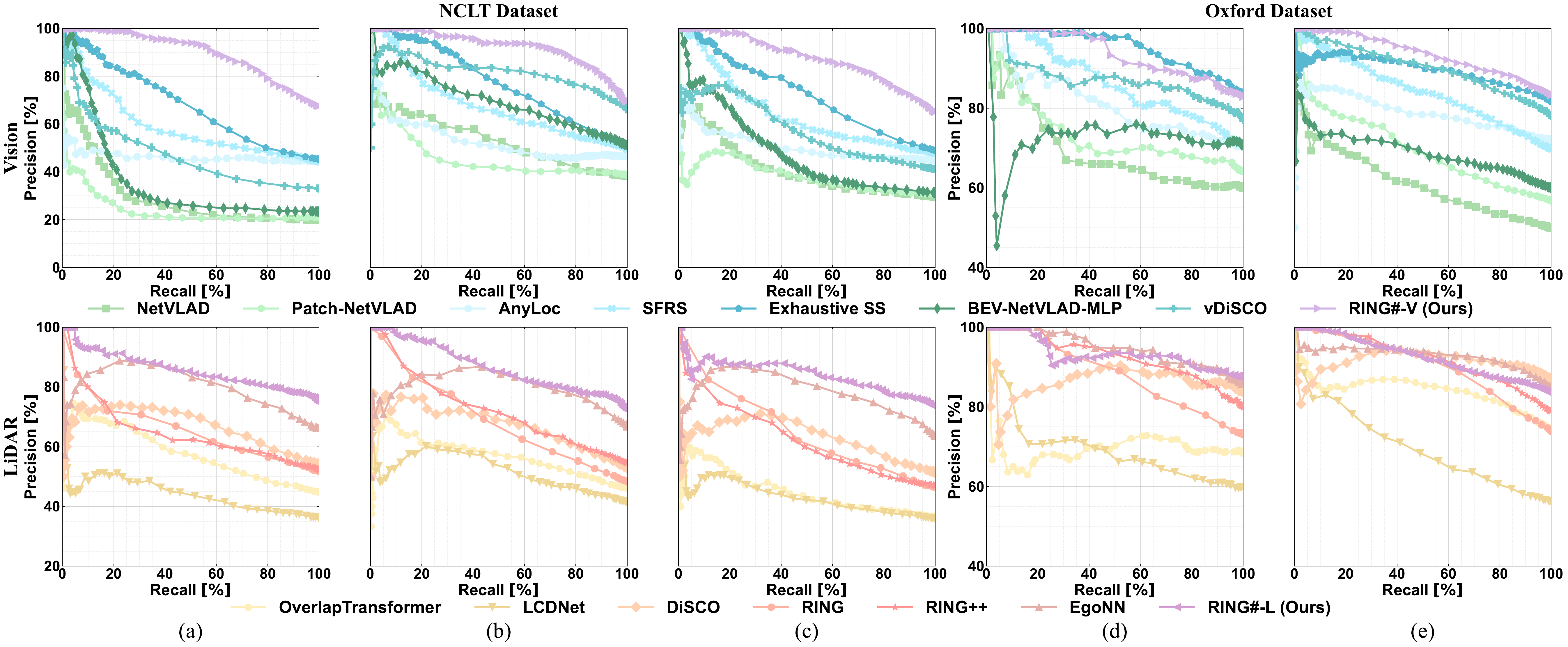}
    \vspace{-0.4cm}
    \caption{\blue{\textbf{Precision-recall curves on the NCLT and Oxford datasets.} (a) 2012-01-08 to 2012-08-20. (b) 2012-01-08 to 2012-11-16. (c) 2012-08-20 to 2012-11-16. (d) 2019-01-11-13-24-51 to 2019-01-15-13-06-37. (e) 2019-01-11-13-24-51 to 2019-01-17-12-48-25.}}
    \label{fig:nclt_oxford_PR}
    \vspace{-0.4cm}
\end{figure*}  

\subsection{Evaluation of Pose Estimation}
\label{sec:pose_estimation}
We perform pure pose estimation evaluation under Protocol 1 without the interference of place recognition to compare the pure pose estimation performance. In detail, we estimate the relative pose of each query with every reference within a 10m radius in the map trajectory.

\textbf{Metrics.} We use three metrics to evaluate the pose estimation performance: 1) RE (Rotation Error), which measures the difference between the estimated and ground truth rotation angle; 2) TE (Translation Error), which measures the difference between the estimated and ground truth translation; 3) PE Succ. (Pose Estimation Success Rate), which calculates the ratio of queries that satisfy $\text{RE} < 5^{\circ}$ and $\text{TE} < 2m$. In this paper, we focus on 3-DoF pose errors (1-DoF RE and 2-DoF TE) \blue{reporting both 50th and 75th percentile errors}.

\textbf{Baselines.} For vision-based methods, we compare RING\#-V with both handcrafted and learning-based \blue{local feature matching} methods. We evaluate SIFT~\cite{lowe2004distinctive} and SuperPoint~\cite{detone2018superpoint} feature extractors with the Nearest Neighbor (NN) and SuperGlue~\cite{sarlin2020superglue} \blue{matchers}. \blue{To perform these methods with multi-view images as inputs, we assign multi-view matched image pairs according to the ground truth rotation angle.} We replace detected 2D keypoints on the reference image with 3D keypoints using the ground truth depth projected from the LiDAR point cloud. After that, we filter out the outliers and estimate the pose transformation with PnP~\cite{kneip2011novel} + RANSAC~\cite{fischler1981random}. \blue{We also assess the pose estimation performance of BEV-NetVLAD-MLP, a BEV-based method that predicts 3-DoF poses.} For \blue{LiDAR-based methods}, we compare RING\#-L with the same methods listed in Sec.~\ref{sec:place_recognition} followed by ICP refinement using FastGICP~\cite{koide2021voxelized}. \blue{Additionally, for methods that} can estimate 3-DoF or 6-DoF poses aside from place recognition, we \blue{evaluate their} pose estimation performance without ICP \blue{refinement to provide a comprehensive comparison}.

\textbf{Results.} We report the pose errors and success rates in Table~\ref{tab:pe_protocol1}. \blue{The findings are summarized as follows:}
\blue{\begin{itemize}
  \item Superpoint + SuperGlue, leveraging learned feature detection and matching, outperforms other vision baselines. 
  \item BEV-NetVLAD-MLP, which directly regresses relative 3-DoF poses from BEV features, suffers from large pose errors due to its lack of interpretability. 
  \item RING\#-V achieves significantly lower 75th percentile pose errors and the highest PE Succ. on both datasets, benefiting from its equivariance design in the BEV space that effectively captures environmental structure for accurate pose estimation. The slightly higher 50th percentile pose errors observed on the Oxford dataset are due to the lower spatial resolution of BEV representations compared to pixel-level image matching methods.
  \item Among the LiDAR baselines capable of predicting 3-DoF or 6-DoF poses without ICP registration, LCDNet and EgoNN rely on local feature matching and the robust RANSAC estimator for pose estimation. Despite this, they are outperformed by RING and RING++, which employ a globally convergent pose solver. RING\#-L further enhances this by introducing learnable equivariant feature extraction, significantly boosting the discriminative power of the features and delivering excellent performance.
  \item After ICP refinement, all approaches present better performance. However, OverlapTransformer and DiSCO suffer from lower PE Succ. since they do not estimate relative poses or only predict 1-DoF rotations, which provide poor initial poses for ICP. In contrast, RING\#-L maintains superior PE Succ. both with and without ICP, demonstrating its robustness and global convergence.
\end{itemize}}

Furthermore, we provide the qualitative results that visualize the pose estimation process of RING\#-V and RING\#-L in Fig.~\ref{fig:case_pe}. As we can see, the neural BEV of RING\#-V and RING\#-L \blue{reveals a pattern consistent with the input LiDAR point cloud, highlighting the strong equivariance and geometric awareness of the neural BEV, which accounts for the superior performance of RING\# in pose estimation.}

\begin{table*}
    \renewcommand\arraystretch{1.1}
    \centering
    \caption{Quantitive Results of Pose Estimation of Protocol 1}
    \label{tab:pe_protocol1}
    \begin{threeparttable}
    \begin{tabular}{clcccccc}
    \toprule[1pt]
    \multicolumn{2}{c}{\multirow{2}{*}{Approach}} & \multicolumn{3}{c}{NCLT} & \multicolumn{3}{c}{Oxford} \\ \cline{3-8}
    \multicolumn{2}{c}{} & RE [\textdegree] $\downarrow$ & TE [m] $\downarrow$ & PE Succ. $\uparrow$ & RE [\textdegree] $\downarrow$ & TE [m] $\downarrow$ & PE Succ. $\uparrow$ \\ \hline
    
    \multirow{5}{*}{Vision} & \blue{{SIFT \cite{lowe2004distinctive} + NN}$^{\dagger}$} & \blue{40.36} / \blue{126.99} & \blue{8.90} / \blue{23.72} & \blue{0.11} & \blue{\underline{0.63}} / \blue{2.87} & \blue{\underline{0.68}} / \blue{3.84} & \blue{0.63} \\
    & \blue{{SuperPoint \cite{detone2018superpoint} + NN}$^{\dagger}$} & \blue{24.76} / \blue{171.14} & \blue{5.12} / \blue{7.84} & \blue{0.22} & \blue{0.84} / \blue{4.07} & \blue{0.85} / \blue{4.50} & \blue{0.58} \\
    & \blue{SuperPoint \cite{detone2018superpoint} + SuperGlue \cite{sarlin2020superglue}} & \blue{\underline{3.31}} / \blue{\underline{9.02}} & \blue{\underline{2.34}} / \blue{\underline{5.68}} & \blue{\underline{0.43}} & \blue{\textbf{0.59}} / \blue{\underline{2.51}} & \blue{\textbf{0.45}} / \blue{\underline{2.62}} & \blue{\underline{0.69}} \\
    & \blue{BEV-NetVLAD-MLP} & \blue{57.98} / \blue{121.56} & \blue{5.82} / \blue{8.24} & \blue{0.01} & \blue{6.71} / \blue{14.60} & \blue{6.75} / \blue{12.07} & \blue{0.05} \\
    & \textbf{RING\#-V (Ours)} & \textbf{1.25} / \textbf{2.22} & \textbf{0.71} / \textbf{1.29} & \textbf{0.85} & 0.76 / \textbf{1.48} & 0.76 / \textbf{1.85} & \textbf{0.75} \\ \hline
    \multirow{5}{*}{LiDAR w/o ICP} & LCDNet \cite{cattaneo2022lcdnet} & 3.91 / 9.04 & 3.47 / 5.65 & 0.25 & 3.35 / 8.49 & 5.11 / 7.91 & 0.14 \\
    & RING \cite{lu2022one} & 1.37 / 2.36 & \underline{0.56} / \textbf{0.83} & 0.88 & 0.79 / 1.49 & 0.60 / 1.03 & 0.78 \\
    & \blue{RING++ \cite{xu2023ring++}} & \blue{\underline{1.30}} / \blue{\underline{2.28}} & \blue{0.58} / \blue{0.88} & \blue{\underline{0.91}} & \blue{0.78} / \blue{1.41} & \blue{\underline{0.55}} / \blue{\underline{0.93}} & \blue{\underline{0.83}} \\
    & EgoNN \cite{komorowski2021egonn} & 1.57 / 4.43 & \textbf{0.40} / 2.02 & 0.71 & \textbf{0.39} / \textbf{0.96} & 0.56 / 4.78 & 0.63 \\
    & \textbf{RING\#-L (Ours)} & \textbf{1.13} / \textbf{1.85} & 0.62 / \underline{0.86} & \textbf{0.97} & \underline{0.54} / \underline{0.98} & \textbf{0.51} / \textbf{0.80} & \textbf{0.87} \\ \hline
    \multirow{7}{*}{LiDAR w/ ICP} & OverlapTransformer \cite{ma2022overlaptransformer} + ICP \cite{koide2021voxelized} & 86.36 / 172.68 & 4.29 / 8.08 & 0.31 & 0.01 / 0.22 & 0.02 / 5.58 & 0.54 \\
    & LCDNet \cite{cattaneo2022lcdnet} + ICP \cite{koide2021voxelized} & 1.23 / 2.70 & 0.22 / 4.05 & 0.67 & \textbf{0.00} / \underline{0.01} & \textbf{0.00} / 1.19 & 0.74 \\
    & DiSCO \cite{xu2021disco} + ICP \cite{koide2021voxelized} & 1.14 / 2.16 & 0.19 / 1.39 & 0.75 & 0.01 / 0.06 & 0.02 / 5.42 & 0.60 \\
    & RING \cite{lu2022one} + ICP \cite{koide2021voxelized} & 1.05 / 1.78 & \textbf{0.14} / 0.25 & 0.92 & \textbf{0.00} / \underline{0.01} & \textbf{0.00} / 0.06 & 0.79 \\
    & \blue{RING++ \cite{xu2023ring++} + ICP \cite{koide2021voxelized}} & \blue{\underline{1.03}} / \blue{\underline{1.77}} & \blue{\textbf{0.14}} / \blue{\underline{0.23}} & \blue{\underline{0.95}} & \blue{\textbf{0.00}} / \blue{\underline{0.01}} & \blue{\textbf{0.00}} / \blue{\textbf{0.01}} & \blue{\underline{0.83}} \\
    & EgoNN \cite{komorowski2021egonn} + ICP \cite{koide2021voxelized} & 1.15 / 2.32 & 0.15 / 0.34 & 0.79 & \textbf{0.00} / 0.02 & \textbf{0.00} / 4.59 & 0.67 \\
    & \textbf{RING\#-L (Ours) + ICP \cite{koide2021voxelized}} & \textbf{0.99} / \textbf{1.68} & \textbf{0.14} / \textbf{0.22} & \textbf{0.97} & \textbf{0.00} / \textbf{0.00} & \textbf{0.00} / \textbf{0.01} & \textbf{0.87} \\
    
    \bottomrule[1pt]
    \end{tabular}
    \begin{tablenotes}
        \footnotesize
        \item[$\dagger$] \blue{NN: Nearest Neighbor.} \blue{We report 50th / 75th percentile errors for RE and TE.} The best result is highlighted in \textbf{bold} and the second best is \underline{underlined}.
    \end{tablenotes}
    \end{threeparttable}
    \vspace{-0.3cm}
\end{table*}
\begin{figure*}[ht!]
    \centering
    \includegraphics[width=18cm]{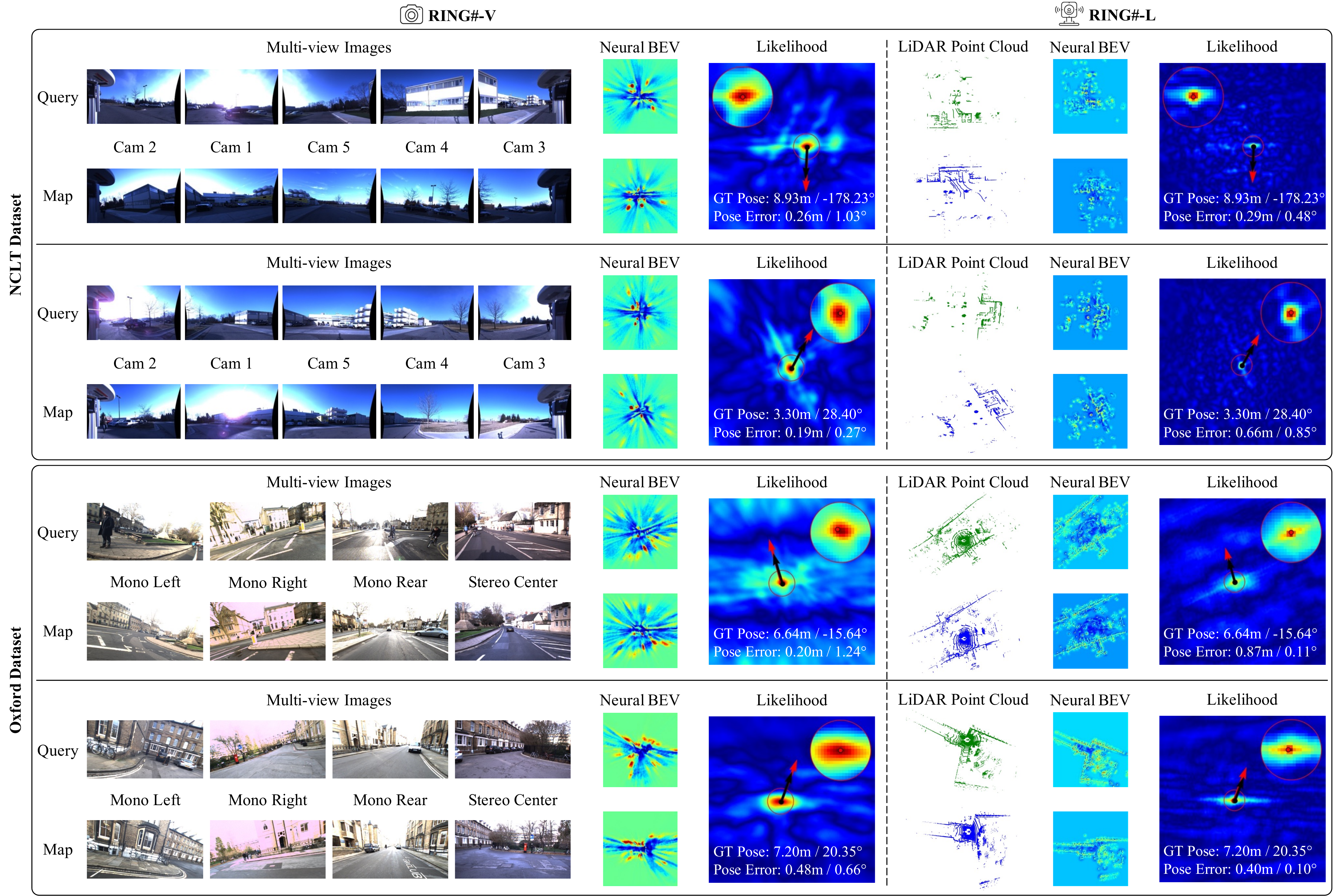}
    \vspace{-0.2cm}
    \caption{\textbf{Qualitative localization results of RING\#-V and RING\#-L.} We display several localization cases of RING\# on the NCLT and Oxford datasets. On the likelihood plot, the black arrow $\rightarrow$ shows the pose estimated by RING\#, the red arrow {\color{red}{$\rightarrow$}} shows the ground truth pose and the red dot {\color{red}{$\bullet$}} shows the ground truth position. Here we rotate the query neural BEV by $\hat{\theta}$ estimated using RING\# to visualize the neural BEV under a 3-DoF pose transformation.}
    \label{fig:case_pe}
    \vspace{-0.3cm}
\end{figure*}

\subsection{\blue{Two-stage Evaluation of Global Localization}}
\label{sec:two_stage_global_localization}
In this subsection, we evaluate global localization performance \blue{in two stages where the success rate of global localization is determined by the success rate of pose estimation, conditioned on the success of place recognition.}

\textbf{Metrics.}
In addition to the evaluation metrics (RE, TE, and PE Succ.) used in Sec.~\ref{sec:pose_estimation}, we introduce one more metric for global localization evaluation: GL Succ., which is defined as the percentage of queries that are correctly localized ($\text{RE} < 5^{\circ}$ and $\text{TE} < 2m$). For \blue{two-stage evaluation of} global localization, GL Succ. is a compound metric of place recognition and pose estimation, formulated as $\text{GL Succ.} = \text{Recall@1} \times \text{PE Succ.}$.

\textbf{Baselines.}
We combine all visual place recognition methods in Sec.~\ref{sec:place_recognition} with pose estimation methods SuperPoint + SuperGlue (abbreviated as SS) as the vision baselines. The LiDAR baselines are composed of LiDAR place recognition methods in Sec.~\ref{sec:place_recognition} followed by ICP registration. 

\textbf{Results.} \blue{A revisit threshold of $r = 10m$ is used in these experiments. The global localization results for all protocols are summarized in Table~\ref{tab:gl_protocol1} to Table~\ref{tab:gl_protocol3}. Our method consistently outperforms other methods across all protocols and datasets. Notably, RING\#-V even outperforms most LiDAR-based methods, reinforcing our claim that the \textit{PR-by-PE localization} paradigm is more effective than the \textit{PR-then-PE localization} paradigm. Key findings include:}
\blue{\begin{itemize}
  \item vDiSCO + SS employs the same pose estimation technique as Exhaustive SS but shows a lower PE Succ. across all protocols, with a notable 17\% drop under Protocol 2. This discrepancy highlights the inconsistency within the \textit{PR-then-PE localization} paradigm, where the success of pose estimation and place recognition diverges due to their inherently different objectives. In contrast, Exhaustive SS, which adheres to the \textit{PR-by-PE localization} paradigm, exhibits the best overall performance among vision baselines.
  \item Across the three protocols, most methods experience a decline in GL Succ. under Protocol 3, which is more challenging due to coupled place and appearance changes. For example, vDiSCO + SS shows a 13\% drop in GL Succ. from Protocol 1 to Protocol 3 on the NCLT dataset, and EgoNN + ICP sees a 9\% decrease. In contrast, RING\# exhibits only a 3\% drop for the vision modality and a 1\% drop for the LiDAR modality, which indicates that our approach performs well even in these challenging scenarios, showcasing strong generalization ability.
  \item EgoNN, a joint \textit{PR-then-PE localization} network that extracts both global and local descriptors for place recognition and pose estimation, shows competitive performance. However, RING\#-L, specifically designed for \textit{PR-by-PE localization}, generally surpasses EgoNN, except under Protocol 2 on the Oxford dataset, where EgoNN shows a slight edge. This exception will be discussed further.
  \item The PE Succ. of RING\# is almost the same and consistently exceeds other methods across all protocols, emphasizing the effectiveness of our globally convergent pose estimation network. However, the GL Succ. of RING\#-L in Protocol 2 is lower than in Protocol 3, despite more training data. This suggests that the drop in GL Succ. is mainly due to place recognition performance degradation.
\end{itemize}}

\textbf{\bluei{Impact of Different Revisit Thresholds.}} \blue{Regarding the last two findings mentioned above,} we hypothesize that the reason is that \textit{the \blue{strict} revisit threshold of place recognition rejects some potential queries that can be correctly localized.} To verify this hypothesis, we \blue{analyze pose errors and global localization success rates at different revisit thresholds (5m, 10m, 20m, 25m) as shown in Fig.~\ref{fig:nclt_vision_PE}} to study the impact of revisit thresholds. \blue{The results reveal that pose errors (RE and TE) for all baselines increase significantly with higher revisit thresholds}. \blue{Nonetheless,} RING\#-V and RING\#-L almost maintain \blue{nearly} constant RE and TE as the revisit threshold increases. \blue{Furthermore, RING\# achieves a substantial improvement in GL Succ. with higher revisit thresholds. This indicates that the $r = 10m$ revisit threshold used in previous experiments is too conservative. RING\# could achieve even better performance with a higher revisit threshold. Additional results in Appendix~\ref{sec:appendix_two_stage_gl} further support this observation.} 

\begin{table*}[htbp]
    \renewcommand\arraystretch{1.1}
    \centering
    \caption{Quantitative Results of Global Localization of Protocol 1}
    \label{tab:gl_protocol1}
    \resizebox{\textwidth}{!}{
    \begin{threeparttable}
    \begin{tabular}{clcccccc}
    \toprule[1pt]
    \multicolumn{2}{c}{\multirow{2}{*}{Approach}} & \multicolumn{3}{c}{NCLT} & \multicolumn{3}{c}{Oxford} \\ \cline{3-8}
    \multicolumn{2}{c}{} & \multicolumn{1}{c}{GL / PE Succ. $\uparrow$} & \multicolumn{1}{c}{RE [\textdegree] $\downarrow$} & \multicolumn{1}{c}{TE [m] $\downarrow$} & \multicolumn{1}{c}{GL / PE Succ. $\uparrow$} & \multicolumn{1}{c}{RE [\textdegree] $\downarrow$} & \multicolumn{1}{c}{TE [m] $\downarrow$} \\ \hline
    
    \multirow{8}{*}{Vision} & \blue{NetVLAD \cite{arandjelovic2016netvlad} + {SS \cite{detone2018superpoint,sarlin2020superglue}}$^{\dagger}$} & \blue{0.20} / \blue{0.52} & \blue{2.92} / \blue{6.79} & \blue{1.71} / \blue{4.07} & \blue{0.51} / \blue{0.81} & \blue{0.38} / \blue{1.16} & \blue{0.37} / \blue{0.93} \\
    & \blue{Patch-NetVLAD \cite{hausler2021patch} + {SS \cite{detone2018superpoint,sarlin2020superglue}}$^{\dagger}$} & \blue{0.22} / \blue{0.54} & \blue{2.63} / \blue{4.92} & \blue{1.72} / \blue{3.84} & \blue{0.55} / \blue{0.82} & \blue{\underline{0.35}} / \blue{1.06} & \blue{\underline{0.35}} / \blue{0.81} \\
    & \blue{AnyLoc \cite{keetha2023anyloc} + {SS \cite{detone2018superpoint,sarlin2020superglue}}$^{\dagger}$} & \blue{0.28} / \blue{0.59} & \blue{\underline{2.15}} / \blue{4.95} & \blue{1.19} / \blue{3.51} & \blue{0.59} / \blue{0.81} & \blue{0.36} / \blue{1.36} & \blue{\textbf{0.34}} / \blue{1.02} \\
    & \blue{SFRS \cite{ge2020self} + {SS \cite{detone2018superpoint,sarlin2020superglue}}$^{\dagger}$} & \blue{0.30} / \blue{\underline{0.61}} & \blue{2.57} / \blue{\underline{4.77}} & \blue{\underline{1.08}} / \blue{\underline{3.34}} & \blue{0.61} / \blue{0.82} & \blue{0.36} / \blue{1.11} & \blue{\underline{0.35}} / \blue{0.96} \\
    & \cellcolor{gray!30}\blue{Exhaustive {SS \cite{detone2018superpoint,sarlin2020superglue}}$^{\dagger}$} & \cellcolor{gray!30}\blue{\underline{0.38}} / \blue{0.57} & \cellcolor{gray!30}\blue{2.71} / \blue{5.32} & \cellcolor{gray!30}\blue{1.45} / \blue{3.45} & \cellcolor{gray!30}\blue{\underline{0.73}} / \blue{\underline{0.84}} & \cellcolor{gray!30}\blue{\textbf{0.34}} / \blue{\textbf{0.89}} & \cellcolor{gray!30}\blue{\underline{0.35}} / \blue{\textbf{0.77}} \\
    & \blue{BEV-NetVLAD-MLP} & \blue{0.00} / \blue{0.01} & \blue{61.05} / \blue{121.83} & \blue{5.16} / \blue{7.48} & \blue{0.05} / \blue{0.06} & \blue{5.77} / \blue{10.64} & \blue{7.08} / \blue{11.47} \\
    & \blue{vDiSCO \cite{xu2023leveraging} + {SS \cite{detone2018superpoint,sarlin2020superglue}}$^{\dagger}$} & \blue{\underline{0.38}} / \blue{0.50} & \blue{2.93} / \blue{7.18} & \blue{1.74} / \blue{4.05} & \blue{0.66} / \blue{0.82} & \blue{\underline{0.35}} / \blue{\underline{0.90}} & \blue{0.36} / \blue{\underline{0.78}} \\
    & \cellcolor{gray!30}\textbf{RING\#-V (Ours)} & \cellcolor{gray!30}\textbf{0.75} / \textbf{0.91} & \cellcolor{gray!30}\textbf{1.11} / \textbf{2.02} & \cellcolor{gray!30}\textbf{0.65} / \textbf{1.03} & \cellcolor{gray!30}\textbf{0.78} / \textbf{0.90} & \cellcolor{gray!30}1.71 / 2.48 & \cellcolor{gray!30}0.68 / 1.07 \\ \hline
    \multirow{7}{*}{LiDAR} & OverlapTransformer \cite{ma2022overlaptransformer} + ICP \cite{koide2021voxelized} & 0.27 / 0.38 & 47.26 / 172.39 & 3.70 / 7.38 & 0.48 / 0.68 & 0.01 / 0.06 & 0.01 / 4.49 \\
    & LCDNet \cite{cattaneo2022lcdnet} + ICP \cite{koide2021voxelized} & 0.50 / 0.72 & 1.21 / 1.95 & 0.17 / 2.81 & 0.54 / 0.87 & \textbf{0.00} / 0.01 & \textbf{0.00} / 0.01 \\
    & DiSCO \cite{xu2021disco} + ICP \cite{koide2021voxelized} & 0.62 / 0.81 & \underline{0.97} / 1.67 & 0.15 / 0.35 & 0.60 / 0.68 & 0.01 / 0.04 & 0.01 / 4.49 \\
    & RING \cite{lu2022one} + ICP \cite{koide2021voxelized} & 0.65 / \underline{0.97} & 0.98 / \underline{1.64} & 0.13 / \underline{0.20} & 0.71 / 0.94 & \textbf{0.00} / \textbf{0.00} & \textbf{0.00} / \textbf{0.00} \\
    & \blue{RING++ \cite{xu2023ring++} + ICP \cite{koide2021voxelized}} & \blue{0.66} / \blue{0.96} & \blue{0.98} / \blue{\underline{1.64}} & \blue{\textbf{0.12}} / \blue{0.21} & \blue{\underline{0.81}} / \blue{\underline{0.98}} & \blue{\textbf{0.00}} / \blue{\textbf{0.00}} & \blue{\textbf{0.00}} / \blue{\textbf{0.00}} \\
    & EgoNN \cite{komorowski2021egonn} + ICP \cite{koide2021voxelized} & \underline{0.76} / 0.95 & 0.98 / 1.68 & \textbf{0.12} / 0.21 & 0.79 / 0.89 & \textbf{0.00} / \textbf{0.00} & \textbf{0.00} / \textbf{0.00} \\
    & \cellcolor{gray!30}\textbf{RING\#-L (Ours) + ICP \cite{koide2021voxelized}} & \cellcolor{gray!30}\textbf{0.83} / \textbf{0.98} & \cellcolor{gray!30}\textbf{0.94} / \textbf{1.53} & \cellcolor{gray!30}\textbf{0.12} / \textbf{0.18} & \cellcolor{gray!30}\textbf{0.90} / \textbf{0.99} & \cellcolor{gray!30}\textbf{0.00} / \textbf{0.00} & \cellcolor{gray!30}\textbf{0.00} / \textbf{0.00} \\
    
    \bottomrule[1pt]
    \end{tabular}
    \begin{tablenotes}
        \footnotesize
        \item[$\dagger$] \blue{SS: SuperPoint + SuperGlue.} \blue{\gray{Gray} rows represent the results of PR-by-PE localization methods. We report 50th / 75th percentile errors for RE and TE.} The best result is highlighted in \textbf{bold} and the second best is \underline{underlined}.
    \end{tablenotes}
    \end{threeparttable}}
    \vspace{-0.5cm}
\end{table*}
\begin{table*}[htbp]
    \renewcommand\arraystretch{1.1}
    \centering
    \caption{Quantitative Results of Global Localization of Protocol 2}
    \label{tab:gl_protocol2}
    \resizebox{\textwidth}{!}{
    \begin{threeparttable}
    \begin{tabular}{clcccccc}
    \toprule[1pt]
    \multicolumn{2}{c}{\multirow{2}{*}{Approach}} & \multicolumn{3}{c}{NCLT} & \multicolumn{3}{c}{Oxford} \\ \cline{3-8}
    \multicolumn{2}{c}{} & \multicolumn{1}{c}{GL / PE Succ. $\uparrow$} & \multicolumn{1}{c}{RE [\textdegree] $\downarrow$} & \multicolumn{1}{c}{TE [m] $\downarrow$} & \multicolumn{1}{c}{GL / PE Succ. $\uparrow$} & \multicolumn{1}{c}{RE [\textdegree] $\downarrow$} & \multicolumn{1}{c}{TE [m] $\downarrow$} \\ \hline
    
    \multirow{8}{*}{Vision} & \blue{NetVLAD \cite{arandjelovic2016netvlad} + {SS \cite{detone2018superpoint,sarlin2020superglue}}$^{\dagger}$} & \blue{0.20} / \blue{0.43} & \blue{2.94} / \blue{8.57} & \blue{2.26} / \blue{5.45} & \blue{0.39} / \blue{0.75} & \blue{0.45} / \blue{1.36} & \blue{0.51} / \blue{1.90} \\
    & \blue{Patch-NetVLAD \cite{hausler2021patch} + {SS \cite{detone2018superpoint,sarlin2020superglue}}$^{\dagger}$} & \blue{0.21} / \blue{0.44} & \blue{3.02} / \blue{8.76} & \blue{2.12} / \blue{5.53} & \blue{0.37} / \blue{0.74} & \blue{0.49} / \blue{1.42} & \blue{0.54} / \blue{2.05} \\
    & \blue{AnyLoc \cite{keetha2023anyloc} + {SS \cite{detone2018superpoint,sarlin2020superglue}}$^{\dagger}$} & \blue{0.25} / \blue{0.48} & \blue{2.61} / \blue{7.23} & \blue{1.78} / \blue{4.54} & \blue{0.58} / \blue{0.78} & \blue{0.44} / \blue{1.20} & \blue{0.46} / \blue{1.54} \\
    & \blue{SFRS \cite{ge2020self} + {SS \cite{detone2018superpoint,sarlin2020superglue}}$^{\dagger}$} & \blue{0.28} / \blue{0.51} & \blue{2.36} / \blue{6.55} & \blue{1.64} / \blue{4.32} & \blue{0.58} / \blue{0.81} & \blue{\textbf{0.37}} / \blue{\textbf{0.98}} & \blue{\textbf{0.40}} / \blue{\underline{1.12}} \\
    & \cellcolor{gray!30}\blue{Exhaustive {SS \cite{detone2018superpoint,sarlin2020superglue}}$^{\dagger}$} & \cellcolor{gray!30}\blue{0.33} / \blue{\underline{0.60}} & \cellcolor{gray!30}\blue{\underline{2.01}} / \blue{\underline{4.48}} & \cellcolor{gray!30}\blue{\underline{1.20}} / \blue{\underline{3.05}} & \cellcolor{gray!30}\blue{\underline{0.69}} / \blue{\underline{0.83}} & \cellcolor{gray!30}\blue{\underline{0.38}} / \blue{\underline{1.02}} & \cellcolor{gray!30}\blue{\underline{0.42}} / \blue{1.13} \\
    & \blue{BEV-NetVLAD-MLP} & \blue{0.03} / \blue{0.04} & \blue{11.56} / \blue{27.42} & \blue{4.27} / \blue{6.46} & \blue{0.13} / \blue{0.17} & \blue{4.99} / \blue{9.50} & \blue{2.93} / \blue{4.72} \\
    & \blue{vDiSCO \cite{xu2023leveraging} + {SS \cite{detone2018superpoint,sarlin2020superglue}}$^{\dagger}$} & \blue{\underline{0.37}} / \blue{0.43} & \blue{2.95} / \blue{8.37} & \blue{2.26} / \blue{5.24} & \blue{0.62} / \blue{0.79} & \blue{0.44} / \blue{1.15} & \blue{0.46} / \blue{1.52} \\
    & \cellcolor{gray!30}\textbf{RING\#-V (Ours)} & \cellcolor{gray!30}\textbf{0.79} / \textbf{0.95} & \cellcolor{gray!30}\textbf{1.31} / \textbf{2.26} & \cellcolor{gray!30}\textbf{0.73} / \textbf{1.02} & \cellcolor{gray!30}\textbf{0.81} / \textbf{0.93} & \cellcolor{gray!30}0.65 / 1.07 & \cellcolor{gray!30}0.48 / \textbf{0.74} \\ \hline
    \multirow{7}{*}{LiDAR} & OverlapTransformer \cite{ma2022overlaptransformer} + ICP \cite{koide2021voxelized} & 0.30 / 0.43 & 3.85 / 164.88 & 2.62 / 7.06 & 0.57 / 0.72 & 0.01 / 0.06 & 0.01 / 2.95 \\
    & LCDNet \cite{cattaneo2022lcdnet} + ICP \cite{koide2021voxelized} & 0.28 / 0.63 & 1.39 / 3.22 & 0.22 / 3.82 & 0.58 / \textbf{0.93} & \textbf{0.00} / \textbf{0.00} & \textbf{0.00} / \textbf{0.00} \\
    & DiSCO \cite{xu2021disco} + ICP \cite{koide2021voxelized} & 0.61 / 0.74 & 1.15 / 2.29 & 0.21 / 0.85 & 0.66 / 0.75 & 0.01 / 0.04 & 0.01 / 2.02 \\
    & RING \cite{lu2022one} + ICP \cite{koide2021voxelized} & 0.53 / 0.95 & \underline{1.09} / 2.05 & \textbf{0.15} / \textbf{0.23} & 0.71 / \textbf{0.93} & \textbf{0.00} / \textbf{0.00} & \textbf{0.00} / \textbf{0.00} \\
    & \blue{RING++ \cite{xu2023ring++} + ICP \cite{koide2021voxelized}} & \blue{0.56} / \blue{\underline{0.96}} & \blue{1.10} / \blue{\underline{2.03}} & \blue{\textbf{0.15}} / \blue{\underline{0.24}} & \blue{0.76} / \blue{\textbf{0.93}} & \blue{\textbf{0.00}} / \blue{\textbf{0.00}} & \blue{\textbf{0.00}} / \blue{\textbf{0.00}} \\
    & EgoNN \cite{komorowski2021egonn} + ICP \cite{koide2021voxelized} & \underline{0.75} / 0.89 & 1.14 / 2.16 & 0.16 / 0.28 & \textbf{0.81} / 0.92 & \textbf{0.00} / \textbf{0.00} & \textbf{0.00} / \textbf{0.00} \\
    & \cellcolor{gray!30}\textbf{RING\#-L (Ours) + ICP \cite{koide2021voxelized}} & \cellcolor{gray!30}\textbf{0.78} / \textbf{0.97} & \cellcolor{gray!30}\textbf{1.07} / \textbf{1.98} & \cellcolor{gray!30}\textbf{0.15} / \underline{0.24} & \cellcolor{gray!30}\underline{0.79} / \textbf{0.93} & \cellcolor{gray!30}\textbf{0.00} / \textbf{0.00} & \cellcolor{gray!30}\textbf{0.00} / \textbf{0.00} \\
    
    \bottomrule[1pt]
    \end{tabular}
    \begin{tablenotes}
        \footnotesize
        \item[$\dagger$] \blue{SS: SuperPoint + SuperGlue.} \blue{\gray{Gray} rows represent the results of PR-by-PE localization methods. We report 50th / 75th percentile errors for RE and TE.} The best result is highlighted in \textbf{bold} and the second best is \underline{underlined}.
    \end{tablenotes}
    \end{threeparttable}}
    \vspace{-0.5cm}
\end{table*}
\begin{table*}[htbp]
    \renewcommand\arraystretch{1.1}
    \centering
    \caption{Quantitative Results of Global Localization of Protocol 3}
    \label{tab:gl_protocol3}
    \resizebox{\textwidth}{!}{
    \begin{threeparttable}
    \begin{tabular}{clcccccc}
    \toprule[1pt]
    \multicolumn{2}{c}{\multirow{2}{*}{Approach}} & \multicolumn{3}{c}{NCLT} & \multicolumn{3}{c}{Oxford} \\ \cline{3-8}
    \multicolumn{2}{c}{} & \multicolumn{1}{c}{GL / PE Succ. $\uparrow$} & \multicolumn{1}{c}{RE [\textdegree] $\downarrow$} & \multicolumn{1}{c}{TE [m] $\downarrow$} & \multicolumn{1}{c}{GL / PE Succ. $\uparrow$} & \multicolumn{1}{c}{RE [\textdegree] $\downarrow$} & \multicolumn{1}{c}{TE [m] $\downarrow$} \\ \hline
    
    \multirow{8}{*}{Vision} & \blue{NetVLAD \cite{arandjelovic2016netvlad} + {SS \cite{detone2018superpoint,sarlin2020superglue}}$^{\dagger}$} & \blue{0.15} / \blue{0.47} & \blue{2.78} / \blue{8.08} & \blue{1.74} / \blue{5.01} & \blue{0.39} / \blue{0.77} & \blue{0.45} / \blue{1.22} & \blue{0.46} / \blue{1.75} \\
    & \blue{Patch-NetVLAD \cite{hausler2021patch} + {SS \cite{detone2018superpoint,sarlin2020superglue}}$^{\dagger}$} & \blue{0.16} / \blue{0.47} & \blue{2.75} / \blue{7.64} & \blue{1.72} / \blue{4.88} & \blue{0.46} / \blue{0.78} & \blue{0.42} / \blue{1.14} & \blue{0.46} / \blue{1.40} \\
    & \blue{AnyLoc \cite{keetha2023anyloc} + {SS \cite{detone2018superpoint,sarlin2020superglue}}$^{\dagger}$} & \blue{0.25} / \blue{0.48} & \blue{2.61} / \blue{7.23} & \blue{1.78} / \blue{4.54} & \blue{0.58} / \blue{0.78} & \blue{0.44} / \blue{1.20} & \blue{0.46} / \blue{1.54} \\
    & \blue{SFRS \cite{ge2020self} + {SS \cite{detone2018superpoint,sarlin2020superglue}}$^{\dagger}$} & \blue{0.28} / \blue{0.51} & \blue{2.36} / \blue{6.55} & \blue{1.64} / \blue{4.32} & \blue{0.58} / \blue{0.81} & \blue{\textbf{0.37}} / \blue{\textbf{0.98}} & \blue{\textbf{0.40}} / \blue{\underline{1.12}} \\
    & \cellcolor{gray!30}\blue{Exhaustive {SS \cite{detone2018superpoint,sarlin2020superglue}}$^{\dagger}$} & \cellcolor{gray!30}\blue{\underline{0.33}} / \blue{\underline{0.60}} & \cellcolor{gray!30}\blue{\underline{2.01}} / \blue{\underline{4.48}} & \cellcolor{gray!30}\blue{\underline{1.20}} / \blue{\underline{3.05}} & \cellcolor{gray!30}\blue{\underline{0.69}} / \blue{\underline{0.83}} & \cellcolor{gray!30}\blue{\underline{0.38}} / \blue{\underline{1.02}} & \cellcolor{gray!30}\blue{\underline{0.42}} / \blue{1.13} \\
    & \blue{BEV-NetVLAD-MLP} & \blue{0.01} / \blue{0.02} & \blue{22.48} / \blue{60.75} & \blue{5.21} / \blue{7.23} & \blue{0.05} / \blue{0.07} & \blue{6.42} / \blue{10.96} & \blue{4.70} / \blue{7.88} \\
    & \blue{vDiSCO \cite{xu2023leveraging} + {SS \cite{detone2018superpoint,sarlin2020superglue}}$^{\dagger}$} & \blue{0.25} / \blue{0.47} & \blue{2.80} / \blue{7.49} & \blue{1.89} / \blue{4.64} & \blue{0.62} / \blue{0.78} & \blue{0.44} / \blue{1.21} & \blue{0.46} / \blue{1.59} \\
    & \cellcolor{gray!30}\textbf{RING\#-V (Ours)} & \cellcolor{gray!30}\textbf{0.72} / \textbf{0.94} & \cellcolor{gray!30}\textbf{1.28} / \textbf{2.25} & \cellcolor{gray!30}\textbf{0.61} / \textbf{0.91} & \cellcolor{gray!30}\textbf{0.79} / \textbf{0.92} & \cellcolor{gray!30}0.66 / 1.10 & \cellcolor{gray!30}0.52 / \textbf{0.81} \\ \hline
    \multirow{7}{*}{LiDAR} & OverlapTransformer \cite{ma2022overlaptransformer} + ICP \cite{koide2021voxelized} & 0.22 / 0.45 & 3.21 / 160.05 & 2.15 / 6.83 & 0.56 / 0.73 & 0.01 / 0.06 & 0.01 / 2.54 \\
    & LCDNet \cite{cattaneo2022lcdnet} + ICP \cite{koide2021voxelized} & 0.27 / 0.62 & 1.42 / 3.26 & 0.21 / 4.07 & 0.53 / 0.91 & \textbf{0.00} / \textbf{0.00} & \textbf{0.00} / 0.01 \\
    & DiSCO \cite{xu2021disco} + ICP \cite{koide2021voxelized} & 0.44 / 0.73 & 1.17 / 2.33 & 0.20 / 2.97 & 0.67 / 0.75 & 0.01 / 0.04 & 0.01 / 2.10 \\
    & RING \cite{lu2022one} + ICP \cite{koide2021voxelized} & 0.53 / 0.95 & \textbf{1.09} / 2.05 & \textbf{0.15} / \textbf{0.23} & 0.71 / \textbf{0.93} & \textbf{0.00} / \textbf{0.00} & \textbf{0.00} / \textbf{0.00} \\
    & \blue{RING++ \cite{xu2023ring++} + ICP \cite{koide2021voxelized}} & \blue{0.56} / \blue{\underline{0.96}} & \blue{1.10} / \blue{\underline{2.03}} & \blue{\textbf{0.15}} / \blue{\underline{0.24}} & \blue{0.76} / \blue{\textbf{0.93}} & \blue{\textbf{0.00}} / \blue{\textbf{0.00}} & \blue{\textbf{0.00}} / \blue{\textbf{0.00}} \\
    & EgoNN \cite{komorowski2021egonn} + ICP \cite{koide2021voxelized} & \underline{0.67} / 0.89 & 1.13 / 2.14 & 0.16 / 0.28 & \underline{0.79} / 0.91 & \textbf{0.00} / \textbf{0.00} & \textbf{0.00} / \textbf{0.00} \\
    & \cellcolor{gray!30}\textbf{RING\#-L (Ours) + ICP \cite{koide2021voxelized}} & \cellcolor{gray!30}\textbf{0.82} / \textbf{0.97} & \cellcolor{gray!30}\textbf{1.09} / \textbf{2.00} & \cellcolor{gray!30}\textbf{0.15} / \underline{0.24} & \cellcolor{gray!30}\textbf{0.80} / \textbf{0.93} & \cellcolor{gray!30}\textbf{0.00} / \textbf{0.00} & \cellcolor{gray!30}\textbf{0.00} / \textbf{0.00} \\
    
    \bottomrule[1pt]
    \end{tabular}
    \begin{tablenotes}
        \footnotesize
        \item[$\dagger$] \blue{SS: SuperPoint + SuperGlue.} \blue{\gray{Gray} rows represent the results of PR-by-PE localization methods. We report 50th / 75th percentile errors for RE and TE.} The best result is highlighted in \textbf{bold} and the second best is \underline{underlined}.
    \end{tablenotes}
    \end{threeparttable}}
    \vspace{-0.4cm}
\end{table*}
\begin{figure*}[htbp]
    \centering
    \includegraphics[width=17.6cm]{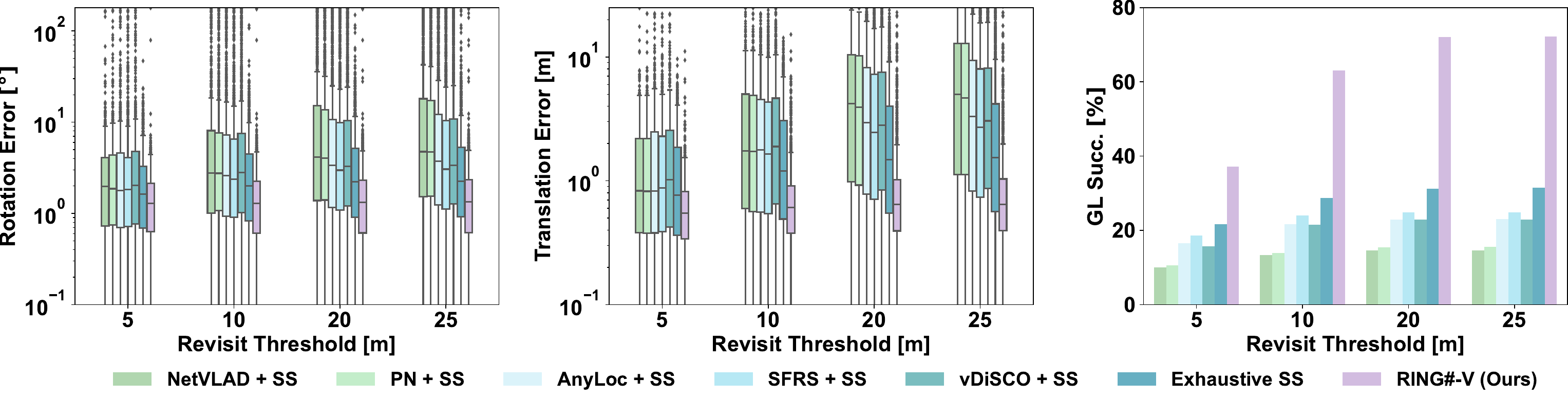}
    \vspace{-0.3cm}
    \caption{\blue{\textbf{Pose errors and global localization success rates at different revisit thresholds of vision-based methods on the NCLT dataset.} SS: SuperPoint + SuperGlue, PN: Patch-NetVLAD.}}
    \label{fig:nclt_vision_PE}
    \vspace{-0.3cm}
\end{figure*}

\subsection{\blue{One-stage Evaluation of Global Localization}}
\label{sec:one_stage_global_localization}
\bluei{
To further confirm the hypothesis in the previous subsection, we remove the strict revisit threshold in two-stage evaluation and conduct one-stage evaluation in this subsection.}

\textbf{Metrics.}
For \blue{two-stage evaluation of} global localization in Sec.~\ref{sec:two_stage_global_localization}, $\text{GL Succ.} = \text{Recall@1} \times \text{PE Succ.}$, where the denominator of PE Succ. is the number of correct retrievals. GL Succ. equals PE Succ. with the denominator being the number of queries \blue{for one-stage evaluation of} global localization.

\textbf{Baselines.} The baselines are the same as that in Sec.~\ref{sec:two_stage_global_localization}.

\textbf{Results.} Fig.~\ref{fig:gl_succ} reports the \blue{global localization success rate under two-stage and one-stage evaluation}. We conduct a statistical analysis using the two-tailed Mann-Whitney U test~\cite{mcknight2010mann} to determine the statistical significance of the proposed method. The results show that GL Succ. of all methods increases from \blue{two-stage evaluation} to \blue{one-stage evaluation} without the influence of the revisit threshold. In particular, RING\#-V and RING\#-L improve GL Succ. by 8.27\% and 18.20\% in comparison to \blue{two-stage evaluation}. This further confirms our hypothesis that \textit{GL Succ. can be improved by eliminating the \blue{strict} revisit threshold of place recognition}. Furthermore, the Mann-Whitney U test indicates that RING\# outperforms state-of-the-art methods with statistically significant performance improvements under both two-stage and one-stage evaluation. \blue{Additional qualitative localization results and a detailed analysis of failure cases are provided in Appendix~\ref{sec:appendix_one_stage_gl} and Appendix~\ref{sec:appendix_failure_cases}, respectively.} 

\begin{figure*}[htbp]
    \centering
    \includegraphics[width=17.5cm]{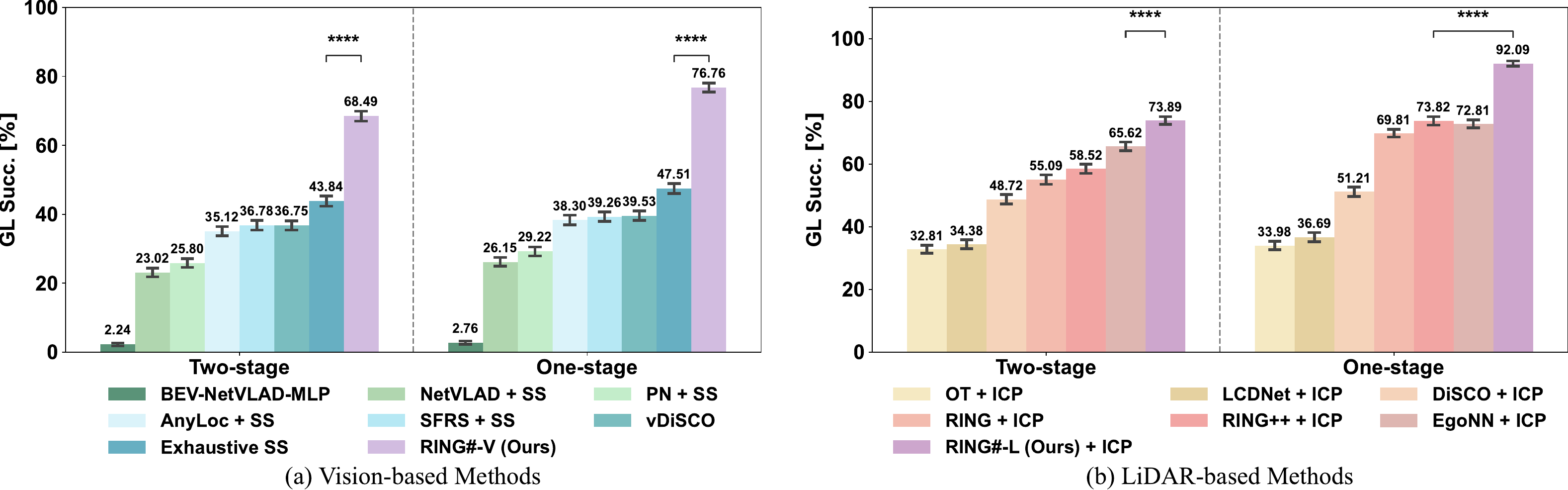}
    \vspace{-0.2cm}
    \caption{\blue{\textbf{Two-stage and one-stage global localization success rates averaged on the NCLT and Oxford datasets.}} \blue{SS: SuperPoint + SuperGlue, PN: Patch-NetVLAD, OT: OverlapTransformer.} The two-tailed Mann-Whitney U test shows statistical significance for a pair of methods comparison: $\star$ P $<$ 0.05; $\star\star$ P $<$ 0.01; $\star\star\star$ P $<$ 0.001; $\star\star\star\star$ P $<$ 0.0001.}
    \label{fig:gl_succ}
    \vspace{-0.4cm}
\end{figure*}

\subsection{Ablation Study}
\label{sec:ablation_study}
\bluei{In this subsection, we perform ablation experiments to investigate how different RING\# components and map intervals affect localization performance.}

\textbf{\blue{Equivariance Construction.}} \blue{Table~\ref{tab:ablation} presents the performance of place recognition and pure pose estimation for RING\#-V with different modules on the NCLT dataset under Protocol 1. Variants $\mathcal{M}_1$ to $\mathcal{M}_4$ differ in their use of CNNs within the rotation and translation branches. Among them, $\mathcal{M}_1$, which excludes CNNs in both branches, performs the worst, underscoring the necessity of CNNs for achieving high performance with RING\#-V. $\mathcal{M}_3$ outperforms $\mathcal{M}_2$, confirming the effectiveness of our rotation equivariance construction in the rotation branch, as stated in Theorem~\ref{theorem 1}. Furthermore, $\mathcal{M}_4$ surpasses $\mathcal{M}_3$ by 20\% in Recall@1, validating the translation equivariance construction detailed in Theorem~\ref{theorem 2}. These results highlight the essential role of a well-designed, learnable equivariance architecture.}

\textbf{\blue{Depth Supervision and Pose Refinement.}} \blue{We investigate the effects of depth supervision and pose refinement by comparing $\mathcal{M}_4 \sim \mathcal{M}_6$. $\mathcal{M}_5$ achieves slightly better Recall@1 and PE Succ. than $\mathcal{M}_4$ attributed to the enhanced depth accuracy enabled by depth supervision. $\mathcal{M}_6$, which builds on $\mathcal{M}_5$ by incorporating 3-DoF exhaustive matching upon the neural BEV in the translation branch for pose refinement, achieves the best performance among all variants. This indicates that the refinement process effectively reduces rotation estimation noise in the rotation branch.}

\textbf{\blue{Localization Paradigm.}} \blue{The performance of RING\#-V under different localization paradigms is assessed by comparing $\mathcal{M}_6 \sim \mathcal{M}_8$. $\mathcal{M}_7$ is trained solely with place recognition loss (Cross-Entropy loss) based on the similarity calculated in the rotation branch, which adheres to the \textit{PR-then-PE localization} paradigm. Its performance is inferior to $\mathcal{M}_6$, due to the limited supervision provided by place recognition alone. $\mathcal{M}_8$ extends $\mathcal{M}_7$ by jointly training place recognition and pose estimation with task-specific losses within the \textit{PR-then-PE localization} paradigm, resulting in a 22\% improvement in Recall@1 over $\mathcal{M}_7$. This demonstrates the benefits of integrating place recognition and pose estimation into a single network for \textit{PR-then-PE localization}. However, $\mathcal{M}_8$ still lags behind $\mathcal{M}_6$ in both Recall@1 and PE Succ., further validating the superior effectiveness of the \textit{PR-by-PE localization} paradigm.}

\begin{table*}[htbp]
    \renewcommand\arraystretch{1.1}
    \centering
    \caption{Ablation Study of Network Architecture$^*$}
    \label{tab:ablation}
    \resizebox{\textwidth}{!}{
    \begin{threeparttable}
    \begin{tabular}{cccccccc|cccc}
    \toprule[1pt]
    \multirow{2}{*}{}{\begin{tabular}[c]{@{}c@{}}\blue{Model}\end{tabular}} &
    \multirow{2}{*}{}{\begin{tabular}[c]{@{}c@{}}CNN in \\ Branch I\end{tabular}} &
    \multirow{2}{*}{}{\begin{tabular}[c]{@{}c@{}}CNN in \\ Branch II\end{tabular}} &
    \multirow{2}{*}{}{\begin{tabular}[c]{@{}c@{}}Depth \\ Loss\end{tabular}} &
    \multirow{2}{*}{}{\begin{tabular}[c]{@{}c@{}}\blue{PR} \\ \blue{Loss}\end{tabular}} &
    \multirow{2}{*}{}{\begin{tabular}[c]{@{}c@{}}\blue{PE} \\ \blue{Loss}\end{tabular}} &
    \multirow{2}{*}{}{\begin{tabular}[c]{@{}c@{}}Pose \\ Refinement\end{tabular}} &
    \multirow{2}{*}{}{\begin{tabular}[c]{@{}c@{}}\blue{PR-by-PE} \\ \blue{Localization}\end{tabular}} &
    \multirow{2}{*}{}{\begin{tabular}[c]{@{}c@{}}Recall@1 $\uparrow$\end{tabular}} &
    \multirow{2}{*}{}{\begin{tabular}[c]{@{}c@{}}RE [\textdegree] $\downarrow$\end{tabular}} &
    \multirow{2}{*}{}{\begin{tabular}[c]{@{}c@{}}TE [m] $\downarrow$\end{tabular}} &
    \multirow{2}{*}{}{\begin{tabular}[c]{@{}c@{}}PE Succ. $\uparrow$\end{tabular}} \\ \midrule
    \blue{$\mathcal{M}_1$} & \red{\xmark} & \red{\xmark} & \red{\xmark} & \red{\xmark} & \green{\cmark} & \red{\xmark} & \green{\cmark} & \blue{0.22} & \blue{6.59 / 15.73} & \blue{3.35 / 5.68} & \blue{0.17} \\
    $\mathcal{M}_2$ & Before RT & Before RC & \red{\xmark} & \red{\xmark} & \green{\cmark} & \red{\xmark} & \green{\cmark} & 0.56 & 2.40 / 5.70 & 1.76 / 3.78 & 0.51 \\
    $\mathcal{M}_3$ & After RT & Before RC & \red{\xmark} & \red{\xmark} & \green{\cmark} & \red{\xmark} & \green{\cmark} & 0.59 & 2.10 / 3.81 & 1.19 / 2.70 & 0.65 \\
    $\mathcal{M}_4$ & After RT & After RC & \red{\xmark} & \red{\xmark} & \green{\cmark} & \red{\xmark} & \green{\cmark} & \underline{0.79} & \underline{2.07} / 4.24 & 1.04 / 2.07 & 0.69 \\
    $\mathcal{M}_5$ & After RT & After RC & \green{\cmark} & \red{\xmark} & \green{\cmark} & \red{\xmark} & \green{\cmark} & \textbf{0.82} & 2.27 / 4.10 & \underline{0.86} / 1.71 & 0.71 \\ \midrule
    \textbf{$\mathcal{M}_6$ (Ours)} & After RT & After RC & \green{\cmark} & \red{\xmark} & \green{\cmark} & \green{\cmark} & \green{\cmark} & \textbf{0.82} & \textbf{1.25} / \textbf{2.22} & \textbf{0.71} / \textbf{1.29} & \textbf{0.85} \\  \midrule
    \blue{$\mathcal{M}_7$} & After RT & After RC & \green{\cmark} & \green{\cmark} & \red{\xmark} &  \red{\xmark} & \red{\xmark} & \blue{0.32} & \blue{-} & \blue{-} & \blue{-} \\
    \blue{$\mathcal{M}_8$} & After RT & After RC & \green{\cmark} & \green{\cmark} & \green{\cmark} & \green{\cmark} & \red{\xmark} & \blue{0.54} & \blue{\textbf{1.25} /  \underline{2.36}} & \blue{0.98 /  \underline{1.56}} & \blue{\underline{0.83}} \\
    
    \bottomrule[1pt]
    \end{tabular}
    \begin{tablenotes}
        \footnotesize
        \item[*] Branch I: Rotation Branch, Branch II: Translation Branch, \blue{PR Loss: Place Recognition Loss, PE Loss: Pose Estimation Loss,} Pose Refinement: 3-DoF Exhaustive Matching upon Neural BEV, RT: Radon Transform, RC: Rotation Compensation. \blue{We report 50th / 75th percentile errors for RE and TE.} The best result is highlighted in \textbf{bold} and the second best is \underline{underlined}.
    \end{tablenotes}
    \end{threeparttable}}
    \vspace{-0.4cm}
\end{table*}

\textbf{Map Interval.} \blue{We evaluate several representative methods (RING, RING++, EgoNN, RING\#-V, and RING\#-L) using Protocol 3 on the NCLT dataset, with map intervals ranging from 20m to 80m. The map interval refers to the distance between consecutive keyframes, where larger intervals result in sparser maps and greater viewpoint variation. The one-stage global localization success rates of these methods across varying map intervals are shown in Fig.~\ref{fig:gl_succ_map_interval}. As expected, larger map intervals result in smaller overlap between query and map keyframes, causing a general performance decline for all methods. Despite this, RING\#-V and RING\#-L exhibit outstanding robustness, consistently surpassing other methods across all intervals. Notably, RING\# tested at large intervals (\eg 50m) still outperforms other methods tested at small intervals (\eg 20m), underscoring the inherent strengths of the \textit{PR-by-PE localization} paradigm.}

\begin{figure}[ht]
    \centering
    \includegraphics[width=8.4cm]{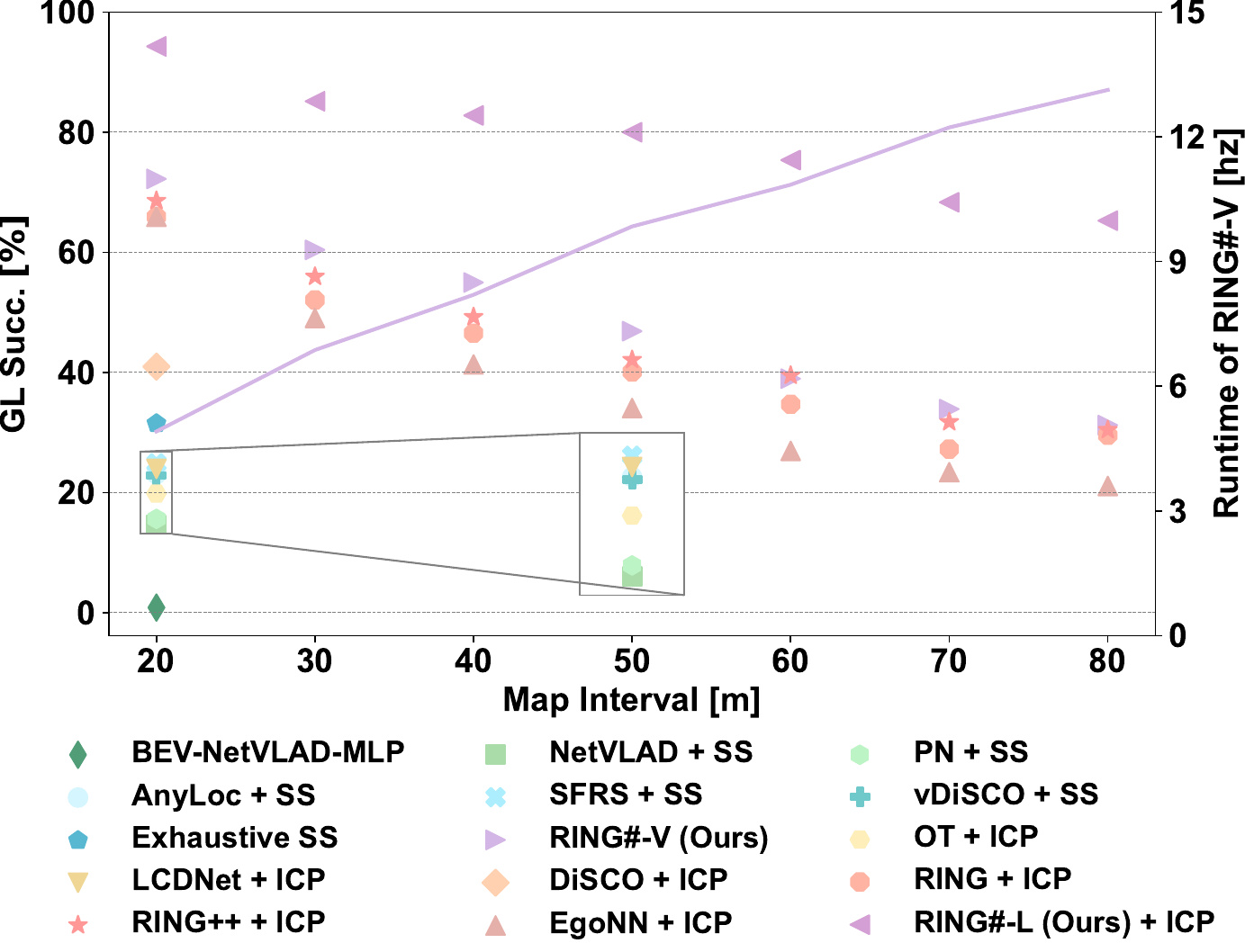}
    \vspace{-0.2cm}
    \caption{\blue{\textbf{One-stage global localization success rates and runtime of RING\#-V on the NCLT dataset at different map intervals.}} \blue{SS: SuperPoint + SuperGlue, PN: Patch-NetVLAD, OT: OverlapTransformer. Markers show the GL Succ. of all methods across various map intervals, while the \textcolor{myPurple}{purple} line depicts the runtime of RING\#-V at these intervals.}}
    \label{fig:gl_succ_map_interval}
    \vspace{-0.5cm}
\end{figure}

\subsection{Runtime Analysis}
\label{sec:runtime_model_size}
In this section, we evaluate the runtime of \blue{all approaches} with an NVIDIA GeForce GTX 4090 GPU. We report the average runtime \blue{per query} for all approaches on ``2012-01-08'' to ``2012-08-20'' of the NCLT dataset in \blue{Fig.~\ref{fig:runtime}}. Vision baselines \blue{except for BEV-NetVLAD-MLP exhibit substantial runtimes due to the extensive time required for detecting and matching local features across multi-view images using SuperPoint and SuperGlue for pose estimation. Among these, Exhaustive SS has the longest runtime due to its exhaustive matching process between the query and all map keyframes.} \blue{Although BEV-NetVLAD-MLP demonstrates the shortest runtime, it is unable to estimate accurate poses.} In contrast, \blue{RING\#-V achieves a balance between feasible runtime and superior performance, thanks to its end-to-end design, which leverages a fast correlation algorithm and batch processing on GPU.} \blue{Among the LiDAR-based methods, OverlapTransformer and DiSCO are faster than other approaches but fail to provide accurate initial poses for ICP alignment.} In contrast, RING\#-L maintains comparable runtime \blue{while delivering the best performance.} Moreover, as shown in Table~\ref{tab:pe_protocol1}, RING\#-L achieves nearly identical performance with or without ICP, allowing for reduced iterations to save time \blue{when necessary}.

\bluei{Table~\ref{tab:runtime} details the runtime of each module within RING\#. The feature extraction time represents the generation time of two equivariant representations in the rotation and translation branches.} The exhaustive search time, \blue{which is the most time-consuming part of RING\#,} reports the time needed to calculate similarities between the query and all \blue{map keyframes} in the database and identify the most similar one. \blue{The pose refinement time of RING\#-V is the time of pose refinement via 3-DoF exhaustive matching. Due to the larger grid size of the LiDAR BEV representation, RING\#-L requires more time than RING\#-V. Additionally, we investigate how the runtime of RING\#-V changes with different map intervals in Fig.~\ref{fig:gl_succ_map_interval}. As the map interval decreases, both runtime and performance increase empirically. To achieve a balance between performance and runtime, the map interval can be adjusted based on the specific application requirements.}

\begin{figure}[ht]
    \centering
    \includegraphics[width=8.3cm]{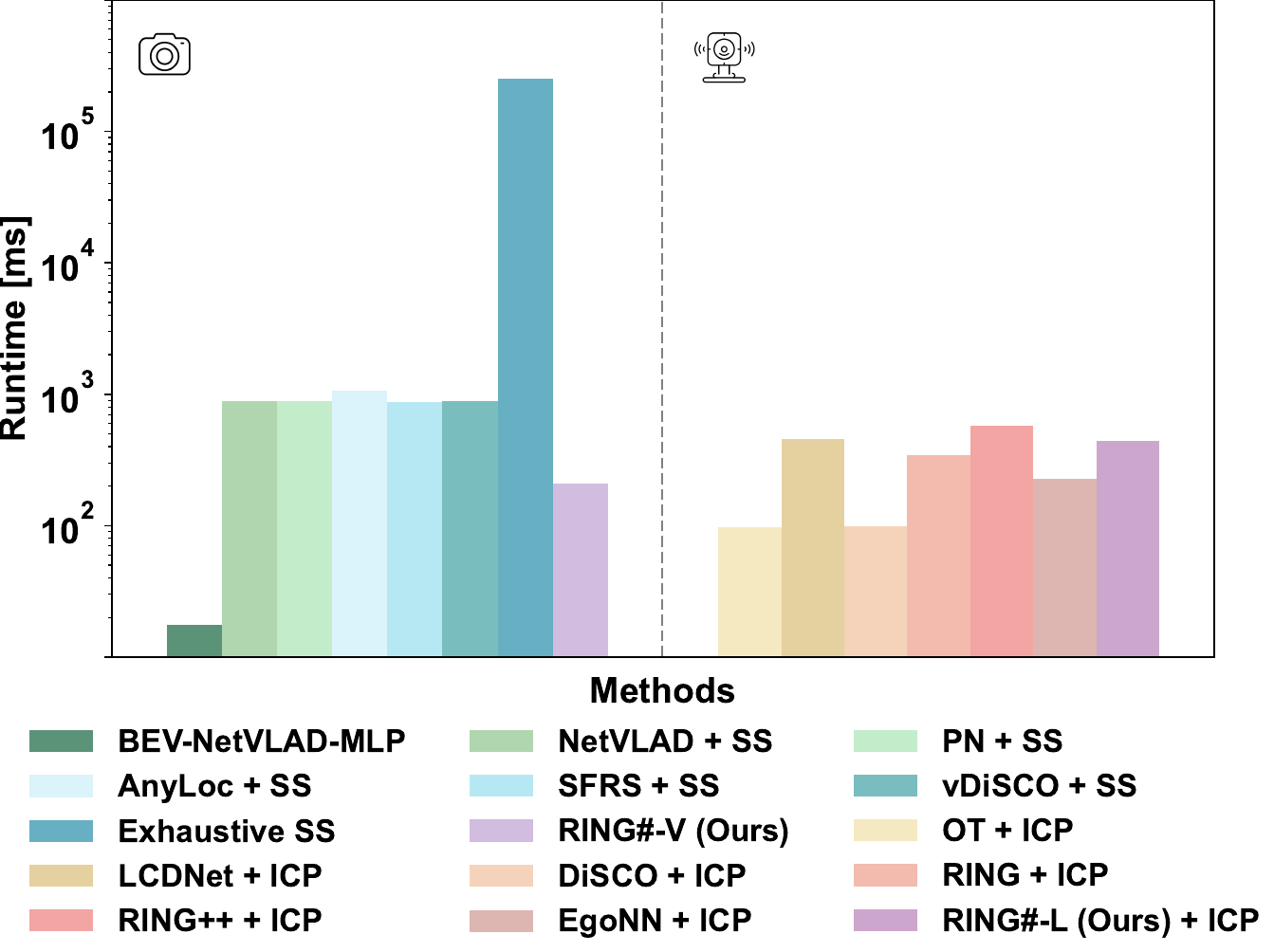}
    \vspace{-0.2cm}
    \caption{\blue{\textbf{Average runtime per query of all methods on the NCLT dataset.}} \blue{SS: SuperPoint + SuperGlue, PN: Patch-NetVLAD, OT: OverlapTransformer.}}
    \label{fig:runtime}
    \vspace{-0.4cm}
\end{figure}
\begin{table}
    \renewcommand\arraystretch{1.1}
    \centering
    \caption{\blue{Runtime of Each RING\# Module}}
    \label{tab:runtime}
    \begin{threeparttable}
    \begin{tabular}{cccc}
    \toprule[1pt]
    \multirow{2}{*}{}{\begin{tabular}[c]{@{}c@{}}\blue{Approach}\end{tabular}} &
    \multirow{2}{*}{}{\begin{tabular}[c]{@{}c@{}}\blue{Feature} \\ \blue{Extraction [ms]}\end{tabular}} &
    \multirow{2}{*}{}{\begin{tabular}[c]{@{}c@{}}\blue{Exhaustive} \\ \blue{Search [ms]}\end{tabular}} &
    \multirow{2}{*}{}{\begin{tabular}[c]{@{}c@{}}\blue{Pose} \\ \blue{Refinement [ms]}\end{tabular}} \\ \midrule
    \blue{RING\#-V} & \blue{12.64} & \blue{152.92} & \blue{44.22} \\
    \blue{RING\#-L} & \blue{16.31} & \blue{330.01} & \blue{-} \\
    \bottomrule[1pt]
    \end{tabular}
    \end{threeparttable}
    \vspace{-0.5cm}
\end{table}

\section{Conclusion \blue{and Future Work}}
In this paper, we propose RING\#, \blue{an end-to-end \textit{PR-by-PE localization} framework in the BEV space}. By leveraging correlation-based exhaustive search on equivariant \blue{BEV} representations, RING\# effectively \blue{captures} the spatial structure of the environment, \blue{enabling globally convergent localization for both vision and LiDAR modalities.} \blue{Extensive experiments on the NCLT and Oxford datasets demonstrate its} superior localization performance, especially under challenging environmental variations. The success of RING\# not only confirms the effectiveness of \blue{the \textit{PR-by-PE localization} paradigm} but also sets a new standard for localization performance, emphasizing the potential of our method in the realm of autonomous navigation and robotic systems.

\textbf{\blue{Limitations and Future Work.}}
\blue{RING\# is designed for 3-DoF localization, which may not fully address 6-DoF applications. Besides, the BEV generation process in RING\#-V relies on specific camera parameters, potentially limiting its adaptability to varied sensor setups. Future work will focus on extending RING\# to handle 6-DoF localization and enhancing its adaptability across different sensor configurations. Exploring the integration of foundation models into our framework, while maintaining essential equivariant properties, is a promising direction to improve adaptability and robustness.}

\bibliographystyle{IEEEtran}
\bibliography{IEEEabrv, IEEEref}

\vfill

\clearpage

\appendix
\label{sec:appendix}
This appendix presents additional details and visualizations that supplement the main text, providing further insights into the proposed RING\# method. We first present a case study to validate the rotation-equivariant and translation-invariant representation introduced in the rotation branch of RING\# (Sec.~\ref{sec:appendix_case_study}). We then provide additional qualitative results for place recognition on the Oxford dataset (Sec.~\ref{sec:appendix_pr}). Next, we present additional results for two-stage global localization on the NCLT and Oxford datasets (Sec.~\ref{sec:appendix_two_stage_gl}), supplementing the findings shown in Fig.~\ref{fig:nclt_vision_PE} in the main paper. We also include qualitative results for one-stage global localization on both datasets (Sec.~\ref{sec:appendix_one_stage_gl}), complementing the results shown in Fig.~\ref{fig:gl_succ} in the main paper. Finally, we analyze failure cases to identify the limitations of RING\# and suggest potential improvements (Sec.~\ref{sec:appendix_failure_cases}), and provide a detailed description of the compared methods used in the evaluation (Sec.~\ref{sec:appendix_methods}).

\subsection{Case Study}
\label{sec:appendix_case_study}
\begin{figure*}[hbp]
    \centering
    \includegraphics[width=17.6cm]{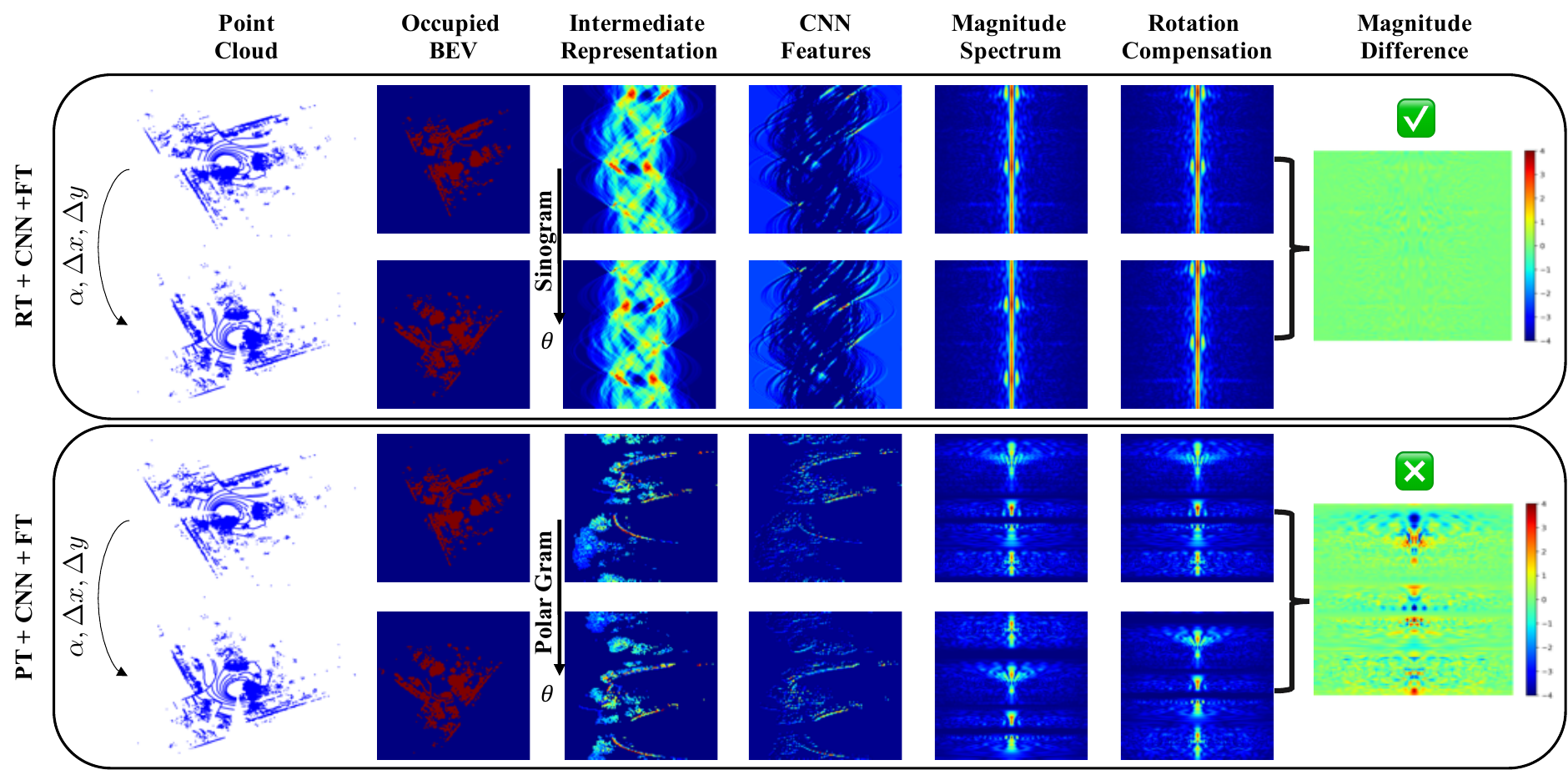}
    \vspace{-0.2cm}
    \caption{\textbf{Case study of rotation equivariance.} Comparison of the learnable rotation-equivariant representations using the Radon transform (RT) and the polar transform (PT). \raisebox{-2pt}{\includegraphics[height=9pt]{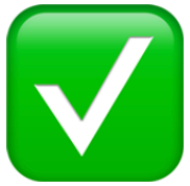}} means that the rotation equivariance is preserved, while \raisebox{-2pt}{\includegraphics[height=9pt]{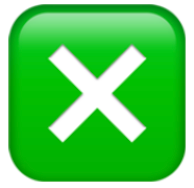}} means that the rotation equivariance is not preserved.}
    \label{fig:case_equivariance}
    \vspace{-0.3cm}
\end{figure*}
We conduct a case study to validate the learnable rotation-equivariant and translation-invariant representation introduced in the rotation branch of RING\#. Specifically, we compare the rotation equivariance properties of the Radon transform (RT) and polar transform (PT), as defined in Eq.~\ref{eq:radon} and Eq.~\ref{eq:polar}, and discussed in Sec.~\ref{sec:radon}. In this experiment, we apply a random 3-DoF transformation $(\alpha, \Delta x, \Delta y)$ to a point cloud, generating a pair of transformed point clouds. These point clouds are then converted into occupied BEV, where 0 represents free space and 1 represents occupied space. We then apply the Radon transform and polar transform to the occupied BEV, generating the corresponding sinogram and polar gram, respectively. Following the procedure outlined in Sec.~\ref{sec:rotation}, we employ a convolutional neural network (CNN) followed by the Fourier transform (FT) to extract the rotation-equivariant magnitude spectra from both representations. Finally, we apply rotation compensation to the magnitude spectra to eliminate the effects of rotation. In theory, if the representation is truly rotation-equivariant, the difference between the two compensated spectra should approach zero. As illustrated in Fig.~\ref{fig:case_equivariance}, the Radon transform results in a near-zero difference, while the polar transform exhibits a larger difference, validating our rotation equivariance design. This empirical result further supports the theoretical analysis provided in Theorem~\ref{theorem 1}.


\subsection{Additional Results of Place Recognition}
\label{sec:appendix_pr}
In addition to the visualization of top 1 retrieved matches for the NCLT dataset shown in Fig.~\ref{fig:nclt_matches}, we provide additional qualitative results for the Oxford dataset in Fig.~\ref{fig:oxford_matches}. The overall performance on the Oxford dataset surpasses that on the NCLT dataset, primarily due to fewer environmental changes in the Oxford dataset. Compared to state-of-the-art methods, both RING\#-V and RING\#-L achieve a higher number of correct matches and fewer incorrect matches, further demonstrating the superior performance of our method. This is consistent with the quantitative findings discussed in Sec.~\ref{sec:place_recognition}.

\begin{figure*}[htbp]
    \centering
    \subfigure[]{
		\includegraphics[trim=1.5cm 1.6cm 0.75cm 1.247cm, clip, width=2.405cm]{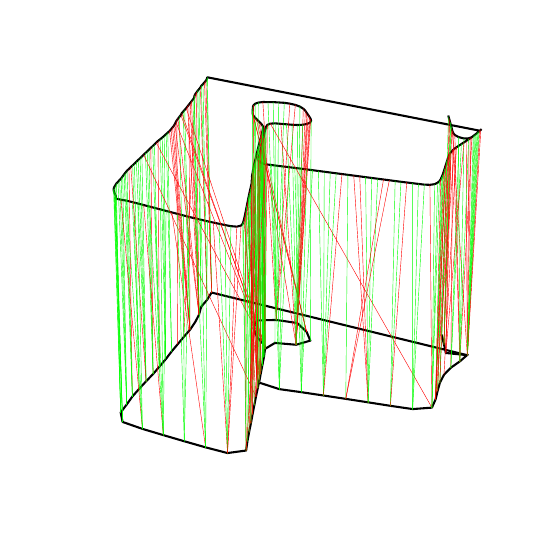}}
    \hspace{-0.515cm}
    \vspace{-0.04cm}
    \subfigure[]{
		\includegraphics[trim=1.5cm 1.6cm 0.75cm 1.247cm, clip, width=2.405cm]{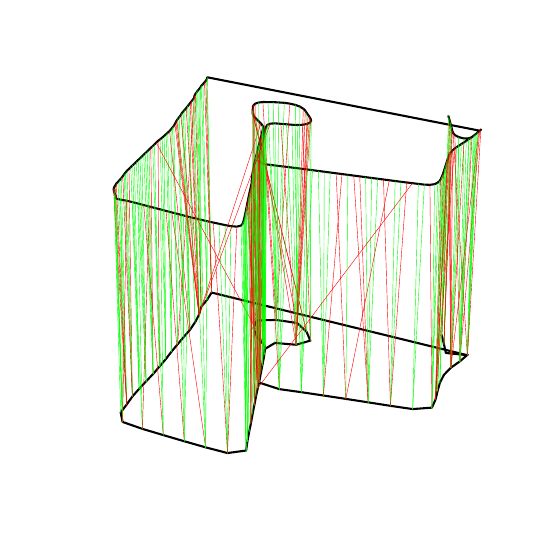}}
    \hspace{-0.515cm}
    \vspace{-0.04cm}
    \subfigure[]{
		\includegraphics[trim=1.5cm 1.6cm 0.75cm 1.247cm, clip, width=2.405cm]{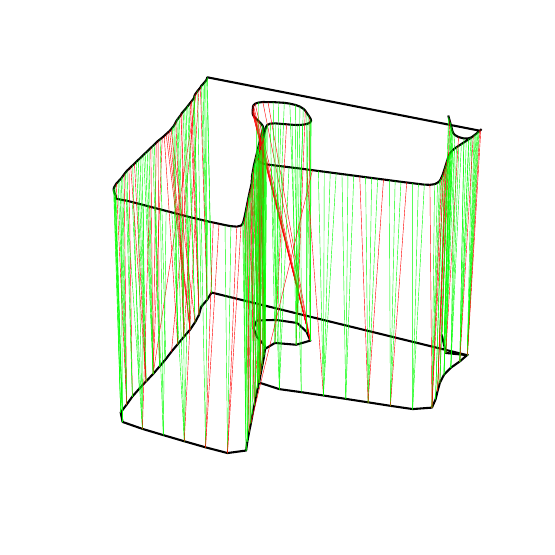}}
    \hspace{-0.515cm}
    \vspace{-0.04cm}
    \subfigure[]{
		\includegraphics[trim=1.5cm 1.6cm 0.75cm 1.247cm, clip, width=2.405cm]{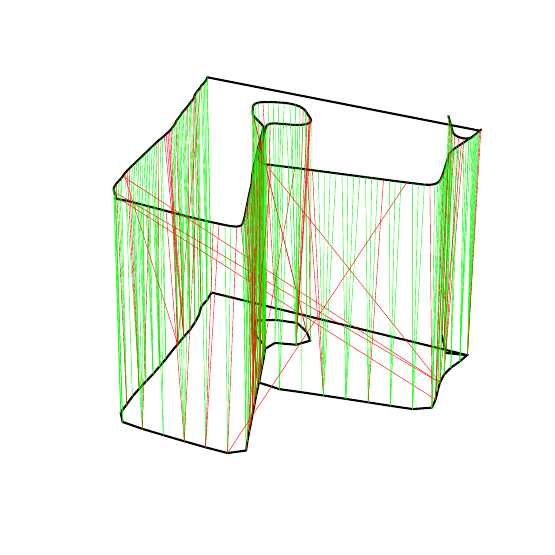}}
    \hspace{-0.515cm}
    \vspace{-0.04cm}
    \subfigure[]{
		\includegraphics[trim=1.5cm 1.6cm 0.75cm 1.247cm, clip, width=2.405cm]{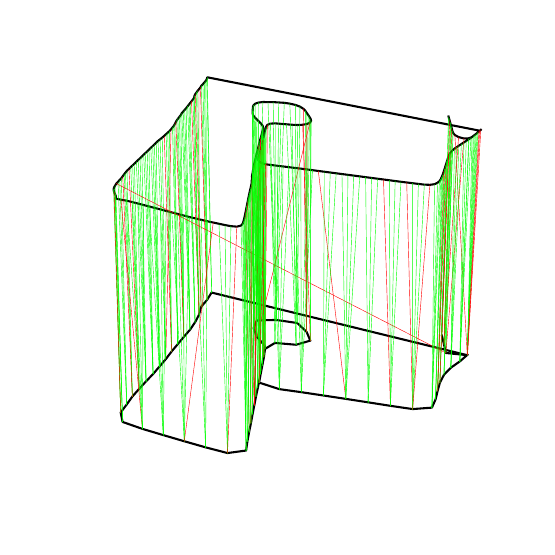}}
    \hspace{-0.515cm}
    \vspace{-0.04cm}    
    \subfigure[]{
		\includegraphics[trim=1.5cm 1.6cm 0.75cm 1.247cm, clip, width=2.405cm]{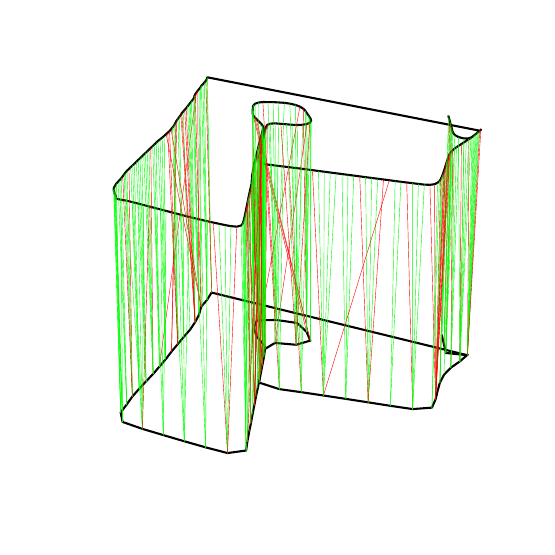}}
    \hspace{-0.515cm}
    \vspace{-0.04cm}
    \subfigure[]{
    \includegraphics[trim=1.5cm 1.6cm 0.75cm 1.247cm, clip, width=2.405cm]{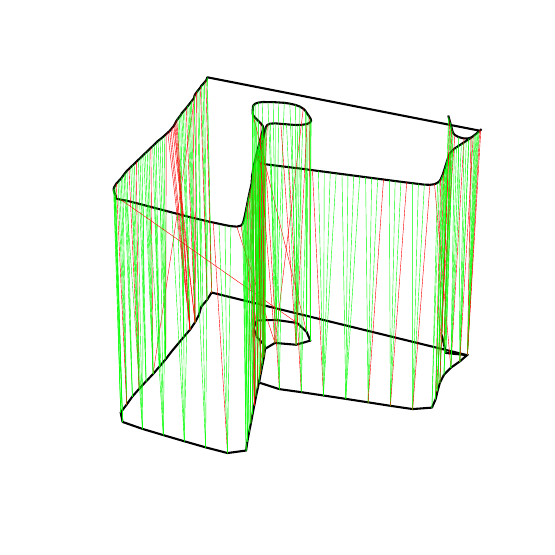}}
    \hspace{-0.515cm}
    \vspace{-0.04cm}
    \subfigure[]{
		\includegraphics[trim=1.5cm 1.6cm 0.75cm 1.247cm, clip, width=2.405cm]{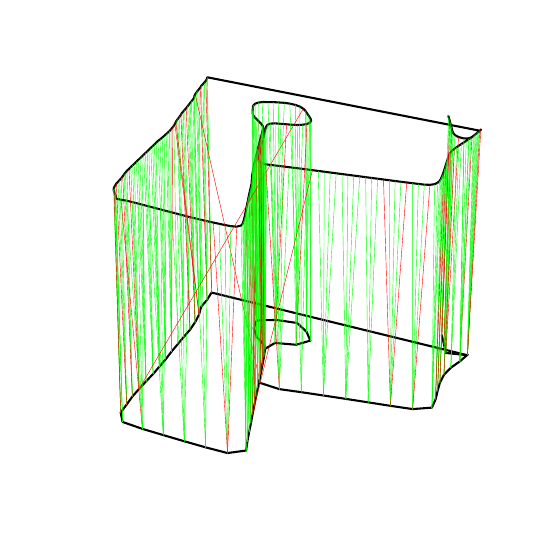}}
    \hspace{-0.515cm}
    \vspace{-0.04cm}
    \subfigure[]{
		\includegraphics[trim=1.5cm 1.6cm 0.75cm 1.247cm, clip, width=2.405cm]{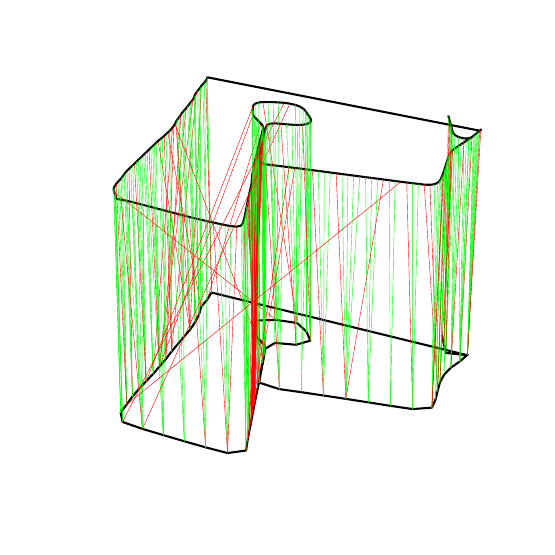}}
    \hspace{-0.515cm}
    \vspace{-0.01cm}
    \subfigure[]{
		\includegraphics[trim=1.5cm 1.6cm 0.75cm 1.247cm, clip, width=2.405cm]{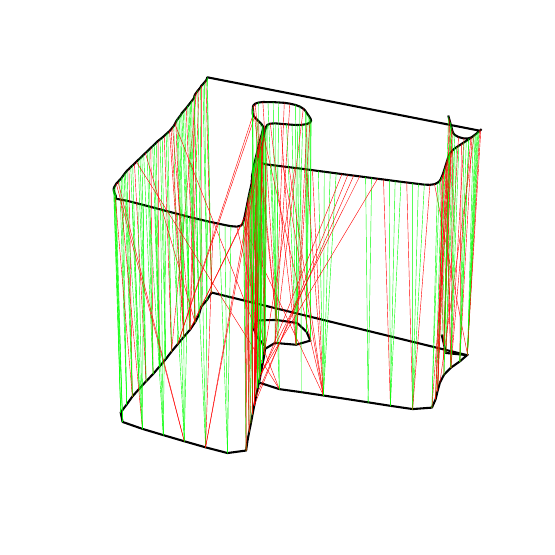}}
    \hspace{-0.515cm}
    \vspace{-0.01cm}
    \subfigure[]{
		\includegraphics[trim=1.5cm 1.6cm 0.75cm 1.247cm, clip, width=2.405cm]{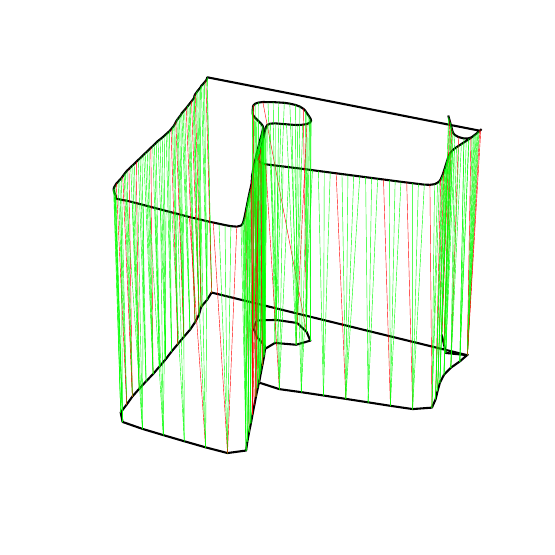}}
    \hspace{-0.515cm}
    \vspace{-0.01cm}
    \subfigure[]{
		\includegraphics[trim=1.5cm 1.6cm 0.75cm 1.247cm, clip, width=2.405cm]{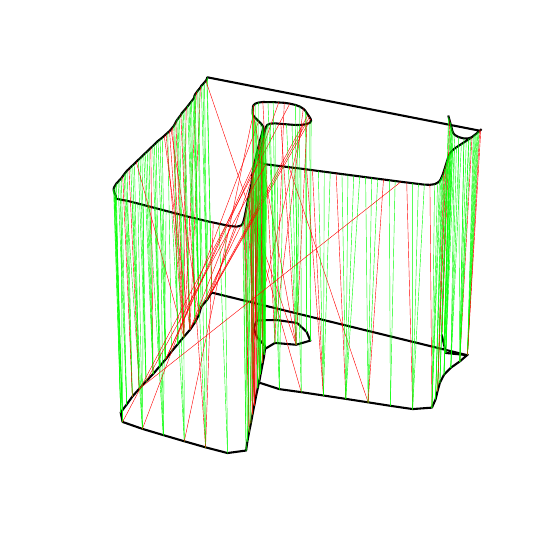}}    
    \hspace{-0.515cm}
    \vspace{-0.01cm}
    \subfigure[]{
		\includegraphics[trim=1.5cm 1.6cm 0.75cm 1.247cm, clip, width=2.405cm]{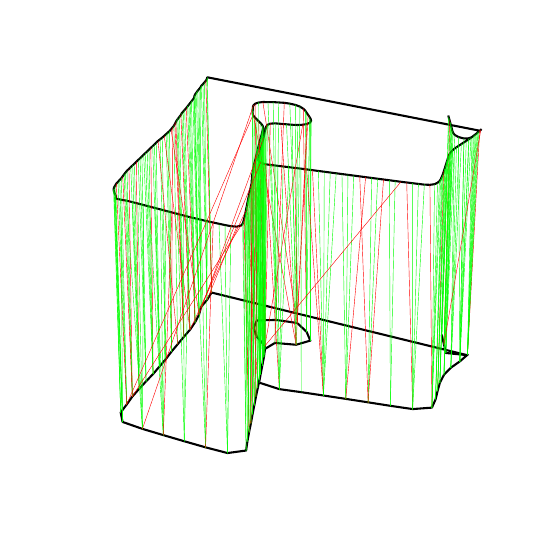}}    
    \hspace{-0.515cm}
    \vspace{-0.01cm}
    \subfigure[]{
		\includegraphics[trim=1.5cm 1.6cm 0.75cm 1.247cm, clip, width=2.405cm]{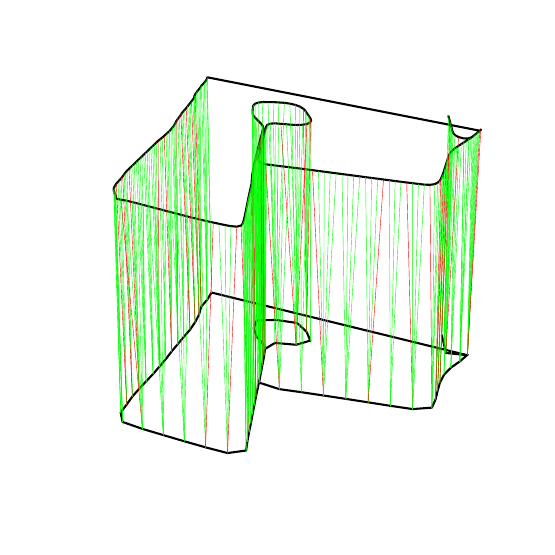}}
    \hspace{-0.515cm}
    \vspace{-0.01cm}
    \subfigure[]{
		\includegraphics[trim=1.5cm 1.6cm 0.75cm 1.247cm, clip, width=2.405cm]{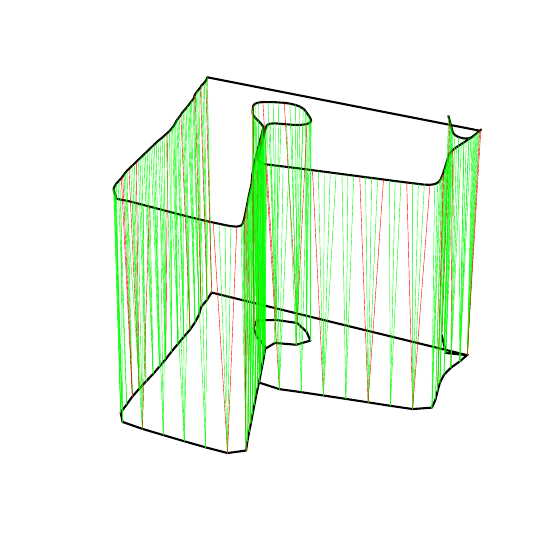}}
    \caption{\textbf{Top 1 retrieved matches for protocol 1 on the Oxford dataset.} (a) NetVLAD~\cite{arandjelovic2016netvlad}. (b) Patch-NetVLAD~\cite{hausler2021patch}. (c) AnyLoc~\cite{keetha2023anyloc}. (d) SFRS~\cite{ge2020self}. (e) Exhaustive SS~\cite{detone2018superpoint, sarlin2020superglue}. (f) BEV-NetVLAD-MLP. (g) vDiSCO~\cite{xu2023leveraging}. (h) RING\#-V (Ours). (i) OverlapTransformer~\cite{ma2022overlaptransformer}. (j) LCDNet~\cite{cattaneo2022lcdnet}. (k) DiSCO~\cite{xu2021disco}. (l) RING~\cite{lu2022one}. (m) RING++~\cite{xu2023ring++}. (n) EgoNN~\cite{komorowski2021egonn}. (o) RING\#-L (Ours). The black line \raisebox{0.5ex}{\rule{0.3cm}{0.5pt}} represents the trajectory, the green line {\color{green}{\raisebox{0.5ex}{\rule{0.3cm}{0.5pt}}}} represents the correct retrieval match, and the red line {\color{red}{\raisebox{0.5ex}{\rule{0.3cm}{0.5pt}}}} represents the wrong retrieval match.}
    \label{fig:oxford_matches}
    \vspace{-0.2cm}
\end{figure*}

\subsection{Additional Results of Two-stage Global Localization}
\label{sec:appendix_two_stage_gl}
Fig.~\ref{fig:nclt_vision_PE} in the main paper has demonstrated the influence of different revisit thresholds on pose errors and global localization success rates for vision-based methods on the NCLT dataset. Here, Fig.~\ref{fig:nclt_lidar_PE} to Fig.~\ref{fig:oxford_lidar_PE} provide additional results for both vision- and LiDAR-based methods on the NCLT and Oxford datasets. As the revisit threshold increases, the performance gap between RING\# and other methods widens, which is especially evident in the vision domain. This observation further validates the hypothesis in Sec.\ref{sec:two_stage_global_localization}, where the strict revisit threshold of place recognition rejects some correct localization results, underestimating the global localization performance. This explains why RING\# consistently achieves high PE Succ. across all protocols, while its global localization performance varies, as shown in Table~\ref{tab:gl_protocol1} to Table~\ref{tab:gl_protocol3} and discussed in the last two findings in Sec.\ref{sec:two_stage_global_localization}.

\begin{figure*}[htbp]
    \centering
    \includegraphics[width=17.6cm]{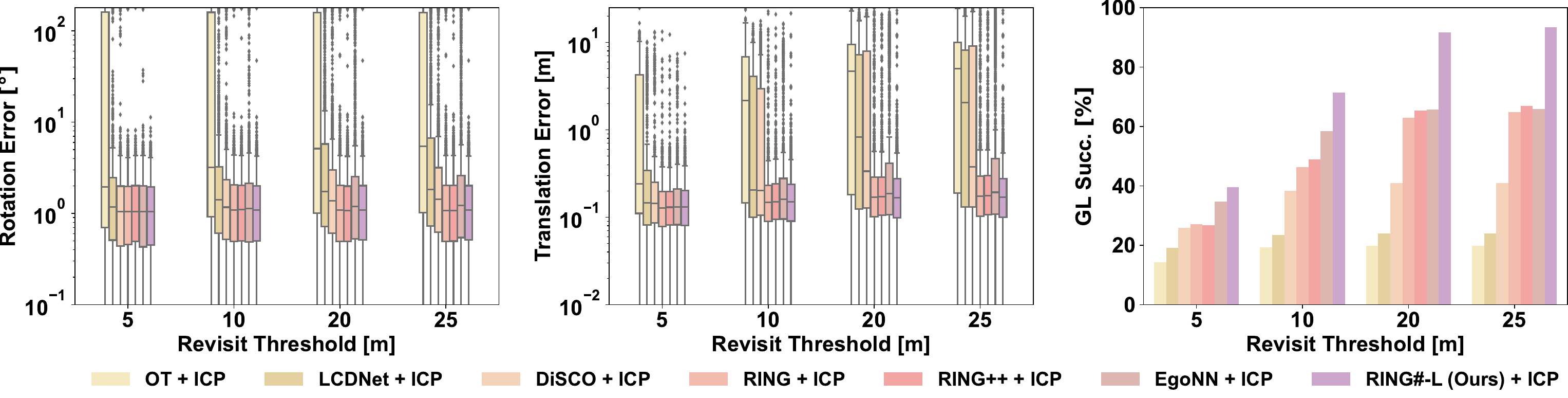}
    \vspace{-0.3cm}
    \caption{\textbf{Pose errors and global localization success rates at different revisit thresholds of LiDAR-based methods on the NCLT dataset.} OT: OverlapTransformer.}
    \label{fig:nclt_lidar_PE}
    \vspace{-0.3cm}
\end{figure*}
\begin{figure*}[htbp]
    \centering
    \includegraphics[width=17.6cm]{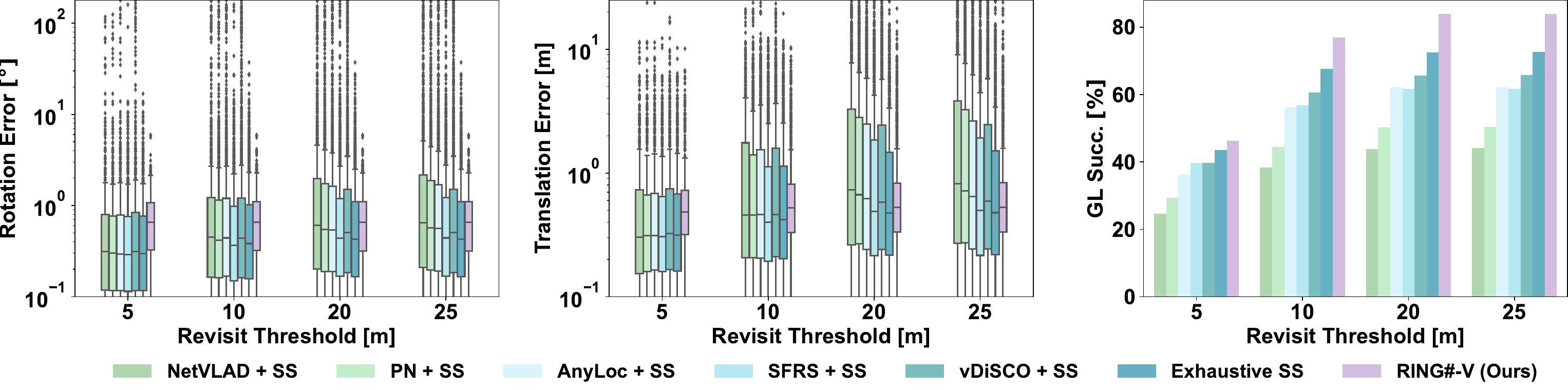}
    \vspace{-0.3cm}
    \caption{\textbf{Pose errors and global localization success rates at different revisit thresholds of vision-based methods on the Oxford dataset.} SS: SuperPoint + SuperGlue, PN: Patch-NetVLAD.}
    \label{fig:oxford_vision_PE}
    \vspace{-0.3cm}
\end{figure*}
\begin{figure*}[htbp]
    \centering
    \includegraphics[width=17.6cm]{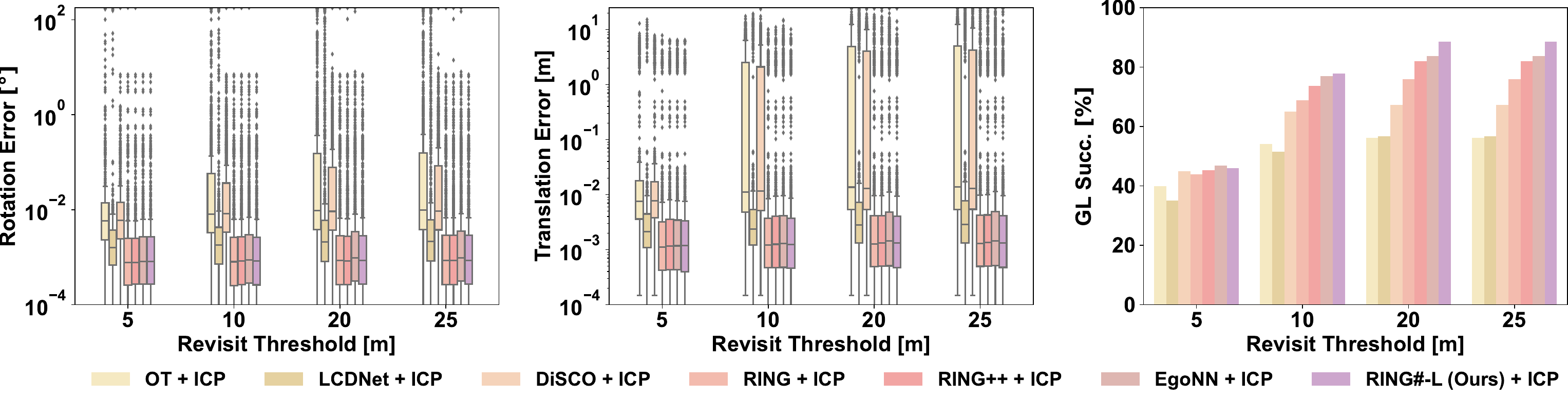}
    \vspace{-0.3cm}
    \caption{\textbf{Pose errors and global localization success rates at different revisit thresholds of LiDAR-based methods on the Oxford dataset.} OT: OverlapTransformer.}
    \label{fig:oxford_lidar_PE}
    \vspace{-0.3cm}
\end{figure*}

\begin{table*}[htbp]
    \renewcommand\arraystretch{1.1}
    \centering
    \caption{Quantitive Results Comparing Two-stage and One-stage Global Localization of Protocol 2}
    \label{tab:pr_gl_protocol2}
    \resizebox{\textwidth}{!}{
    \begin{threeparttable}
    \begin{tabular}{clcccccccc}
    \toprule[1pt]
    \multicolumn{2}{c}{\multirow{3}{*}{Approach}} & \multicolumn{4}{c}{NCLT} & \multicolumn{4}{c}{Oxford} \\ \cline{3-10}
    \multicolumn{2}{c}{} & \multirow{2}{*}{Recall@1 $\uparrow$} & \multirow{2}{*}{PE Succ. $\uparrow$} & \multicolumn{2}{c}{GL Succ. $\uparrow$} & \multirow{2}{*}{Recall@1 $\uparrow$} & \multirow{2}{*}{PE Succ. $\uparrow$} & \multicolumn{2}{c}{GL Succ. $\uparrow$} \\ \cline{5-6} \cline{9-10}
    \multicolumn{2}{c}{} &  &  & Two-stage & One-stage &  &  & Two-stage & One-stage \\ \hline
    \multirow{8}{*}{Vision} & NetVLAD \cite{arandjelovic2016netvlad} + {SS \cite{detone2018superpoint,sarlin2020superglue}}$^{\dagger}$ & 0.47 & 0.43 & 0.20 & 0.19 & 0.52 & 0.75 & 0.39 & 0.45 \\
    & Patch-NetVLAD \cite{hausler2021patch} + {SS \cite{detone2018superpoint,sarlin2020superglue}}$^{\dagger}$ & 0.48 & 0.44 & 0.21 & 0.2 & 0.50 & 0.74 & 0.37 & 0.44 \\
    & AnyLoc \cite{keetha2023anyloc} + {SS \cite{detone2018superpoint,sarlin2020superglue}}$^{\dagger}$ & 0.52 & 0.48 & 0.25 & 0.23 & 0.74 & 0.78 & 0.58 & 0.62 \\
    & SFRS \cite{ge2020self} + {SS \cite{detone2018superpoint,sarlin2020superglue}}$^{\dagger}$ & 0.54 & 0.51 & 0.28 & 0.25 & 0.72 & 0.81 & 0.58 & 0.62 \\
    & \cellcolor{gray!30}Exhaustive {SS \cite{detone2018superpoint,sarlin2020superglue}}$^{\dagger}$ & \cellcolor{gray!30}0.55 & \cellcolor{gray!30}\underline{0.60} & \cellcolor{gray!30}0.33 & \cellcolor{gray!30}0.32 & \cellcolor{gray!30}\underline{0.84} & \cellcolor{gray!30}\underline{0.83} & \cellcolor{gray!30}\underline{0.69} & \cellcolor{gray!30}\underline{0.73} \\
    & BEV-NetVLAD-MLP & 0.73 & 0.04 & 0.03 & 0.03 & 0.74 & 0.17 & 0.13 & 0.13 \\
    & vDiSCO \cite{xu2023leveraging} + {SS \cite{detone2018superpoint,sarlin2020superglue}}$^{\dagger}$ & \textbf{0.86} & 0.43 & \underline{0.37} & \underline{0.33} & 0.78 & 0.79 & 0.62 & 0.66 \\
    & \cellcolor{gray!30}\textbf{RING\#-V (Ours)} & \cellcolor{gray!30}\underline{0.83} & \cellcolor{gray!30}\textbf{0.95} & \cellcolor{gray!30}\textbf{0.79} & \cellcolor{gray!30}\textbf{0.81} & \cellcolor{gray!30}\textbf{0.86} & \cellcolor{gray!30}\textbf{0.93} & \cellcolor{gray!30}\textbf{0.81} & \cellcolor{gray!30}\textbf{0.86} \\ \hline
    \multirow{7}{*}{LiDAR} & OverlapTransformer \cite{ma2022overlaptransformer} + ICP \cite{koide2021voxelized} & 0.69 & 0.43 & 0.30 & 0.27 & 0.79 & 0.72 & 0.57 & 0.57 \\
    & LCDNet \cite{cattaneo2022lcdnet} + ICP \cite{koide2021voxelized} & 0.44 & 0.63 & 0.28 & 0.25 & 0.62 & \textbf{0.93} & 0.58 & 0.63 \\
    & DiSCO \cite{xu2021disco} + ICP \cite{koide2021voxelized} & \underline{0.82} & 0.74 & 0.61 & 0.56 & \textbf{0.88} & 0.75 & 0.66 & 0.67 \\
    & RING \cite{lu2022one} + ICP \cite{koide2021voxelized} & 0.56 & 0.95 & 0.53 & 0.66 & 0.76 & \textbf{0.93} & 0.71 & 0.76 \\
    & RING++ \cite{xu2023ring++} + ICP \cite{koide2021voxelized} & 0.59 & \underline{0.96} & 0.56 & 0.69 & 0.81 & \textbf{0.93} & 0.76 & 0.82 \\
    & EgoNN \cite{komorowski2021egonn} + ICP \cite{koide2021voxelized} & \textbf{0.84} & 0.89 & \underline{0.75} & \underline{0.76} & \textbf{0.88} & 0.92 & \textbf{0.81} & \underline{0.86} \\
    & \cellcolor{gray!30}\textbf{RING\#-L (Ours) + ICP \cite{koide2021voxelized}} & \cellcolor{gray!30}0.81 & \cellcolor{gray!30}\textbf{0.97} & \cellcolor{gray!30}\textbf{0.78} & \cellcolor{gray!30}\textbf{0.93} & \cellcolor{gray!30}0.84 & \cellcolor{gray!30}\textbf{0.93} & \cellcolor{gray!30}\underline{0.79} & \cellcolor{gray!30}\textbf{0.89} \\
    
    \bottomrule[1pt]
    \end{tabular}
    \begin{tablenotes}
        \footnotesize
        \item[$\dagger$] SS: Superpoint + SuperGlue. \gray{Gray} rows represent the results of PR-by-PE localization methods. The best result is highlighted in \textbf{bold} and the second best is \underline{underlined}.
    \end{tablenotes}
    \end{threeparttable}}
    \vspace{-0.3cm}
\end{table*}

To further clarify these findings, we provide the quantitative results comparing two-stage and one-stage global localization for Protocol 2 in Table~\ref{tab:pr_gl_protocol2}. While RING\# achieves the best PE Succ., it exhibits slightly lower Recall@1 in some cases, which affects its two-stage GL Succ.. This explains why RING\#-L have a slightly lower two-stage GL Succ. compared to EgoNN on the Oxford dataset. However, one-stage evaluation of global localization is more realistic and practical, as it directly evaluates the global localization performance without the influence of the revisit threshold. Under this evaluation, RING\# achieves the highest GL Succ., outperforming other methods by a significant margin on both datasets, which demonstrates the true potential of our approach.


\subsection{Additional Results of One-stage Global Localization}
\label{sec:appendix_one_stage_gl}
While Fig.~\ref{fig:gl_succ} in the main paper focuses on the quantitative evaluation of global localization performance, we provide additional qualitative results on both the NCLT and Oxford datasets. For the NCLT dataset, Fig.~\ref{fig:nclt_vision_gl_traj} and Fig.~\ref{fig:nclt_lidar_gl_traj} compare the global localization performance of RING\#-V against Exhaustive SS and RING\#-L against EgoNN. Similarly, for the Oxford dataset, Fig.~\ref{fig:oxford_vision_gl_traj} and Fig.~\ref{fig:oxford_lidar_gl_traj} depict global localization results of RING\#-V versus Exhaustive SS and RING\#-L versus EgoNN. In these figures, topological nodes are represented in \gray{gray}, while the ground truth location is shown in \green{green}.

\begin{figure*}[htbp]
  \centering
  \subfigure[RING\#-V (Ours) vs. Exhaustive SS~\cite{detone2018superpoint, sarlin2020superglue}]{
    \label{fig:nclt_vision_gl_traj}
    \includegraphics[width=8.2cm]{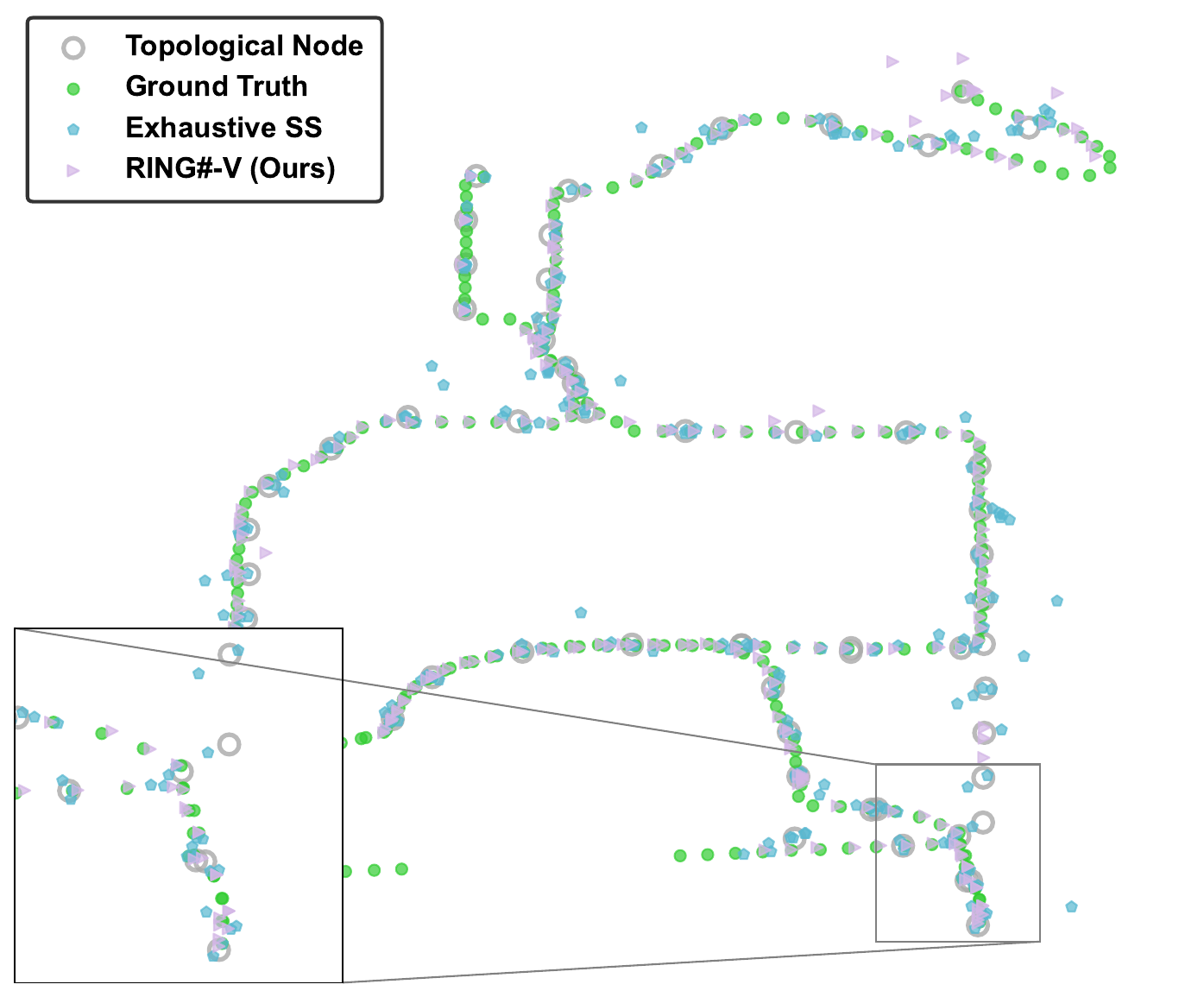}}
  \subfigure[RING\#-L (Ours) + ICP~\cite{koide2021voxelized} vs. EgoNN~\cite{komorowski2021egonn} + ICP~\cite{koide2021voxelized}]{
    \label{fig:nclt_lidar_gl_traj}
    \includegraphics[width=8.2cm]{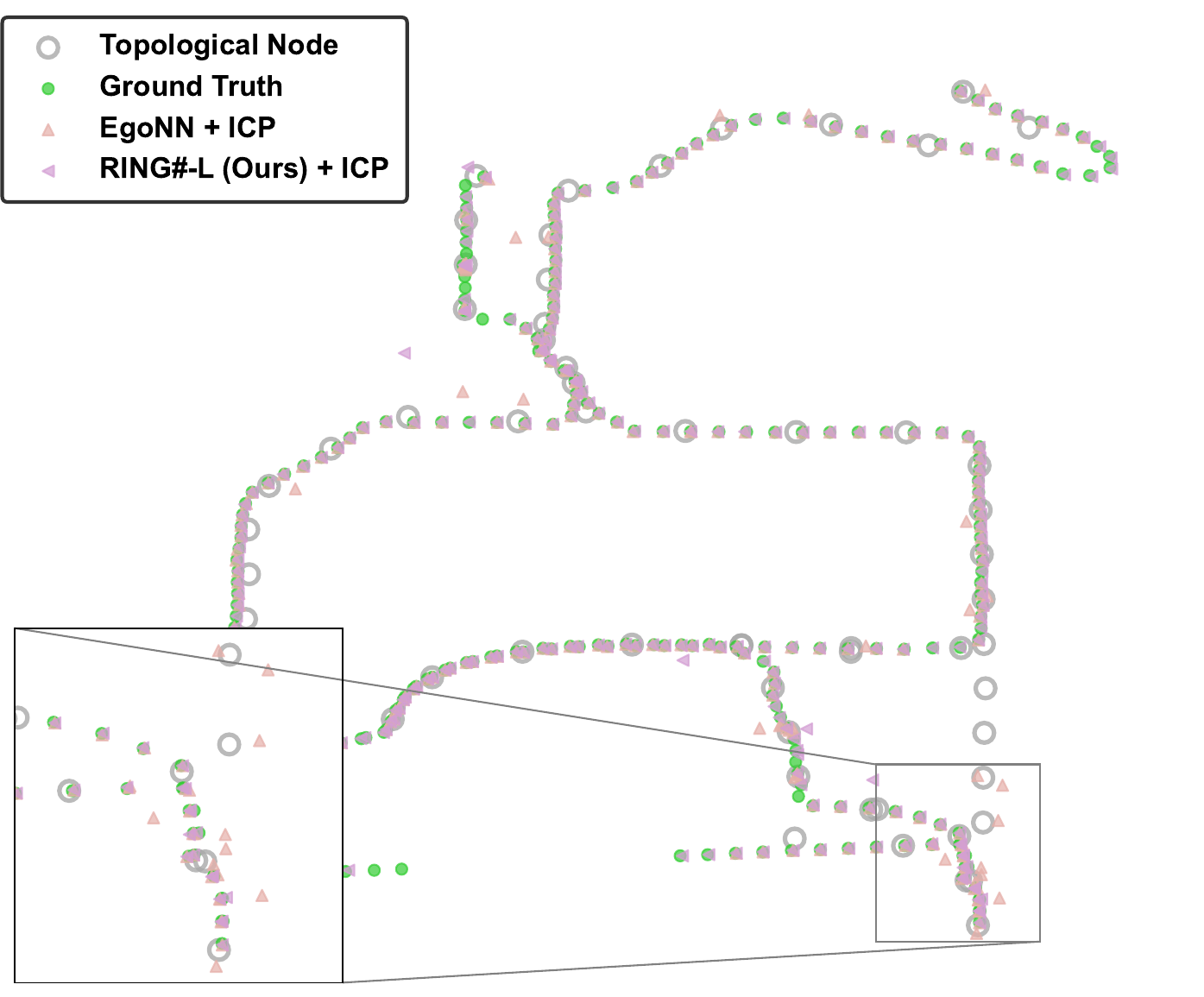}}
  \vspace{-0.2cm}
  \caption{\textbf{Global localization results of RING\#-V, Exhaustive SS, RING\#-L and EgoNN on the NCLT dataset.} The topological node is shown as \gray{gray} and the ground truth location is shown in \green{green}.}
  \label{fig:nclt_gl_traj}
  \vspace{-0.3cm}
\end{figure*}
\begin{figure*}[htbp]
  \centering
  \subfigure[RING\#-V (Ours) vs. Exhaustive SS~\cite{detone2018superpoint, sarlin2020superglue}]{
    \label{fig:oxford_vision_gl_traj}
    \includegraphics[width=8.2cm]{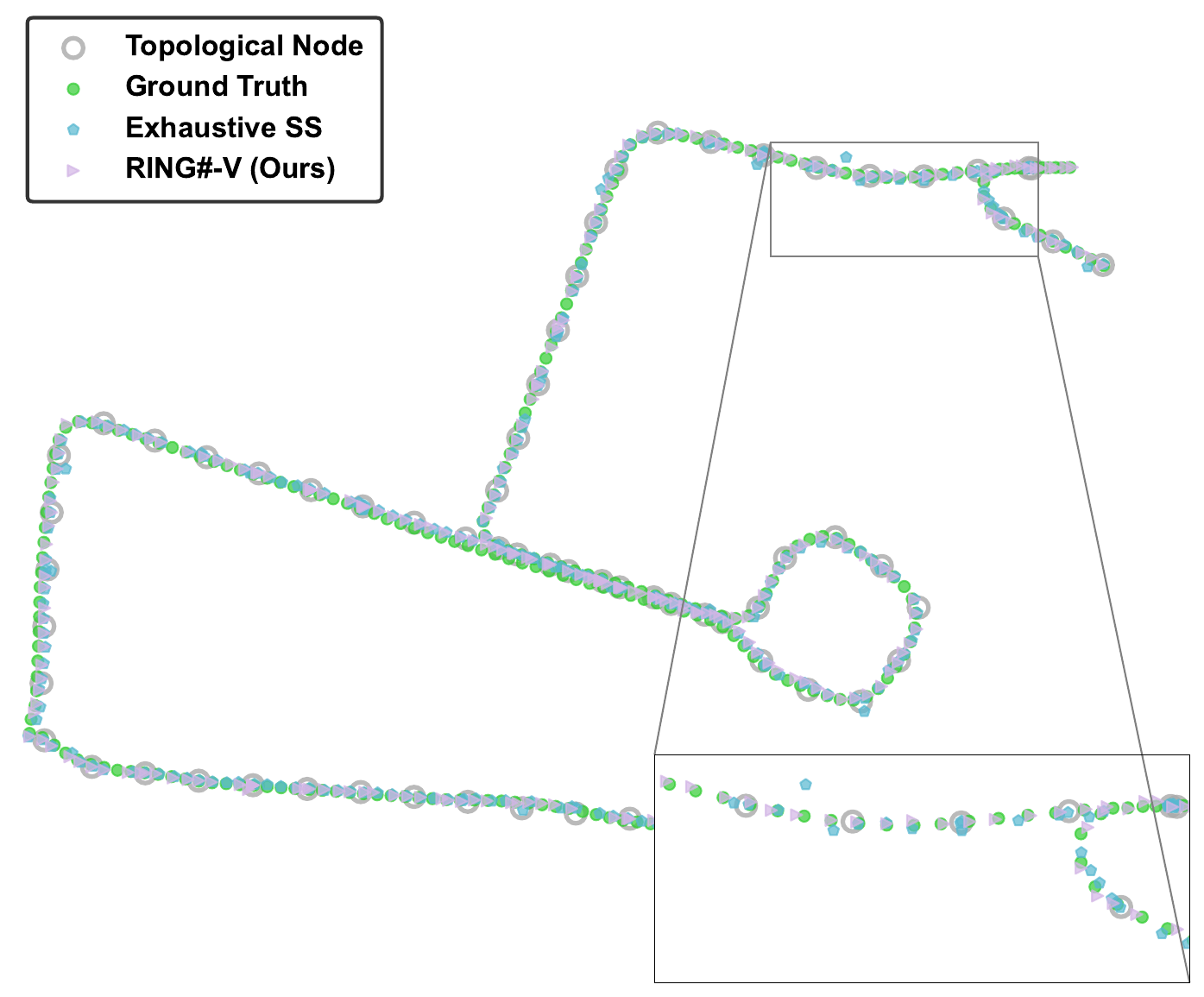}}
  \subfigure[RING\#-L (Ours) + ICP~\cite{koide2021voxelized} vs. EgoNN~\cite{komorowski2021egonn} + ICP~\cite{koide2021voxelized}]{
    \label{fig:oxford_lidar_gl_traj}
    \includegraphics[width=8.2cm]{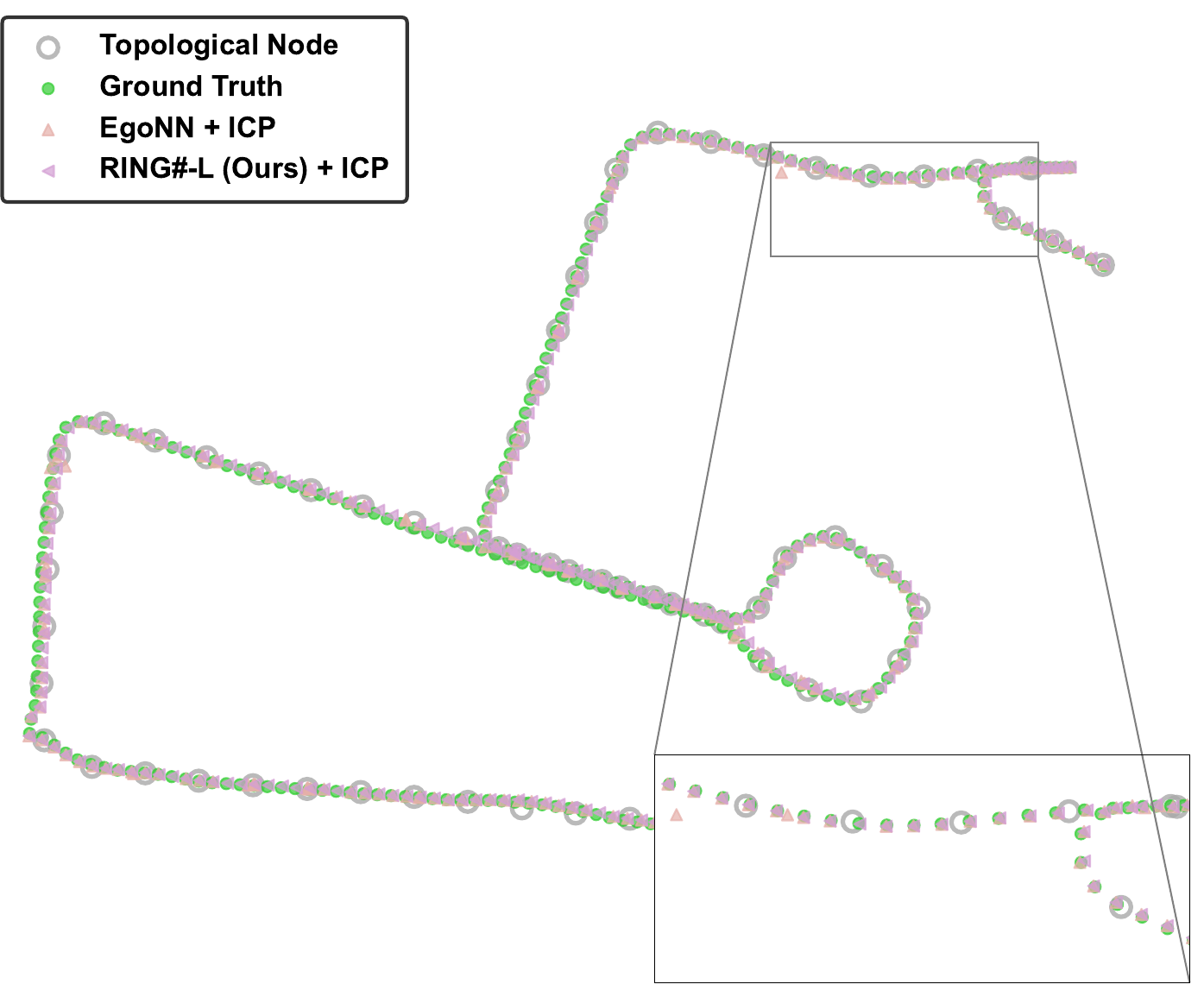}}
  \vspace{-0.2cm}
  \caption{\textbf{Global localization results of RING\#-V, Exhaustive SS, RING\#-L and EgoNN on the Oxford dataset.} The topological node is shown as \gray{gray} and the ground truth location is shown in \green{green}.}
  \label{fig:oxford_gl_traj}
  \vspace{-0.3cm}
\end{figure*}

As depicted in Fig.~\ref{fig:nclt_gl_traj}, significant viewpoint changes, such as large turns or backward movements in the trajectory, challenge the methods in localizing the robot accurately. In contrast, Fig.~\ref{fig:oxford_gl_traj} reveals more accurate localization results in the Oxford dataset, owing to fewer viewpoint and environmental variations. Despite challenging scenarios and sparse topological nodes, RING\#-V and RING\#-L  outperform the other methods on both datasets, showcasing the effectiveness of our approach. Specifically, RING\#-L benefits from explicit geometric constraints, allowing it to outperform RING\#-V, particularly on the NCLT dataset.

\subsection{Failure Cases Analysis}
\label{sec:appendix_failure_cases}
\begin{figure*}[htbp]
    \centering
    \includegraphics[width=18cm]{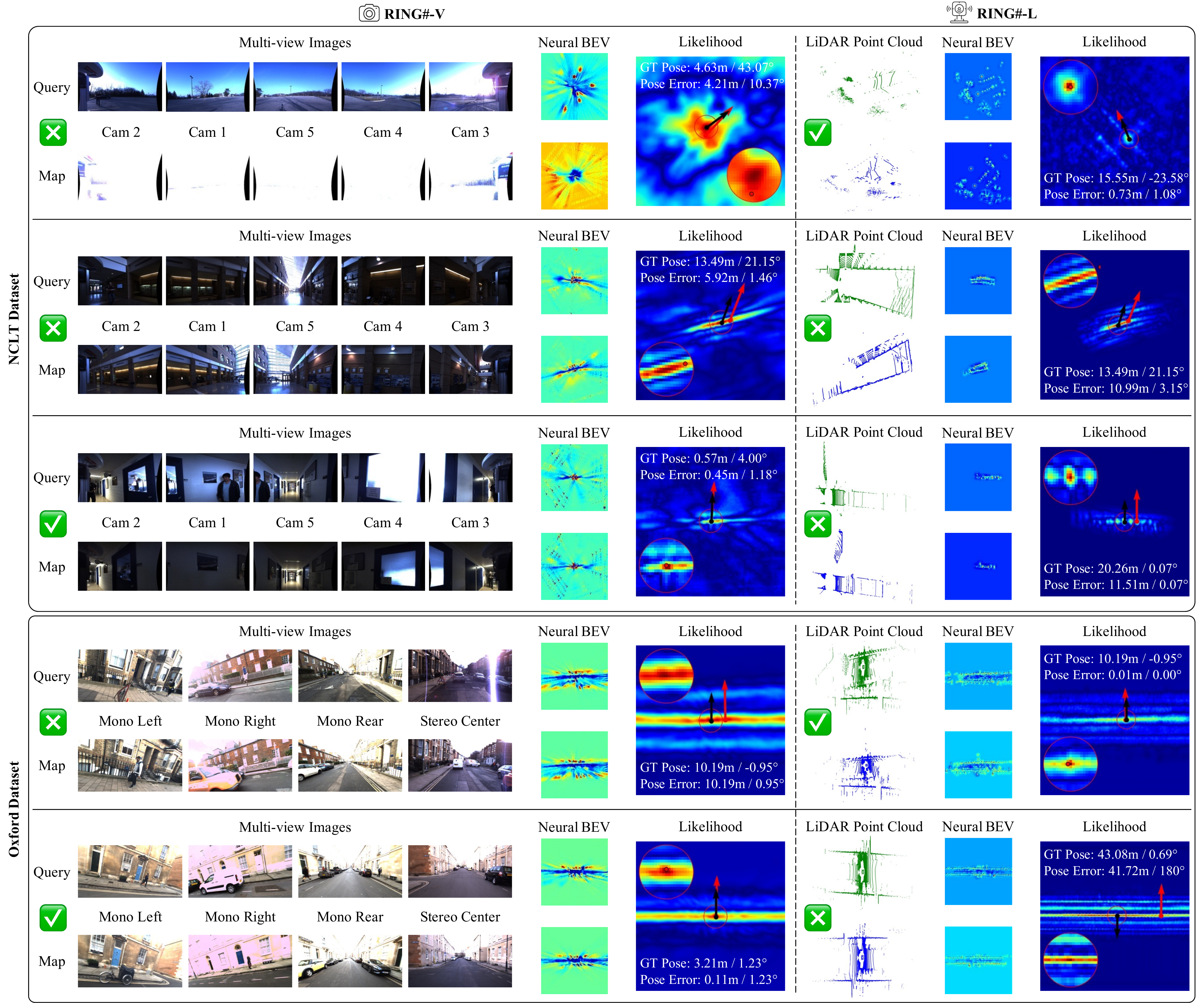}
    \vspace{-0.1cm}
    \caption{\textbf{Failure localization cases of RING\#-V and RING\#-L.} We display several failure cases of RING\#. On the likelihood plot, the black arrow $\rightarrow$ shows the pose estimated by RING\#, the red arrow {\color{red}{$\rightarrow$}} shows the ground truth pose and the red dot {\color{red}{$\bullet$}} shows the ground truth position. Here we rotate the query neural BEV by $\hat{\theta}$ estimated using RING\# to visualize the neural BEV under a 3-DoF pose transformation. \raisebox{-2pt}{\includegraphics[height=9pt]{figs/right.pdf}} denotes the successful localization case, while \raisebox{-2pt}{\includegraphics[height=9pt]{figs/wrong.pdf}} denotes the failure case.}
    \label{fig:case_gl_fail}
    \vspace{-0.5cm}
\end{figure*}

We conduct a detailed analysis of failure cases where RING\# delivers large rotation errors ($\text{RE} \geq 5^{\circ}$) or translation errors ($\text{TE} \geq 2m$) on the NCLT and Oxford datasets. Fig.~\ref{fig:case_gl_fail} illustrates representative examples of failure cases for both RING\#-V and RING\#-L, with successful cases marked by \raisebox{-2pt}{\includegraphics[height=9pt]{figs/right.pdf}} and failed cases indicated by \raisebox{-2pt}{\includegraphics[height=9pt]{figs/wrong.pdf}}.

Specifically, RING\#-V has shown vulnerabilities in environments with significant lighting variations, or in areas with repetitive patterns like building facades, where the extracted visual features are less discriminative. On the other hand, RING\#-L struggles in degraded environments characterized by repetitive geometries, such as long, uniform corridors. In such scenarios, LiDAR scans may appear similar across different sections, leading to localization ambiguities.

To address these limitations, future work could explore integrating multi-modal data, such as fusing vision and LiDAR information. Combining these modalities could enhance our approach's ability to disambiguate challenging environments and improve overall robustness. Addressing these issues could further improve the accuracy and reliability of RING\#, making it more effective across a broader range of scenarios.

\subsection{Details of Compared Methods}
\label{sec:appendix_methods}
We briefly outline the key vision- and LiDAR-based methods compared in our experiments:

\textbf{NetVLAD~\cite{arandjelovic2016netvlad}\footnote[5]{\url{https://github.com/Nanne/pytorch-NetVlad}}.} NetVLAD is a classic visual place recognition (VPR) method that integrates a differentiable VLAD layer into neural networks. We use the official PyTorch implementation with the released VGG16 model trained on the Pitts30k dataset to finetune the NetVLAD model on the NCLT and Oxford datasets.

\textbf{Patch-NetVLAD~\cite{hausler2021patch}\footnote[6]{\url{https://github.com/QVPR/Patch-NetVLAD}}.} Patch-NetVLAD enhances the original NetVLAD by incorporating a patch-based representation, improving robustness against challenging environmental changes. We finetune the official pre-trained model provided by the authors to evaluate Patch-NetVLAD on the NCLT and Oxford datasets.

\textbf{AnyLoc~\cite{keetha2023anyloc}\footnote[7]{\url{https://github.com/AnyLoc/AnyLoc}}.} AnyLoc is a recent VPR method that extracts features using the foundation models like DINOv2~\cite{oquab2023dinov2} for universal place recognition. We use the official model provided by the authors for evaluation.

\textbf{SFRS~\cite{ge2020self}\footnote[8]{\url{https://github.com/yxgeee/OpenIBL}}.} SFRS proposes self-supervised image-to-region similarites to learn discriminative features for VPR. We evaluate the performance of the official model on the NCLT and Oxford datasets to benchmark against our method.

\textbf{Exhaustive SS~\cite{detone2018superpoint, sarlin2020superglue}\footnote[9]{\url{https://github.com/magicleap/SuperGluePretrainedNetwork}}.} Exhaustive SS utilizes SuperPoint~\cite{detone2018superpoint} and SuperGlue~\cite{sarlin2020superglue} to extract and match local features for VPR. It exhaustively searches for the best pose across the entire database. We evaluate it using the official implementation on both datasets.

\textbf{BEV-NetVLAD-MLP.} BEV-NetVLAD-MLP is a variant of NetVLAD that operates in the BEV space, combining NetVLAD for place recognition with an MLP for pose estimation. It shares the BEV generation process with RING\# and serves as a direct comparison to assess the efficacy of BEV-based approaches on the NCLT and Oxford datasets.

\textbf{vDiSCO~\cite{xu2023leveraging}\footnote[10]{\url{https://github.com/MaverickPeter/vDiSCO}}.} vDiSCO is a BEV-based VPR method that leverages a transformer-based view transformation network to learn BEV features from multi-view images. We use the official implementation provided by the authors to evaluate vDiSCO on the NCLT and Oxford datasets.

\textbf{OverlapTransformer~\cite{ma2022overlaptransformer}\footnote[11]{\url{https://github.com/haomo-ai/OverlapTransformer}}.} OverlapTransformer is a LiDAR-based place recognition (LPR) method that leverages a transformer architecture to learn rotation-invariant global features from range images. We train and evaluate OverlapTransformer using the official code on both datasets, benchmarking its performance in LiDAR-based global localization.

\textbf{LCDNet~\cite{cattaneo2022lcdnet}\footnote[12]{\url{https://github.com/robot-learning-freiburg/LCDNet}}.} LCDNet incorporates a shared encoder, a place recognition head and a pose estimation head based on the unbalanced optimal transport theory, performing loop closure detection and 6-DoF point cloud registration simultaneously. We train LCDNet on the NCLT and Oxford datasets using the official code for evaluation.

\textbf{DiSCO~\cite{xu2021disco}\footnote[13]{\url{https://github.com/MaverickPeter/DiSCO-pytorch}}.} DiSCO transforms a 3D LiDAR point cloud into a 2D polar BEV representation, from which it extracts rotation-invariant global descriptors for place recognition. We train and evaluate the official implementation of DiSCO on the NCLT and Oxford datasets.

\textbf{RING~\cite{lu2022one}\footnote[14]{\url{https://github.com/lus6-Jenny/RING}}.} RING introduces a rotation-invariant and translation-invariant global descriptor by combining the Radon transform and Fourier transform for global localization. We evaluate the official implementation of RING on both datasets.

\textbf{RING++~\cite{xu2023ring++}\footnote[15]{\url{https://github.com/lus6-Jenny/RING}}.} RING++ extends RING by introducing multi-channel representations for improved global localization performance. We evaluate the official RING++ implementation on both datasets to compare against our method.


\textbf{EgoNN~\cite{komorowski2021egonn}\footnote[16]{\url{https://github.com/jac99/Egonn}}.} EgoNN jointly learns global and local features in the cylindrical coordinate to achieve both place recognition and 6-DoF pose estimation. Using the official implementation, we train EgoNN on the NCLT and Oxford datasets to benchmark its performance against other LiDAR-based methods.

\end{document}